\newif\ifFINAL
\FINALtrue

\documentclass{article}

\usepackage{microtype}
\usepackage{graphicx}
\usepackage{subfigure}
\usepackage{booktabs}
\usepackage{hyperref}

\usepackage[usenames,dvipsnames]{xcolor} 
\usepackage[accepted]{icml2024-arxiv}

\usepackage{amsmath}
\usepackage{amssymb}
\usepackage{mathtools}
\usepackage{amsthm}

\usepackage[capitalize,noabbrev]{cleveref}

\theoremstyle{plain}
\newtheorem{theorem}{Theorem}[section]
\newtheorem{proposition}[theorem]{Proposition}
\newtheorem{lemma}[theorem]{Lemma}
\newtheorem{corollary}[theorem]{Corollary}
\theoremstyle{definition}

\newtheorem{assumption}{Assumption}

\theoremstyle{remark}

\usepackage[textsize=tiny]{todonotes}

\def\safedef#1{%
   \ifx#1\undefined
      \expandafter\def\expandafter#1%
   \else
      \errmessage{The \string#1 is defined already}%
      \expandafter\def\expandafter\tmp
   \fi
}

\usepackage{amsthm,amsmath,bbm,amsfonts,amssymb}
\usepackage{dsfont} 
\usepackage{mathtools}
\usepackage{commath}
\usepackage{eqnarray}
\mathtoolsset{showonlyrefs=false}
\allowdisplaybreaks 

\usepackage{xspace}
\usepackage[T1]{fontenc}
\usepackage[utf8]{inputenc}

\usepackage{hyperref,url}

\usepackage{wrapfig}
\usepackage[export]{adjustbox}
\safedef\imagetop#1{\vtop{\null\hbox{#1}}} 
\usepackage{graphicx,tabularx}
\usepackage{multirow,hhline}
\usepackage{booktabs}
\usepackage{pbox} 
\usepackage{algorithm,algorithmic}

\usepackage[export]{adjustbox}

\usepackage{enumitem}

\newcommand{\kj}[1]{{\color{RedOrange}[KJ: #1]}}

\newcommand{\gray}[1]{{ \color[rgb]{.6,.6,.6} #1 }}
\newcommand{\blue}[1]{{\color[rgb]{.3,.5,1}#1}}

\newcommand{\tblue}[1]{{\color[rgb]{0,0,1}#1}} 

\newcommand{\guide}[1]{{[\color{Violet}#1]}}

\usepackage{framed}
\usepackage[most]{tcolorbox}
\definecolor{kjgray}{rgb}{.7,.7,.7}

\newtheoremstyle{kjstyle}
{1ex} 
{\topsep} 
{\itshape} 
{} 
{\bfseries} 
{.} 
{.5em} 
{} 

\newtheoremstyle{kjstyle2}
{.0em} 
{.0em} 
{\itshape} 
{} 
{\bfseries} 
{.} 
{.5em} 
{} 

\newtheoremstyle{kjstylenoitalic}
{1ex} 
{\topsep} 
{} 
{} 
{\bfseries} 
{.} 
{.5em} 
{} 

\tcbset{kjboxstyle/.style={title={},breakable,colback=white,enhanced jigsaw,boxrule=1.3pt,sharp corners,colframe=kjgray,boxsep=0pt,coltitle={black},attach title to upper={},left=.8ex,bottom=.4em}}    

\tcbset{kjboxstylec/.style={title={},breakable,colback=white,enhanced jigsaw,boxrule=1.3pt,sharp corners,colframe=kjgray,boxsep=0pt,coltitle={black},attach title to upper={},before skip=1.5ex,left=.8ex,bottom=.4em,top=0.4em,enlarge top by=-0.3em,enlarge bottom by=-0.3em}}    

\usepackage{mdframed}
\usepackage{lipsum}
\definecolor{kjgray}{rgb}{.7,.7,.7}
\makeatletter



\makeatletter
\renewcommand{\paragraph}{%
  \@startsection{paragraph}{4}%
  {\z@}{0.50ex \@plus 1ex \@minus .2ex}{-1em}%
  {\normalfont\normalsize\bfseries}%
}
\makeatother

\usepackage[customcolors,shade]{hf-tikz} 

\safedef\horizontalline{\noindent\rule{\textwidth}{1pt} }

\usepackage{tabularx} 
\newcolumntype{P}[1]{>{\centering\arraybackslash}p{#1}}
\newcolumntype{M}[1]{>{\centering\arraybackslash}m{#1}}

\def\ddefloop#1{\ifx\ddefloop#1\else\ddef{#1}\expandafter\ddefloop\fi}

\def\ddef#1{\expandafter\def\csname #1#1\endcsname{\ensuremath{\mathbb{#1}}}}
\ddefloop ABCDFGHIJKLMNORSTUWXYZ\ddefloop 

\def\ddef#1{\expandafter\def\csname c#1\endcsname{\ensuremath{\mathcal{#1}}}}
\ddefloop ABCDEFGHIJKLMNOPQRSTUVWXYZ\ddefloop

\def\ddef#1{\expandafter\def\csname b#1\endcsname{\ensuremath{{\mathbf{#1}}}}}
\ddefloop ABCDEFGHIJKLMNOPQRSTUVWXYZ\ddefloop  
\def\ddef#1{\expandafter\def\csname b#1\endcsname{\ensuremath{{\boldsymbol{#1}}}}}
\ddefloop abcdeghijklmnopqrtsuvwxyz\ddefloop  

\def\ddef#1{\expandafter\def\csname h#1\endcsname{\ensuremath{\hat{#1}}}}
\ddefloop ABCDEFGHIJKLMNOPQRSTUVWXYZabcdefghijklmnopqrsuvwxyz\ddefloop 
\def\ddef#1{\expandafter\def\csname hc#1\endcsname{\ensuremath{\hat{\mathcal{#1}}}}}
\ddefloop ABCDEFGHIJKLMNOPQRSTUVWXYZ\ddefloop
\def\ddef#1{\expandafter\def\csname hb#1\endcsname{\ensuremath{\hat{\mathbf{#1}}}}}
\ddefloop ABCDEFGHIJKLMNOPQRSTUVWXYZ\ddefloop %
\def\ddef#1{\expandafter\def\csname hb#1\endcsname{\ensuremath{\hat{\boldsymbol{#1}}}}}
\ddefloop abcdefghijklmnopqrstuvwxyz\ddefloop %

\def\ddef#1{\expandafter\def\csname t#1\endcsname{\ensuremath{\tilde{#1}}}}
\ddefloop ABCDEFGHIJKLMNOPQRSTUVWXYZabcdefgijklmnpqtsuvwxyz\ddefloop 
\def\ddef#1{\expandafter\def\csname tc#1\endcsname{\ensuremath{\tilde{\mathcal{#1}}}}}
\ddefloop ABCDEFGHIJKLMNOPQRSTUVWXYZ\ddefloop
\def\ddef#1{\expandafter\def\csname tb#1\endcsname{\ensuremath{\tilde{\mathbf{#1}}}}}
\ddefloop ABCDEFGHIJKLMNOPQRSTUVWXYZ\ddefloop
\def\ddef#1{\expandafter\def\csname tb#1\endcsname{\ensuremath{\tilde{\boldsymbol{#1}}}}}
\ddefloop abcdefghijklmnopqrstuvwxyz\ddefloop %

\def\ddef#1{\expandafter\def\csname bar#1\endcsname{\ensuremath{\bar{#1}}}}
\ddefloop ABCDEFGHIJKLMNOPQRSTUVWXYZabcdefghijklmnopqrtsuvwxyz\ddefloop
\def\ddef#1{\expandafter\def\csname barc#1\endcsname{\ensuremath{\bar{\mathcal{#1}}}}}
\ddefloop ABCDEFGHIJKLMNOPQRSTUVWXYZ\ddefloop
\def\ddef#1{\expandafter\def\csname barb#1\endcsname{\ensuremath{\bar{\mathbf{#1}}}}}
\ddefloop ABCDEFGHIJKLMNOPQRSTUVWXYZ\ddefloop
\def\ddef#1{\expandafter\def\csname barb#1\endcsname{\ensuremath{\bar{\boldsymbol{#1}}}}}
\ddefloop abcdefghijklmnopqrstuvwxyz\ddefloop %

\def\ddef#1{\expandafter\def\csname war#1\endcsname{\ensuremath{\overline{#1}}}}
\ddefloop ABCDEFGHIJKLMNOPQRSTUVWXYZabcdefghijklmnopqrtsuvwxyz\ddefloop
\def\ddef#1{\expandafter\def\csname warc#1\endcsname{\ensuremath{\overline{\mathcal{#1}}}}}
\ddefloop ABCDEFGHIJKLMNOPQRSTUVWXYZ\ddefloop
\def\ddef#1{\expandafter\def\csname warb#1\endcsname{\ensuremath{\overline{\mathbf{#1}}}}}
\ddefloop ABCDEFGHIJKLMNOPQRSTUVWXYZ\ddefloop
\def\ddef#1{\expandafter\def\csname warb#1\endcsname{\ensuremath{\overline{\boldsymbol{#1}}}}}
\ddefloop abcdefghijklmnopqrstuvwxyz\ddefloop %

\safedef\tilr{\tilde r}
\safedef\bff{{\boldsymbol f}}
\safedef\hbff{{\hat{\boldsymbol f}}}
\safedef\hatt{\hat{t}}
\safedef\tilo{{\tilde{o}}}
\safedef\tilh{{\tilde{h}}}
\safedef\bell{{{\boldsymbol\ell}}}
\safedef\tell{\ensuremath{\tilde{\ell}}} 
\safedef\btell{\ensuremath{\widetilde{\boldsymbol{\ell}}}} 
\safedef\hell{{{\hat\ell}}}

\safedef\rialpha{\ensuremath{{\mathring{\alpha}}}} 
\safedef\riz{\ensuremath{\mathring{z}}} 
\safedef\ribeta{\ensuremath{\mathring{\beta}}} 

\newcommand{\fr}[2]{ { \frac{#1}{#2} }}

\newcommand{\wbar}[1]{{\ensuremath{\overline{#1}}}}

\newcommand{\T}{\top}
\safedef\wed{\wedge}
\safedef\tsty{\textstyle}
\safedef\bec{\because}

\safedef\cd{\cdot}
\safedef\cc{{\circ}}
\safedef\la{\langle}
\safedef\ra{\rangle}
\safedef\dsum{\ensuremath{\displaystyle\sum}} 
\safedef\der{\ensuremath{\partial}\xspace}
\safedef\llfl{\left\lfloor} 
\safedef\rrfl{\right\rfloor}  
\safedef\llcl{\left\lceil}  
\safedef\rrcl{\right\rceil}  
\safedef\lfl{\lfloor} 
\safedef\rfl{\rfloor}  
\safedef\lcl{\lceil}  
\safedef\rcl{\rceil}  
\safedef\larrow{\ensuremath{\leftarrow}\xspace} 
\safedef\rarrow{\ensuremath{\rightarrow}\xspace} 
\safedef\sm{{\ensuremath{\setminus}\xspace} }
\safedef\grad{\ensuremath{\mathbf{\nabla}}\xspace}  
\safedef\lt{\left}
\safedef\rt{\right}

\definecolor{mygrn}{rgb}{0,.8,0}
\definecolor{myred}{rgb}{.8,0,0}

\safedef\sig{\sigma}
\safedef\om{\omega}
\safedef\dt{\delta}
\safedef\gam{\gamma}
\safedef\lam{\lambda}
\safedef\kap{\kappa}
\safedef\eps{\varepsilon}
\def\epsilon{\varepsilon}
\def\th{\theta}
\safedef\Lam{\Lambda}
\safedef\Dt{\Delta}
\safedef\Gam{\Gamma}
\safedef\Sig{\Sigma}
\safedef\Th{\Theta} 
\safedef\Om{\Omega}

\usepackage{pgffor}
\safedef\greeksymbols{alpha,beta,gamma,gam,delta,dt,eps,epsilon,zeta,eta,theta,th,iota,kappa,kap,lambda,lam,mu,nu,xi,pi,rho,sigma,sig,tau,phi,chi,psi,omega,om,Gamma,Gam,Delta,Dt,Theta,Th,Lambda,Lam,Pi,Sigma,Sig,Phi,Psi,Omega,Om}
\safedef\greeksymbolsnoeta{alpha,beta,gamma,gam,delta,dt,eps,epsilon,zeta,theta,th,iota,kappa,kap,lambda,lam,mu,nu,xi,pi,rho,sigma,sig,tau,phi,chi,psi,omega,om,Gamma,Gam,Delta,Dt,Theta,Th,Lambda,Lam,Pi,Sigma,Sig,Phi,Psi,Omega,Om} 

\foreach \x in \greeksymbolsnoeta{\expandafter\xdef\csname b\x\endcsname{\noexpand\ensuremath{\noexpand\boldsymbol{\csname \x\endcsname}}}}
\safedef\bfeta{{\boldsymbol \eta}}

\foreach \x in \greeksymbols{\expandafter\xdef\csname h\x\endcsname{\noexpand\ensuremath{\noexpand\hat{\csname \x\endcsname}}}}
\foreach \x in \greeksymbolsnoeta{\expandafter\xdef\csname hb\x\endcsname{\noexpand\ensuremath{\noexpand\hat{\noexpand\boldsymbol{ \csname \x\endcsname}}}}}
\safedef\hbfeta{{\hat{\boldsymbol \eta}}}

\foreach \x in \greeksymbols{\expandafter\xdef\csname bar\x\endcsname{\noexpand\ensuremath{\noexpand\bar{\csname \x\endcsname}}}}
\foreach \x in \greeksymbolsnoeta{%
\expandafter\xdef\csname barb\x\endcsname{\noexpand\ensuremath{\noexpand\bar{\noexpand\boldsymbol{ \csname \x\endcsname}}}}
}
\safedef\barbfeta{{\bar{\boldsymbol \eta}}}

\foreach \x in \greeksymbols{\expandafter\xdef\csname t\x\endcsname{\noexpand\ensuremath{\noexpand\tilde{\csname \x\endcsname}}}}
\foreach \x in \greeksymbolsnoeta{\expandafter\xdef\csname tb\x\endcsname{\noexpand\ensuremath{\noexpand\tilde{\noexpand\boldsymbol{ \csname \x\endcsname}}}}}
\safedef\tbfeta{{\tilde{\boldsymbol \eta}}}

\safedef\dmu{{\dot\mu}}
\safedef\ddmu{{\ddot\mu}}

\DeclareMathOperator{\EE}{\mathbb{E}} 

\DeclareMathOperator{\PP}{\mathbb{P}}

\safedef\clip#1{\wbar{\del{#1}}}
\DeclareMathOperator{\one}{\mathds{1}\hspace{-.1em}}
\safedef\onec#1{\one\cbr{#1}}

\newcommand{\normz}[1]{{\norm[0]{#1}}}
\DeclarePairedDelimiterX{\inp}[2]{\langle}{\rangle}{#1, #2}

\newcommand\declareop[3]{%
  \newcommand#1{%
    \mskip\muexpr\medmuskip*#2\relax
    {#3}%
    \mskip\muexpr\medmuskip*#2\relax
}}
\declareop\capprox{1}{{\sr{\const}{\approx}}} 
\declareop\logapprox{1}{{\sr{\mathsf{log}}{\approx}}} 

\newcommand{\lsim}{\mathop{}\!\lesssim}

\safedef\Bin{\mathsf{Bin}}
\safedef\Uniform{{\mathsf{Uniform}}}
\safedef\Bernoulli{{\ensuremath{\mathsf{Bernoulli}}}}

\safedef\kt{{\mathsf{kt}}}
\safedef\mle{{\mathsf{mle}}}

\safedef\Approx{{\mathsf{Approx}} }
\safedef\denom{{\mathsf{denom}}}
\safedef\eff{{\mathsf{eff}}}
\safedef\Seff{{S_{\mathsf{eff}}}}
\safedef\opt{{\mathsf{opt}}}
\safedef\pes{{\mathsf{pes}}}
\safedef\faury{{\mathsf{faury}}}
\safedef\nice{{\textsf{nice} } } 
\safedef\ErrPrb{{\mathsf{ErrPrb}}}
\safedef\Seg{{\mathsf{Seg}}}
\safedef\COM{\mathsf{COM}}
\safedef\const{\mathsf{const}}
\safedef\wo{{\ensuremath{\mathsf{wo}}}}
\safedef\Top{\mathsf{Top}}
\safedef\Bot{\mathsf{Bot}}
\safedef\Sim{\mathsf{Sim}}
\safedef\TV{\mathsf{TV}}
\safedef\tmin{{\min}}
\safedef\tmax{{\max}}

\safedef\kl{{\mathsf{kl}}}
\safedef\err{\mathsf{err}} 
\safedef\logloss{{\mathsf{logloss}}}
\safedef\Ber{{\mathsf{Ber}}}
\safedef\erf{{\text{erf}}}
\safedef\erfc{\mathsf{erfc}}
\safedef\rcF{\ensuremath{\mathring{\cF}}} 
\safedef\Cf{{\ensuremath{{\normalfont{\text{Cf}}}}}}
\safedef\barCf{{\ensuremath{\wbar{\text{Cf}}}}}

\safedef\SR{{\ensuremath{\text{SR}}}\xspace}
\safedef\lin{{\ensuremath{\mathsf{lin}}}}

\safedef\IC{{\ensuremath{\normalfont{\text{IC}}}}}

\safedef\Reward{\ensuremath{\text{Reward}}}
\safedef\poly{\operatorname{poly}}
\safedef\Misid{\operatorname{Misid}}
\safedef\Corral{\ensuremath{\normalfont{\textsc{Corral}}}\xspace}
\safedef\AUL{{\ensuremath{\normalfont{\text{AUL}}}}} 
\safedef\Rel{{\ensuremath{\normalfont{\text{Rel}}}}} 
\safedef\Mis{{\ensuremath{\normalfont{\text{Mis}}}}} 
\safedef\Rad{\ensuremath{\text{\normalfont{Rad}}}}

\safedef\Reg{{\mathsf{Reg}}}
\safedef\Regret{\ensuremath{\normalfont{\text{Regret}}}}
\safedef\Wealth{\ensuremath{\normalfont{\text{Wealth}}}}
\safedef\Active{\ensuremath{\text{Active}}}
\safedef\decomp{\ensuremath{\mbox{decomp}}\xspace}
\safedef\sym{{\ensuremath{\text{Sym}}\xspace}} 
\safedef\suchthat{\ensuremath{\text{ s.t. }}}

\usepackage{pifont}
\newcommand{\sr}{\stackrel}

\safedef\bigmid{\,\middle|\,\xspace}

\makeatletter
\newcommand{\vast}{\bBigg@{3}}
\newcommand{\Vast}{\bBigg@{4}}
\makeatother

\safedef\rhoX{{\rho_{\mathcal{X}}}}
\safedef\lammin{{\lambda_{\min}}}
\safedef\elllog{{\ell^{\mathsf{log}}}}
\safedef\lampp{{\lam_\pp}}

\safedef\resh{\text{resh}}
\safedef\SVD{\text{SVD}}
\safedef\op{{\text{op}}}
\safedef\sT{{*\T}}

\safedef\Vol{{\text{Vol}}}
\safedef\pp{\perp}

\safedef\rell{{\mathring{\ell}}}
\safedef\lamI{{\lam I}}

\def\semi{{\normalfont\text{semi}}}
\def\full{{\normalfont\text{full}}}
\def\polylog{{\normalfont\text{polylog}}}

\def\barSig{{\wbar{\Sig}}}

\DeclareMathOperator{\barlnln}{\overline{\ln\ln}}

\def\ellw{\ell^{\mathsf{w}}}
\def\reg{\text{reg}}

\def\portfolio{{\text{portfolio}}}

\let\cite\citep 

\usepackage{enumitem}
\setlist[itemize]{topsep=.5pt,itemsep=0pt,parsep=2pt}
\setlist[enumerate]{topsep=.5pt,itemsep=0pt,parsep=2pt}

\usepackage{minitoc}

\def\Reg{{\normalfont\text{Reg}}}

\ifFINAL
  \def\blue#1{#1}
  \def\guide#1{}
  \def\gray#1{}
  \def\kj#1{}
\else
\fi

\icmltitlerunning{Noise-Adaptive Confidence Sets for Linear Bandits and Application to Bayesian Optimization}

\begin{document}
\textfloatsep=.6em
  
\doparttoc 
\faketableofcontents 

\twocolumn[
\icmltitle{Noise-Adaptive Confidence Sets for Linear Bandits and Application to Bayesian Optimization}



\icmlsetsymbol{equal}{*}

\begin{icmlauthorlist}
\icmlauthor{Kwang-Sung Jun}{yyy}
\icmlauthor{Jungtaek Kim}{zzz}
\end{icmlauthorlist}

\icmlaffiliation{yyy}{University of Arizona}
\icmlaffiliation{zzz}{University of Pittsburgh}

\icmlcorrespondingauthor{Kwang-Sung Jun}{kjun@cs.arizona.edu}

\icmlkeywords{Contextual bandits, multi-armed bandits, sequential decision-making}

\vskip 0.3in
]



\printAffiliationsAndNotice{}

\begin{abstract}
Adapting to a priori unknown noise level is a very important but challenging problem in sequential decision-making as efficient exploration typically requires knowledge of the noise level, which is often loosely specified. We report significant progress in addressing this issue for linear bandits in two respects. First, we propose a novel confidence set that is `semi-adaptive' to the unknown sub-Gaussian parameter $\sigma_*^2$ in the sense that the (normalized) confidence width scales with $\sqrt{d\sigma_*^2 + \sigma_0^2}$ where $d$ is the dimension and $\sigma_0^2$ is the  specified sub-Gaussian parameter (known) that can be much larger than $\sigma_*^2$. This is a significant improvement over $\sqrt{d\sigma_0^2}$ of the standard confidence set of Abbasi-Yadkori et al. (2011), especially when $d$ is large. We show that this leads to an improved regret bound in linear bandits. Second, for bounded rewards, we propose a novel variance-adaptive confidence set that has much improved numerical performance upon prior art. We then apply this confidence set to develop, as we claim, the first practical variance-adaptive linear bandit algorithm via an optimistic approach, which is enabled by our novel regret analysis technique. Both of our confidence sets rely critically on `regret equality' from online learning. Our empirical evaluation in diverse Bayesian optimization tasks shows that our proposed algorithms demonstrate better or comparable performance compared to existing methods.
\end{abstract}

\section{Introduction}
\label{sec:intro}

In linear bandits~\cite{abe99associative,auer02using,dani08stochastic,ay11improved}, the learner faces the challenge of making judicious sequential decisions with observed features so that the rewards obtained from those decisions would be maximized.\footnote{The implementation of our proposed methods is available at \url{https://github.com/jungtaekkim/LOSAN-LOFAV}.}
Specifically, at each time step $t=1,2,\ldots,n$, the learner is given an arm set $\blue{\cX_t} \subset \RR^d$ and chooses an arm $\blue{x_t} \in \cX_t$ to observe a reward
\begin{align}\label{eq:model} 
  \blue{y_t} = \la x_t, \th^* \ra + \eta_t,
\end{align}
where $\blue{\th^*} \in \RR^d$ is unknown, $\blue{\eta_t} \mid \cF_{t-1}$ is a zero-mean stochastic noise, and $\blue{\cF_{t-1}} := \allowbreak \sigma(x_1,\allowbreak y_1,\allowbreak\ldots,\allowbreak x_{t-1},\allowbreak y_{t-1}, x_t)$.
The goal is to maximize the cumulative reward $\sum_{t=1}^n y_t$. 
The standard theoretical performance measure is cumulative pseudo-regret (or simply `regret') defined as
\begin{align}
  \Reg_n := \sum_{t=1}^n \la x_{t,*}, \th^*\ra - \la x_{t}, \th^*\ra,
\end{align}
where $x_{t,*} = \arg \max_{x\in\cX_t} \la x, \th^*\ra$ is the arm with the highest mean reward.
The significance of this problem lies in its wide applicability, ranging from recommendation systems~\cite{li10acontextual} to online advertising~\cite{avadhanula21stochastic} where the learner must balance between exploration (acquiring information about the model) and exploitation (pulling the arm estimated to give high reward).

A critical aspect of the linear bandit problem is the fact that an algorithm requires prior knowledge of the noise level (or an upper bound of it) and that its performance critically depends on the \textit{specified} noise level rather than the \textit{actual} noise level.
In this paper, we make significant progress in addressing this issue by proposing two algorithms that can semi or fully adapt to the actual noise level in two different setups.

The first setup is when the noise $\eta_t \mid \cF_{t-1}$ is $\sig_*^2$-sub-Gaussian where the sub-Gaussian parameter specified to the algorithm is $\sigma_0^2$ that can be much larger than $\sig_*^2$.
We propose a novel linear bandit algorithm called \textbf{LOSAN} (Linear Optimism with Semi-Adaptivity to Noise).
The critical ingredient for this algorithm is a novel confidence set whose (normalized) confidence width contains online variance estimators and is no larger than $\tcO(\sqrt{d\sig_*^2 + \sig_0^2})$ with high probability where $\tcO$ hides polylogarithmic factors.
This is no worse than $\tcO(\sqrt{d\sig_0^2 + \sig_0^2})$ of the standard confidence set of~\citet{ay11improved}, and in fact our confidence set can be significantly smaller when $\sig_0^2$ is largely over-specified.
LOSAN leverages our novel confidence set to perform optimistic exploration, which we prove to have a regret bound of $\tcO(\sig_0 \sqrt{dn} + \sig_* d\sqrt{n})$.
This bound, compared to the state-of-the-art bound of $\tcO(\sig_0d\sqrt{n})$, effectively achieves a factor of $\sqrt{d}$ improvement when $\sig^2_* \ll \sig^2_0$.
We present our confidence set and LOSAN in Section~\ref{sec:semi}.

The second setup is the case of bounded noise; i.e., $\forall t, \eta_t \in [-R,R]$ almost surely for some $R\ge0$.
Among the recent studies reporting regret bounds that adapt to the unknown variance $\sig_t^2$ of the noise $\eta_t$~\cite{zhang21variance,kim22improved,xu23noise,zhao23variance}, the seminal work of~\citet{zhao23variance} proposed an algorithm called SAVE and proved a regret bound of $\tcO(d\sqrt{R^2 + \sum_{t=1}^n \sig_t^2})$, which is unimprovable in general as it matches the optimal worst-case regret~\cite{lattimore20bandit}.
Note that this rate is orderwise never worse than existing regret bounds with the sub-Gaussian assumption like OFUL~\cite{ay11improved} that achieves $d\sqrt{R^2 n}$.

However, SAVE follows the SupLinRel~\cite{auer02using} style whose numerical performance is bad compared to other approaches such as optimistic~\cite{dani08stochastic} or posterior sampling~\cite{agrawal13thompson} strategies since it builds confidence bounds based on a small subset of samples only.
Furthermore, their algorithm adapts to the variance that is unknown yet \textit{fixed deterministically} before the bandit game starts. 
This severely limits its applicability.
For example, such a setting cannot incorporate the environment where the variance at time $t$ changes as a function of the chosen arm at time $t$ or the past behavior of the learner.

To overcome these limitations, we propose a novel confidence set that not only removes the limiting assumption on the noise discussed above but also significantly improves the numerical performance.
We then propose a novel optimistic-style linear bandit algorithm called \textbf{LOFAV} (Linear Optimism with Full Adaptivity to Variance) that computes the confidence set with all the available samples rather than a small subset of them.
Our analysis shows that LOFAV enjoys the same order of regret as SAVE, which is optimal up to logarithmic factors.
LOFAV can be implemented with the computational complexity of $\tcO((d^2\max_{t=1}^n|\cX_t|)n)$, which is the same as SAVE and is within a logarithmic factor of that of the standard algorithm OFUL~\cite{ay11improved}.
We present our confidence set, LOFAV, and their analysis in~\cref{sec:fully}.

Finally, we empirically validate the performance of both of our algorithms in synthetic experiments and the benchmarks widely used in the Bayesian optimization community including NATS-Bench~\citep{DongX2021ieeetpami}.
In this empirical analysis,
our algorithms yield better or comparable performance compared to OFUL and potentially the simple discrete Bayesian optimization strategy~\citep{GarridoEC2020neucom}.
We present our empirical results in~\cref{sec:expr}, discuss related work in~\cref{sec:related}, and conclude our paper with exciting future research directions in~\cref{sec:conclusion}.

\paragraph{Preliminaries.}
Throughout, we assume that both the confidence set and bandit algorithms have prior knowledge of $\blue{S}$ such that $\normz{\th^*} \le S$.
Note that techniques employed in~\citet{gales22norm} can be readily applied to remove this assumption with only a constant factor inflation in the leading term of $\sqrt{n}$ in the regret bound and a polynomial factor in the lower order term.
We assume that $\cX_t \subseteq \{x\in \RR^d: \normz{x}_2 \le 1\}$.
For a vector $x$, we define $\blue{\normz{x}} := \|x\|_2$ as the euclidean norm.
We define $\blue{\normz{x}_{V}} := \sqrt{x^\T V x}$ where $V$ is a positive semi-definite matrix.
We use $\blue{a \lsim b}$ to denote that there exists an absolute constant $c>0$ such that $a \le cb$. 
We use $\tcO$ as the big-O notation that omits polylogarithmic factors.
Let $\blue{a \vee b} := \max\{a,b\}$ and $a \wed b := \min\{a,b\}$.
Define $\blue{\barlnln(x)} := \ln\ln(e \vee x)$ and $[a..b] := \{a, a+1, \ldots, b\}$.

\section{Semi-Adaptation for Sub-Gaussian Noise}
\label{sec:semi}

In this section, we assume the standard linear bandit setup with the reward model of \eqref{eq:model} where the conditional noise $\eta_t \mid \cF_{t-1}$ satisfies the following assumption, which is standard in linear bandits.
\begin{assumption}\label{ass:semi}
  The noise $\eta_t \mid \cF_{t-1}$ is $\sig_*^2$-sub-Gaussian, $\forall t$.
  The algorithm has a prior knowledge of $\blue{\sig_0^2}$ such that $\sig_*^2 \le \sig_0^2$.
\end{assumption}
We first introduce our confidence set.
Our intuition comes from the standard FTRL (Follow-The-Regularized-Leader) regret equality~\citep[Lemma 7.1]{orabona19modern} (restated in~\cref{lem:regret_equality}).
For the case of the squared loss $\ell(\th) = \fr12(x_t^\T \th - y_t)^2$ with the standard online ridge regression estimator $\th_t = \arg \min_{\th} \sum_{s=1}^t \ell(\th) + \fr\lam2 \normz{\th}^2$, the regret equality implies
\begin{align*}
  \sum_{s=1}^t \ell_s(\th_{s-1}) - \ell_s(\th^*)
  &= \fr\lam2\normz{\th^*}^2 + \sum_{s=1}^t \ell_s(\th_{s-1}) \normz{x_s}^2_{V_s^{-1}}
\\&\quad - \fr12 \normz{\th_t - \th^*}^2_{V_t},
\end{align*}
where
\begin{align*}
    \blue{V_t} := \lam I + \sum_{s=1}^t x_s x_s^\T.
\end{align*}
The negative term on the RHS is often not useful and thus ignored in online learning.
However, this term exactly appears in confidence sets for linear bandits such as the standard self-normalized confidence set (SNCS)~\cite{ay11improved}.
Using the fact that $\sum_{s=1}^t \ell_s(\th_{s-1}) - \ell_s(\th^*) \ge - \sig_*^2 \ln(1/\dt) $ with high probability (see \cref{lem:negregret}), which we call the `negative regret bound', we obtain
\begin{align*}
  \fr12\normz{\th_t - \th^*}^2_{V_t} 
  &\le \fr{\lam}{2} \normz{\th^*}^2 \nonumber\\
  &\quad+ \underbrace{\sum_{s=1}^t \ell_s(\th_{s-1}) \normz{x_s}^2_{V_s^{-1}}}_{=: A_t} + \sig_*^2\ln(\frac1\dt)~.
\end{align*}
One can bound $\|\th^*\| \le S$ and use \cref{ass:semi} to bound $\sig_*^2 \le \sig_0^2$.
One can further try to upper bound $A_t$ to construct a confidence set.
Indeed, this is the approach taken by~\citet{dekel12selective} for online selective sampling (with minor differences) where they assume bounded noise; i.e., $\eta_t \mid \cF_{t-1} \in [0,1]$ almost surely.
This means that we have $\sig_0^2 = 1/4$.
This allows them to control $\ell_s(\th_{s-1})$ loosely such that $A_t \le \blue{\barA_t} := O(\sig_0^2 d\ln(t))$ and construct a confidence set
\begin{align*}
  \cbr[2]{\th \in\RR^d: \fr12 \normz{\th_t - \th}^2_{V_t} \le \fr\lam2 S^2 + \barA_t + \sig_0^2\ln(\frac1\dt)}~.
\end{align*}
While details differ, this is the essence of~\citet{dekel12selective}.
However, this does not provide an improvement over SNCS.
In another work of \citet{jun17scalable}, a slightly different technique of online-to-confidence-set conversion (see \cref{sec:related} for details) also introduces a term like the second term above; they use a union bound over the time steps to control $A_t \le O(\sig_0^2 S d \ln^2(t))$, which leads to an even looser bound than SNCS.

Departing from prior work, we propose to keep $A_t$ as is, which motivates the following confidence set:
\begin{align}\label{eq:conf-set-0}
  \cbr[2]{\th \in\RR^d: \fr12\normz{\th_t - \th}^2_{V_t} \le \fr\lam2 S^2 + A_t + \sig_0^2\ln(\frac1\dt)} ~.
\end{align}
The benefit is the following observation: If $\ell_s(\th_{s-1}) \approx \ell_s(\th^*)$, then $A_t = O(\sig_*^2 d \ln(t))$ with high probability.
Then, we can have a confidence set whose (normalized) confidence width (i.e., the upper bound on $\|\th_t - \th\|_{V_t}$) is of order $\cO\del[0]{\sqrt{d\sig_*^2 + \sig_0^2}}$ rather than $\cO(\sqrt{d\sig_0^2})$ of SNCS.
Note that the assumption of $\ell_s(\th_{s-1}) \approx \ell_s(\th^*)$ is sensible since for large enough $s$, the loss of $\th_{s-1}$ should be sufficiently similar to the true parameter $\th^*$.

However, our intuition does not easily lead to a confidence set whose radius is $\sqrt{d\sig_*^2 + \sig_0^2}$ for technical reasons that we omit here.
Still, one can see that the proposed confidence set above cannot be strictly better than SNCS.
The reason is that $A_t$ contains $\ell_1(\hth_0) = \ell_1(0) = y_1^2/2$, which can be $\Omega(B^2)$ where $B = \max_t \max_{x\in\cX_1}|\inp{x}{\th^*}|$ that can be as large as $S$.
Such a dependence does not appear in SNCS (note the factor $S$ can be canceled out by setting a large $\lam$), and thus~\eqref{eq:conf-set-0} cannot be strictly better than SNCS.
Instead, we found that an added assumption of $\| x_s \|^2_{V_s^{-1}} \le \fr12$ helps attain the desired inequality (\cref{lem:empirical-var}):
\begin{align*}
  A_t \lsim \sum_{s=1}^t \ell_s(\th^*) D_s^2 + \lam S^2 + \sig_0^2 \ln(1/\dt)~.
\end{align*}
While the tools proposed by~\citet{zhao23variance} can be used to satisfy $\| x_s \|^2_{V_s^{-1}} \le \fr12$.
However, this requires using weighted ridge regression, which changes the left-hand side $\fr12 \|\th_t - \th\|^2_{V_t}$ of the constraint now involves $V_t$ that consists of the weighted versions of $\{x_s\}_{s=1}^t$, which will block us from applying the standard elliptical potential lemma (e.g., \citet[Lemma 11]{ay11improved}) from the existing analysis technique. 
That is, even if we can achieve the target confidence width, the regret analysis must be done differently to accommodate the change.
Therefore, both our confidence set and the regret analysis are our novelty.

\paragraph{Proposed confidence set.}
Departing from the standard ridge regression, we use a weighted ridge regression estimator inspired by~\citet{zhao23variance}. 
Define the weight $\blue{w_t}$:
\begin{equation*}
    \blue{w_t} = 1 \wed \fr{1}{\normz{x_t}_{\Sig^{-1}_{t-1}}} \in \lparen0,1\rbrack,
\end{equation*}
where $\blue{\Sig_t} = \lam I + \sum_{s=1}^t w_s^2 x_s x_s^\T$.
Since $\| w_t x_t \|^2_{\Sig_{t-1}^{-1}} \le 1$, using Woodbury matrix identity, we obtain 
\begin{align*}
  \blue{D_t^2} :=  \| w_t x_t \|^2_{\Sig_t^{-1}}\le \fr12~,
\end{align*}
which will be the key property that enables our semi-adaptive bound.
The weighted ridge regression estimator we use is then defined as
\begin{align*}
  \hth_t = \arg \min_{\th} \sum_{s=1}^t w_s^2\ell_s(\th) + \fr{\lam}{2} \normz{\th}^2_2,
\end{align*}
where $\blue{\ell_s(\th)} := \fr12(x_s^\T \th - y_s)^2$ is the squared loss.
Equivalently, one can assume the noise model of
\begin{align*}
  w_t y_t = \la w_t x_t, \th^*\ra + w_t \eta_t,
\end{align*}
and write down the weighted estimator as
\begin{align*}
  \hth_t = \arg \min_{\th }  \sum_{s=1}^t \ellw_s(\th) + \fr{\lam}{2} \normz{\th}^2_2,
\end{align*}
where $\blue{\ellw_s(\th)} := \fr12(w_s x_s^\T \th - w_s y_s)^2 = w_s^2\ell(\th)$.

We then construct the following confidence set
\begin{align} \label{eq:confset}
  &\cC_{t}^\semi = \{\th \in \RR^d:
      \fr12\normz{\hth_{t} - \th}_{\Sig_{t}}^2 \notag
    \\&\quad\quad\quad\quad\le\! \fr{\lam}{2}   S^2 \!+\! \sum_{s=1}^{t} \ellw_s(\hth_{s-1}) D^2_s \!+\! \sig_0^2 \ln(1/\dt) =: \blue{\gam_t}
  \},
\end{align}
and show that the confidence set is time-uniformly valid as follows.
Throughout, all the proofs are deferred to the appendix.
\begin{theorem}\label{thm:confset-semi}
  Take \cref{ass:semi}. Then,
  \begin{align*}
    \PP(\forall t\ge1, \th^* \in \cC^\semi_t) \ge 1-\dt.
  \end{align*}
\end{theorem}
We show later (Appendix~\ref{sec:details-semi}) that, with high probability,
\begin{align*}
  \gam_t = O\del{ \lam S^2 + \sig_*^2 d \ln(1 + \fr{t}{d\lam}) + \sig_0^2 \ln(1/\dt)}  ~,
\end{align*}
which shows an orderwise improvement upon SNCS that has $O\del{ \lam S^2 + \sig_0^2 d \ln(1 + \fr{t}{d\lam}) + \sig_0^2 \ln(1/\dt)}$.
The gap becomes much larger when $\sig_*^2 \ll \sig_0^2$ or even $\sig_*^2 = 0$.

\textbf{Proposed bandit algorithm.~}
We are now ready to present our algorithm, which follows the optimistic approach~\cite{dani08stochastic,ay11improved} with our confidence set $\cC^\semi_t$, which pulls the arm with the largest upper confidence bound.
The full pseudo-code is presented in \cref{alg:semi-adaptive}.
\begin{algorithm}[t]
  \caption{LOSAN (Linear Optimism with Semi-Adaptivity to Noise)}
  \label{alg:semi-adaptive}
  \begin{algorithmic}[1]
    \STATE {\bfseries Input:} norm bound $S$, sub-Gaussian parameter $\sig_0^2$
    \FOR{$t=1,2,\ldots$ }
    \STATE Observe the arm set $\cX_t$.
    \STATE Pull $x_t 
    = \arg \max_{x\in \cX_t} \max_{\th \in \cC^\semi_{t-1}}~ \la x, \th \ra~$
    where $\cC^\semi_{t-1}$ is defined in~\eqref{eq:confset} and
    \begin{align*}
      \max_{\th \in \cC^\semi_{t-1}} \la x, \th \ra = \inp{x}{\hth_{t-1}} + \sqrt{2\gam_{t-1}} \normz{x}_{\Sig^{-1}_{t-1}}~.
    \end{align*}
    \STATE Receive reward $y_t$.
    \ENDFOR
  \end{algorithmic}
\end{algorithm}

\begin{theorem}\label{thm:semi-regret-bound}
  Let $B = \max_{t=1}^\infty \max_{x\in \cX_t} |\inp{x}{\th^*}|$.
  Under \cref{ass:semi}, \cref{alg:semi-adaptive} with $\lam = \sig_0^2/S^2$ satisfies that, with probability at least $1-O(\dt)$, $\forall n\ge1$,
  \begin{align*}
    \Reg_n \lsim \sig_* d \sqrt{n} + \sig_0 \sqrt{dn\ln(1/\dt)}+ dB,
  \end{align*}
  where we omit $\polylog(d,n,S,\sig^2_0)$ factors.
\end{theorem}
\cref{thm:semi-regret-bound} shows an improvement upon the state-of-the-art regret bound of $\sig_0 d\sqrt{n\ln(1/\dt)} + dB$ which is reported in~\citet[Exercise 19.3]{lattimore20bandit}.
When the sub-Gaussian parameter is largely over-specified (i.e., $\sig_0^2 \gg \sig_*^2$), our algorithm achieves a factor of $\sqrt{d}$ improvement, which is significant when $d$ is not too small.
To our knowledge, our confidence set and LOSAN are the first ones to achieve the semi-adaptivity under the sub-Gaussian noise assumption.

\section{Full Adaptation to Bounded Noise}
\label{sec:fully}

In this section, we turn to the bounded noise case.
Specifically, we assume the model of~\eqref{eq:model} with the following noise assumption.
\begin{assumption}\label{ass:fully}
  The reward noise $\eta_t$ satisfy $\eta_t\in[-R,R]$ for some $R \ge 0$ with probability 1 for every $t$.
  The algorithm has a priori knowledge of $R$.
\end{assumption}

Of recently proposed studies on variance-adaptive linear bandits (see \cref{sec:related} for more discussion), \citet{zhao23variance} for the first time proposed an algorithm called SAVE that enjoys $\tcO(d\sqrt{(R^2 + \sum_{t=1}^n \sig_t^2)})$ with polynomial time and space complexity.
In fact, its space and time complexity is of the same order (up to logarithmic factors) as the standard linear bandits such as OFUL~\cite{ay11improved}, which is $O(d^2 \max_{t=1}^n|\cX_t|  n)$ time complexity and $O(d^2)$ space complexity.
SAVE has improved both the regret bound (a factor of $\sqrt{d}$) and the time complexity (exponential to polynomial) upon the previous state-of-the-art~\citet{kim22improved}.

However, SAVE is inherently based on SupLinRel~\cite{auer02using} or SupLinUCB~\cite{chu11contextual}, which maintains $L$ disjoint buckets of the observed samples.
Since each estimator $\hth_\ell$ is computed from the samples from $\ell$-th bucket only, SupLinRel-style algorithms are usually an order of magnitude worse than the standard algorithmic frameworks such as optimistic approach~\cite{ay11improved} or posterior sampling approach~\cite{agrawal13thompson}.

\textit{Is it possible to achieve the same order of regret bound and computational complexity without wasting samples?}
We answer this question in the affirmative by developing an optimistic-style algorithm.
As is usual, we first need to construct a confidence set for the unknown parameter $\th^*$.
While one can leverage the existing confidence set used for SAVE, it works under the assumption where the conditional variance of $\eta_t \mid \cF_{t-1}$ at time $t$ is fixed ahead of time before the bandit game starts.
That is, they assume that there exists a sequence of deterministic values $\barsig^2_1,\ldots,\barsig^2_n$ such that
\begin{align}\label{eq:zhao23-assumption}
  \forall t\in[n], (\sig_t^2 \mid \cF_{t-1}) = \barsig^2_t~.
\end{align}
We found this unrealistic as it can only deal with the noise that varies only as a function of the time step $t$ (e.g., a seasonal effect of the customers' behavior in recommendation systems).
For example, such an assumption cannot effectively capture the case where the noise variance changes as a function of the specific arm being pulled.
Furthermore, we have found that the variance-adaptive confidence set used in SAVE is quite loose, which requires a lot of samples until outperforms even SNCS; see \cref{fig:conf-set}.

Motivated by the limitations of the sample-inefficiency of the prior art, we propose a novel confidence set and bandit algorithm that remove the limited assumption on the variance and exhibit much improved numerical performance.
Our proposed confidence set computes $L$ estimators and builds a confidence set as an intersection of $L$ base confidence sets.
Inspired by the base confidence set of SAVE~\cite{zhao23variance}, we construct our base confidence sets by leveraging weighted ridge regression estimators, but with a critical difference that (i) we leverage the regret equality (\cref{lem:regret_equality}), which results in a significantly tightened confidence set (see \cref{fig:conf-set} for numerical results) and (ii) we use an exponential cover to remove the restrictive assumption~\eqref{eq:zhao23-assumption} and adapt to any conditional variance, and (iii) compute a set of secondary estimator that will be the center of our confidence ellipsoid.

\begin{figure}[t]
  \centering
  \includegraphics[width=0.5\linewidth]{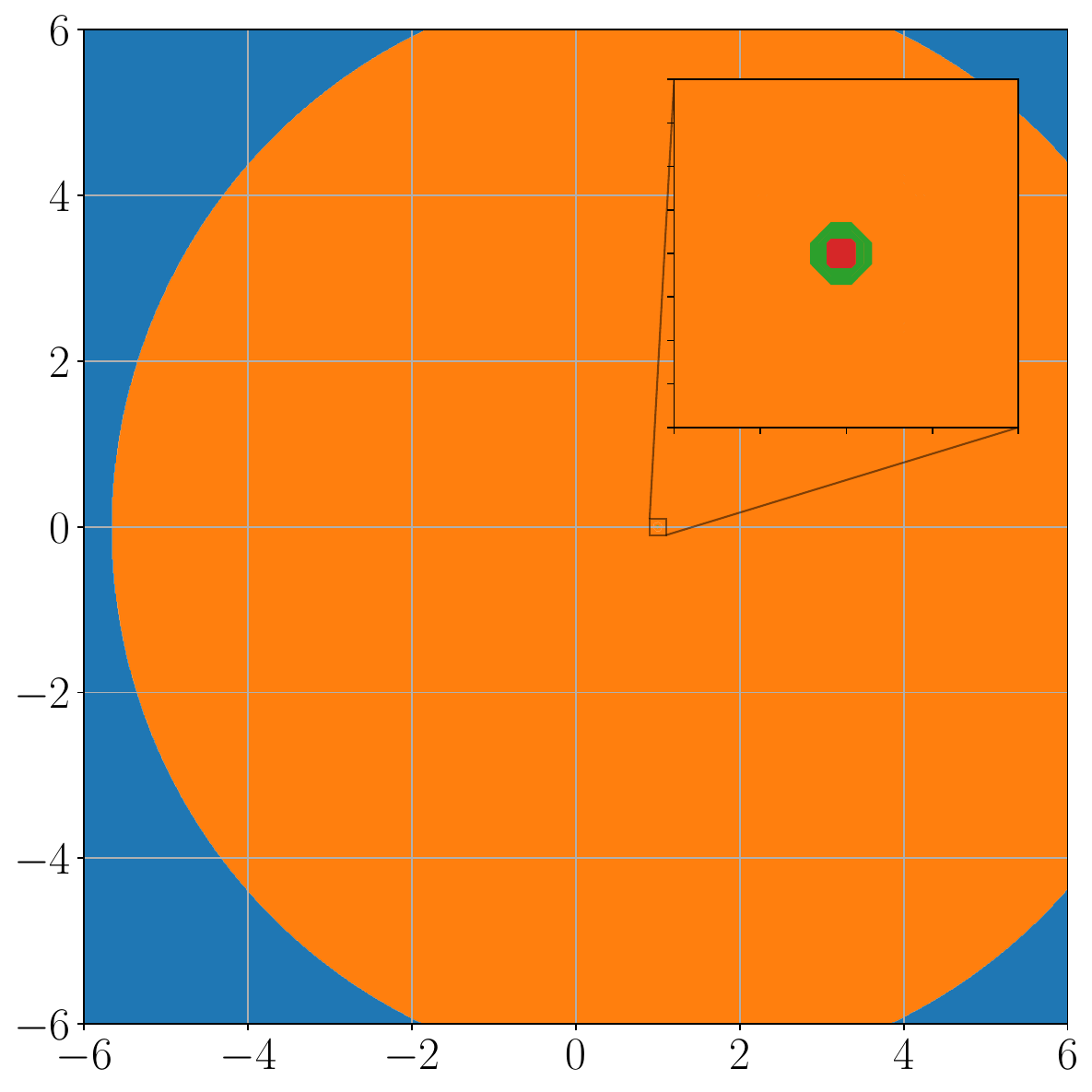}
  \caption{The green line (very small) represents our confidence set $\cC^\full_t$ and the orange area represents the confidence set of~\citet{zhao23variance}, which is also implemented as an intersection of $L$ confidence sets like ours.
    We use $n=$ 500,000 samples, $d=2$, $L=9$, and $\sig_t^2=0.1,\forall t$. 
    With $\th^* = (1,0)$, the upper confidence bound on the mean reward of the arm $x=(1,0)$ is $1.01$ with our method while it is $7.66$ with their method and $1.05$ with SNCS.}
  \label{fig:conf-set}
\end{figure}

\paragraph{Proposed confidence set.}
Fix $L \in \NN_+$ and let $\ell \in [L]$.
We let $\blue{\rho_\ell} = 2^{-\ell}$ and set the regularization parameter $\blue{\lam_\ell} = \fr{R^2}{S^2} \rho_\ell^2$.
We set the weights 
\begin{align*}
  \blue{w_{s,\ell}} = 1 \wed \fr{\rho_\ell}{\normz{x_s}_{\Sig^{-1}_{s-1,\ell} } }~,
\end{align*}
where
\begin{align}\label{eq:Sig_t_ell}
  \blue{\Sig_{t,\ell}} = \lam_\ell I + \sum_{s=1}^{t} w^2_{s,\ell} x_s x_s^\T ~.
\end{align}
The weight ensures that $\normz{w_{t,\ell} x_t}_{\Sig^{-1}_{t-1,\ell} } \le \rho_\ell$.
Define $\ellw_{s,\ell}(\th) := w^2_{s,\ell}\ell_s(\th)$ and $\ellw_{s,\ell}(\th; \th') = \fr{w^2_{s,\ell}}{2}(x_s^\T(\th - \th'))^2$.
We then compute the weighted ridge regression estimator:
\begin{align*}
  \blue{\hth_{t,\ell}} 
  = \arg \min_{\th} L_{t,\ell}(\th),
\end{align*}
where
\begin{align*}
  \blue{L_{t,\ell}(\th)} = \sum_{s=1}^t \ellw_{s,\ell}(\th) + \fr{\lam_\ell}{2}  \normz{\th}^2_2  ~.
\end{align*}
We then compute our secondary estimators as follows:
\begin{align}
  \blue{K_{t,\ell}(\th)} &:= \sum_{s=1}^t \ellw_{s,\ell}(\th) + \sum_{s=1}^t \ellw_{s,\ell}(\th; \hth_{s-1}) + \fr{\lam_\ell}2 \normz{\th}^2, \notag
  \\  \blue{\barth_{t,\ell}} &:= \arg \min_{\th} K_{t,\ell}(\th), \label{eq:barth}
\end{align}
With $\blue{\barSig_{t,\ell}} := \lam_\ell I + 2\sum_{s=1}^t w^2_{s,\ell}  x_s x_s^\T$, we define our confidence set as an intersection of confidence ellipsoids centered at $\barth_{t,\ell}$ as follows:
\begin{align}\label{eq:confset-fa}
  \cC^\full_t := \cap_{\ell=1}^{L} \cC^\full_{t,\ell},
\end{align}
where
\begin{align}
  \forall \ell\in[L], \blue{\cC^\full_{t,\ell}} = \cbr{\th\in\RR^d: \fr12 \normz{\th - \barth_{t,\ell}}_{\barSig_{t,\ell}}^2 \le \beta_{t,\ell}}, \notag
\end{align}
and $\beta_{t,\ell}$ is defined based on its previous version $\tblue{\beta_{t-1,\ell}}$. 
Specifically, with $\blue{D_{t,\ell}^2} := \normz{w_{t,\ell} x_t}^2_{\Sig_{t,\ell}^{-1}}$,
\begin{align}
  \blue{\beta_{t,\ell}} 
  &\!:=\! L_{t,\ell}(\hth_{t,\ell}) \!-\! K_{t,\ell}(\barth_{t,\ell}) \!+\! \fr{\lam_\ell} 2 S^2 \!+\! \sum_{s=1}^{t} \ellw_{s,\ell}(\hth_{s-1,\ell}) D_{s,\ell}^2 \notag
  \\&~ +\sqrt{8\rho_\ell^2\tblue{\barbeta_{t-1,\ell}}\del{\tsty\sum_{s=1}^t \ellw_{s,\ell}(\hth_{s-1}) + R^2\ln(2L/\dt)} \xi_{t,\ell} }  \notag
  \\&~ + 2^{k_{t,\ell}} \rho_\ell R\sqrt{2\beta_{0,\ell}}\xi_{t,\ell},\label{eq:beta_t_ell}
\end{align}
where $\tblue{\barbeta_{t-1,\ell}} := \max_{s=0}^{t-1} \tblue{\beta_{s,\ell}}$, $\blue{\beta_{0,\ell}} := \fr{\lam_\ell}{2}S^2$, $\blue{\xi_{t,\ell}} := \ln( \sqrt{\pi(t+1)} \cd \fr{ 6.8L\cd k_{t,\ell} \ln^2(1+k_{t,\ell}) }{\dt})$, and $\blue{k_{t,\ell}} := 1 \vee \lcl \log_2(\sqrt{\tblue{\barbeta_{t-1,\ell}}/\beta_{0,\ell}}) \rcl$.

Our confidence set enjoys the following correctness guarantee.
\begin{theorem}
  \label{thm:confset-fa}
  Under \cref{ass:fully}, we have
  \begin{align*}
    \PP(\forall t\ge1, \th^* \in \cC^\full_t) \ge 1-\dt.
  \end{align*}
\end{theorem}
As one will see from the proof, the introduction of the secondary estimator $\barth_{t,\ell}$ is a nonessential part of the theoretical guarantee as it only helps the numerical tightness.
The secondary estimator only appears during the attempt to write down the confidence set as a canonical quadratic form -- such an estimator also appears for the same reason in~\citet{ay12online}.
One can easily see that it is possible to use $\hth_{t,\ell}$ directly with a slightly looser confidence set.

We later show that (\cref{lem:beta_t-true-variance}) 
\begin{align*}
  \beta_{t,\ell} \le \barbeta_{t,\ell} = \tcO(\rho_\ell^2 (R^2 + \sum_{s=1}^t \sig_s^2) \ln^2(t/\dt))~,
\end{align*}
which scales with the conditional variances as desired and matches the order of the confidence set of \citet{zhao23variance}.
Note that due to the particular weights being used, each confidence set $\ell$ is tight in certain regimes only -- this is the reason why we take an intersection over $L$ of them.

\begin{algorithm}[t]
  \caption{LOFAV (Linear Optimism with Full Adaptivity to Variance)}
  \label{alg:fully-adaptive}
  \begin{algorithmic}[1]
    \STATE {\bfseries Input:} norm bound $S$, time horizon $n$, the number of levels $L$
    \FOR{$t=1,2,\ldots,n$ }
    \STATE Observe the arm set $\cX_t$.
    \STATE Pull $x_t 
    = \arg \max_{x\in \cX_t} \max_{\th \in \cC^\full_{t-1}}~ \la x, \th \ra~$
    where 
    \begin{align*}
      \max_{\th \in \cC^\full_{t-1}} \!\! \la x, \th \ra \! = \! \min_{\ell \in [L]} \inp{x}{\barth_{t-1,\ell}} \!+\!  \sqrt{2\beta_{t-1,\ell}} \normz{x}_{\barSig^{-1}_{t-1,\ell}}.
    \end{align*}
    and $\cC^\full_{t-1}$, $\barth_{t-1,\ell}$, $\beta_{t-1,\ell}$, and $\barSig_{t-1,\ell}$ are defined in~\eqref{eq:confset-fa}, \eqref{eq:Sig_t_ell}, \eqref{eq:barth}, and \eqref{eq:beta_t_ell}, respectively.
    \STATE Receive reward $y_t$.
    \ENDFOR
  \end{algorithmic}
\end{algorithm}

\paragraph{Proposed bandit algorithm.}

Equipped with our improved confidence set, we construct an OFUL-style algorithm, which we call \textbf{LOFAV} (Linear Optimism with Full Adaptivity to Variance); see \cref{alg:fully-adaptive}.
The time and space complexity is the same as OFUL, i.e., $O(d^2n\max_{t=1}^n|\cX_t|)$, up to logarithmic factors when setting $L = O(\ln(n/d))$ as suggested in \cref{thm:fully-regret-bound} below.
For efficient implementation, one needs to maintain sufficient statistics for the estimators $\hth_{t,\ell}$ and $\barth_{t,\ell}$, and update the inverse of the covariance matrices $\Sig^{-1}_{t,\ell}$ and $\barSig^{-1}_{t,\ell}$ using the matrix inversion lemma.
Evaluating the loss functions such as $L_{t,\ell}(\hth_{t,\ell})$ can be done incrementally as well by expanding the square and extracting sufficient statistics.

We report our regret analysis result in \cref{thm:fully-regret-bound} below, which attains the optimal variance-adaptive regret bound up to logarithmic factors.
\begin{theorem}\label{thm:fully-regret-bound}
  Under \cref{ass:fully}, \cref{alg:fully-adaptive} with $\dt \in \lparen 0,1/2 \rbrack$ and $\blue{L} = 1 \vee \lcl \fr12 \log_2 (n/d) \rcl$ achieves, with probability at least $1-\dt$, 
  \begin{align*}
    \Reg_n \lsim d\sqrt{\del[2]{R^2 + \sum_{t=1}^n \sig_t^2}\ln^2\del[2]{\fr{1}{\dt} }} + dB,
  \end{align*}
  where we omit $\polylog(d,n,S,R)$ factors.
\end{theorem}
The regret analysis is challenging since, unlike SAVE and other variance-adaptive algorithms such as LinNATS~\cite{xu23noise}, LOFAV does not perform the stratification of the arms that provides an easy control of how many samples fall in each bucket, which is directly related to the regret in a systematic manner.
That is, each bucket $\ell$ is explicitly associated with a particular instantaneous regret bound of order $2^{-\ell}$).
We overcome this difficulty by turning to the peeling-based regret analysis~\cite{kim22improved,he21logarithmic}, which provides a strong control on the instantaneous regret when used with the elliptical potential `count' lemma \cite{kim22improved} that we restate in \cref{lem:epc}.

\paragraph{Practical version.}
The performance of LOFAV can be further improved by adding extra $L$ confidence sets without further splitting the target failure rate $\dt$.
This is because there is an event being assumed in the proof of confidence set $\cC^\full_t$ that automatically implies a semi-adaptive style confidence set (one set for each $\ell\in[L]$).
The reason why it helps is that there are extra logarithmic factors and constants in $\cC^\full_t$, which can only be overcome after observing a large number of samples.
Equipping LOFAV with semi-adaptive style confidence sets helps avoid excessive exploration for the small-sample regime.
We provide a precise description of this modification in \cref{sec:lofav-implementation}.

\paragraph{Anytime version.}
The current version of LOFAV requires knowledge of the time horizon $n$.
This can be easily lifted by changing the union bound over $L$ confidence sets into a union bound over $\NN_+$.
Specifically, one can divide the failure rate $\dt$ into $\{\dt_\ell\}_{\ell=1}^\infty$ such that $\sum_{\ell=1}^\infty \dt_\ell = \dt$; e.g., $\dt_\ell = \fr{\dt}{\ell^2}\cd \fr{6}{\pi^2}$.
The algorithm needs to maintain only $L_t = \Theta(\ln(t/d))$ confidence sets up to time step $t$.

\begin{figure}[t]
    \centering
    \subfigure[Sub-Gaussian noise]{
        \includegraphics[width=0.225\textwidth]{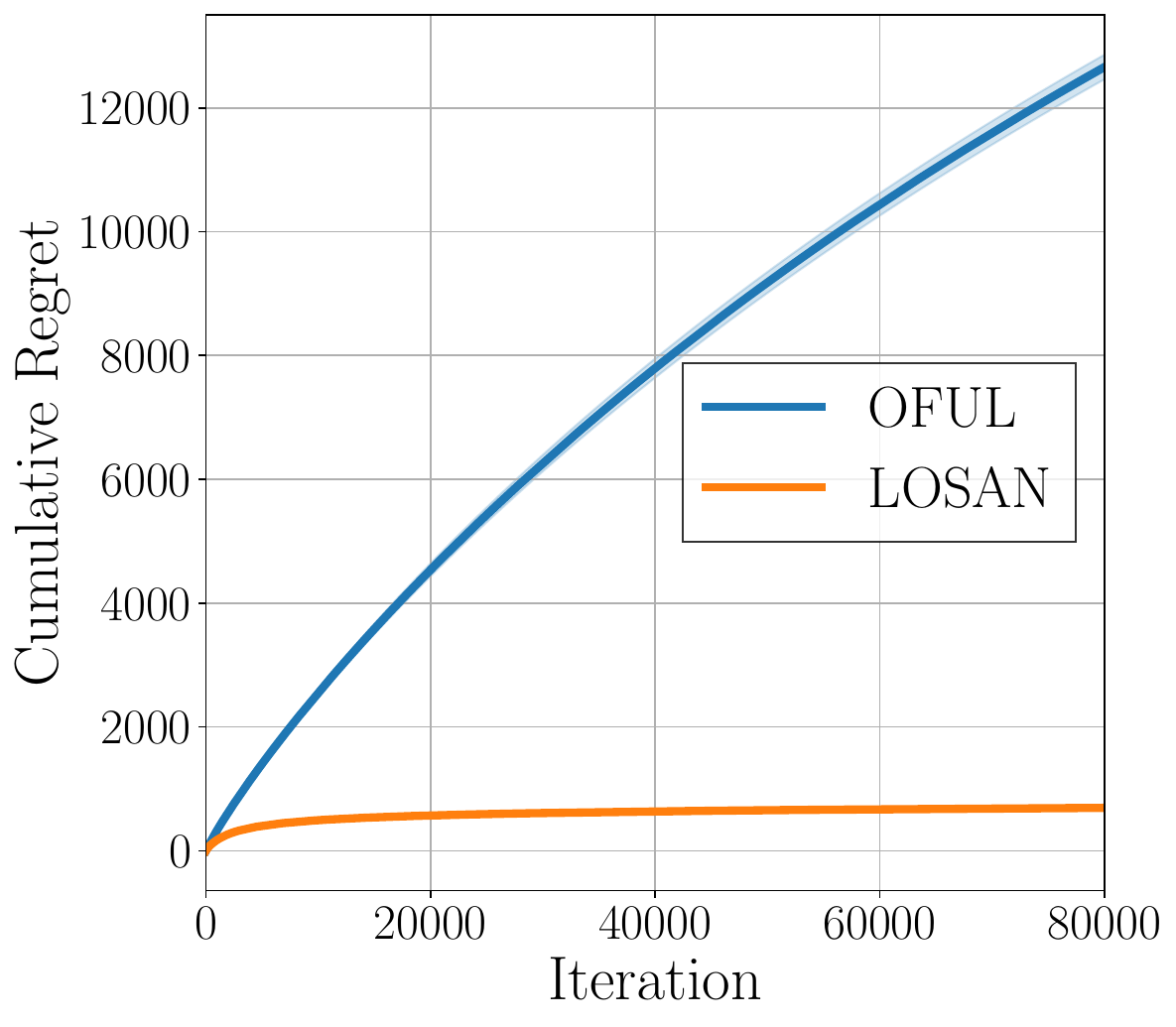}
    }
    \subfigure[Bounded noise]{
        \includegraphics[width=0.225\textwidth]{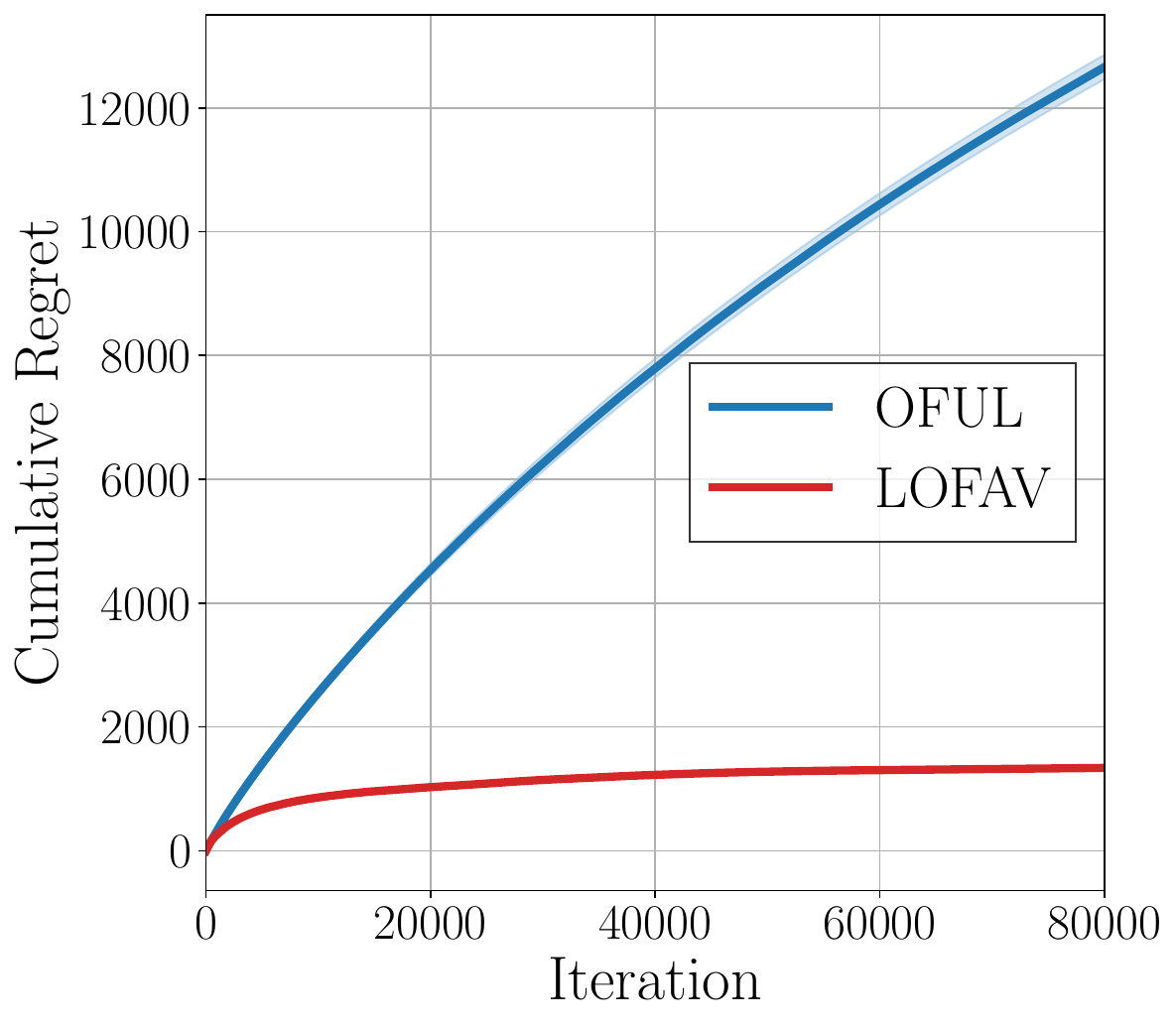}
    }
    \caption{Results of synthetic experiments with LOSAN and LOFAV. To fairly compare our algorithms to OFUL, we perform each experiment over 50 rounds where $S = 1.0$, $d = 32$, $|\cX_t| = 128$, and $\sigma_0 \ \textrm{or} \ R = 1.0$.}
    \label{fig:synthetic}
\end{figure}

\section{Experiments}
\label{sec:expr}

\begin{figure*}[t]
    \centering
    \subfigure[Beale]{
        \includegraphics[width=0.225\textwidth]{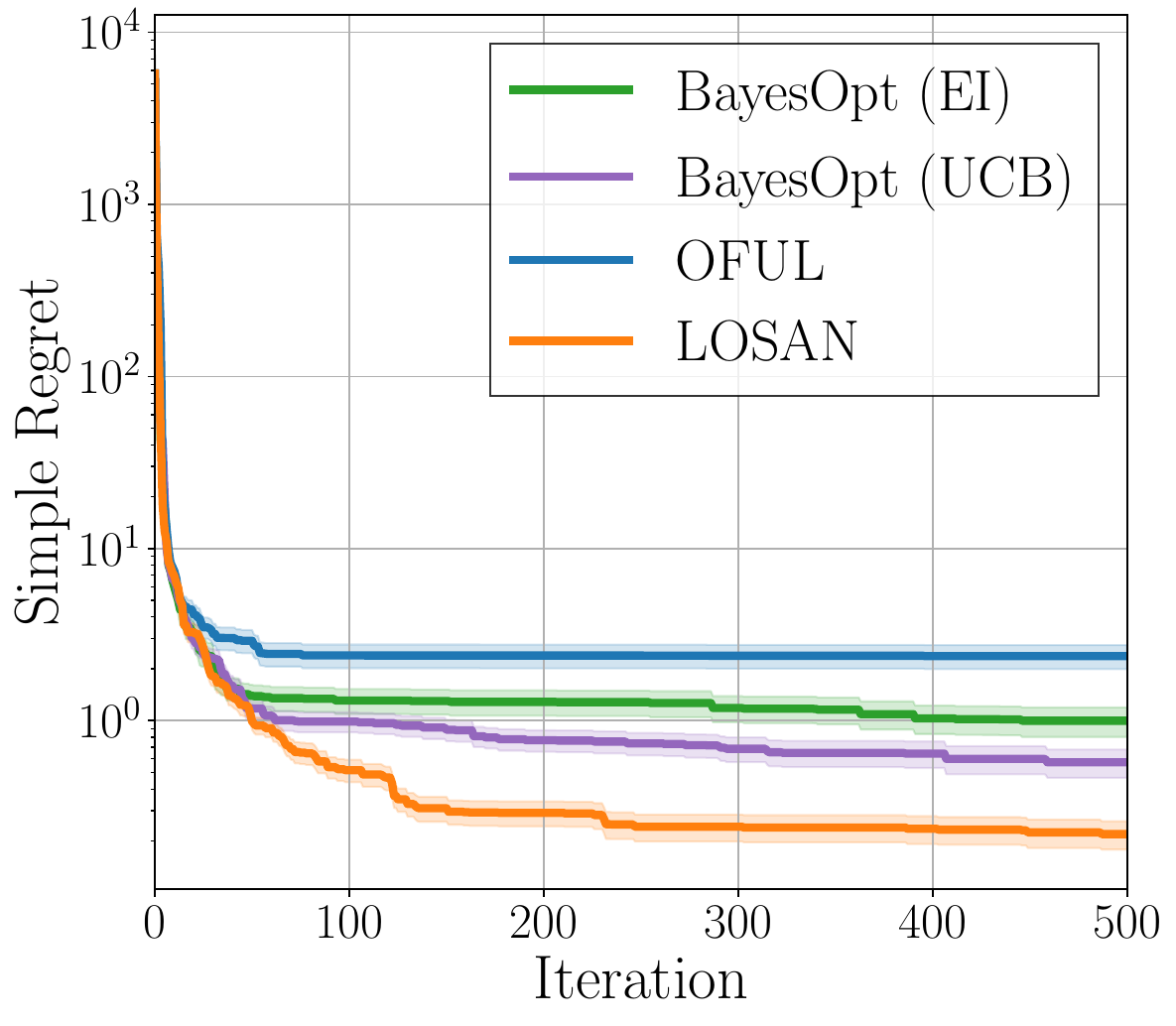}
    }
    \subfigure[Branin]{
        \includegraphics[width=0.225\textwidth]{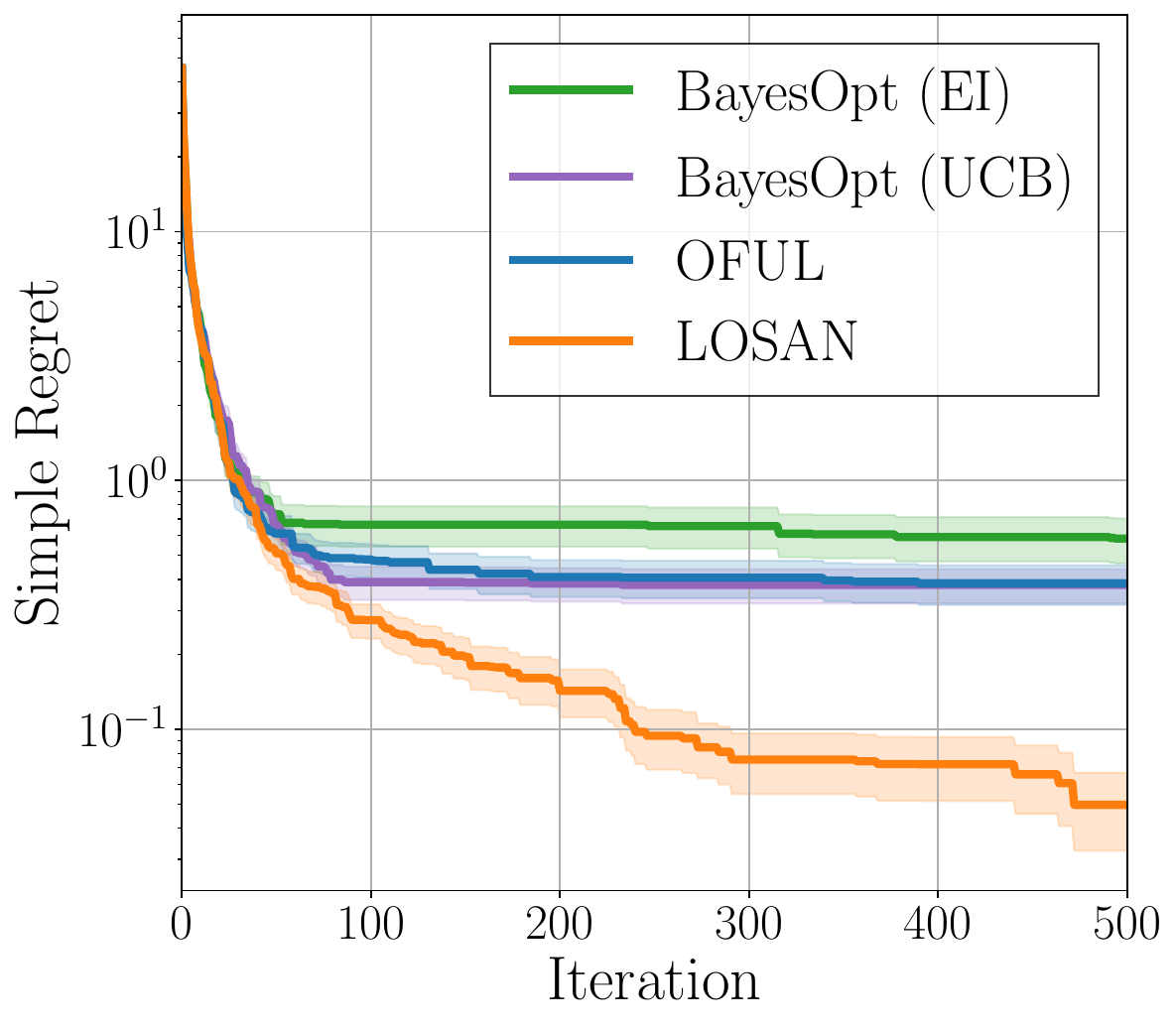}
    }
    \subfigure[Three-Hump Camel]{
        \includegraphics[width=0.225\textwidth]{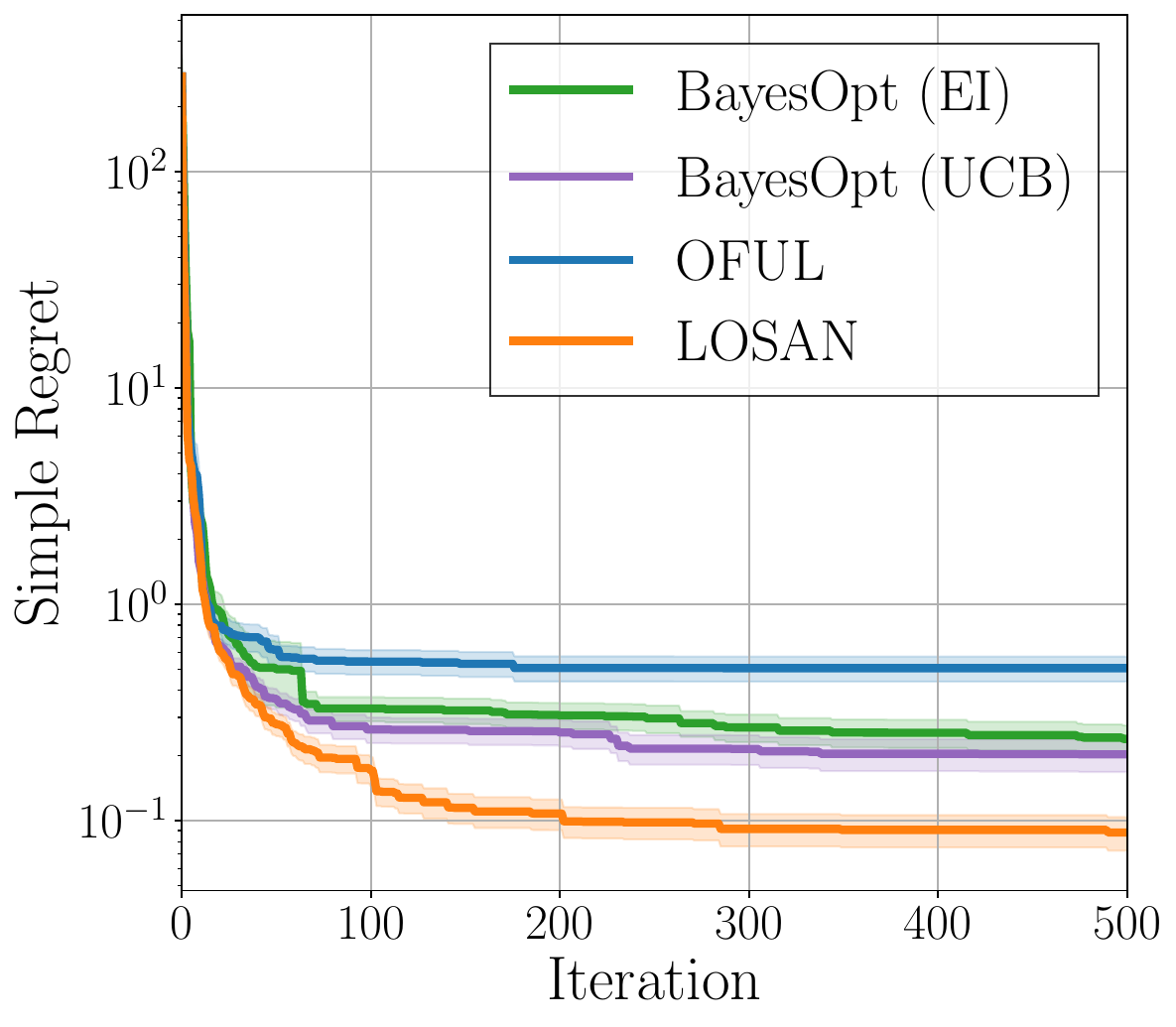}
    }
    \subfigure[Zakharov 4D]{
        \includegraphics[width=0.225\textwidth]{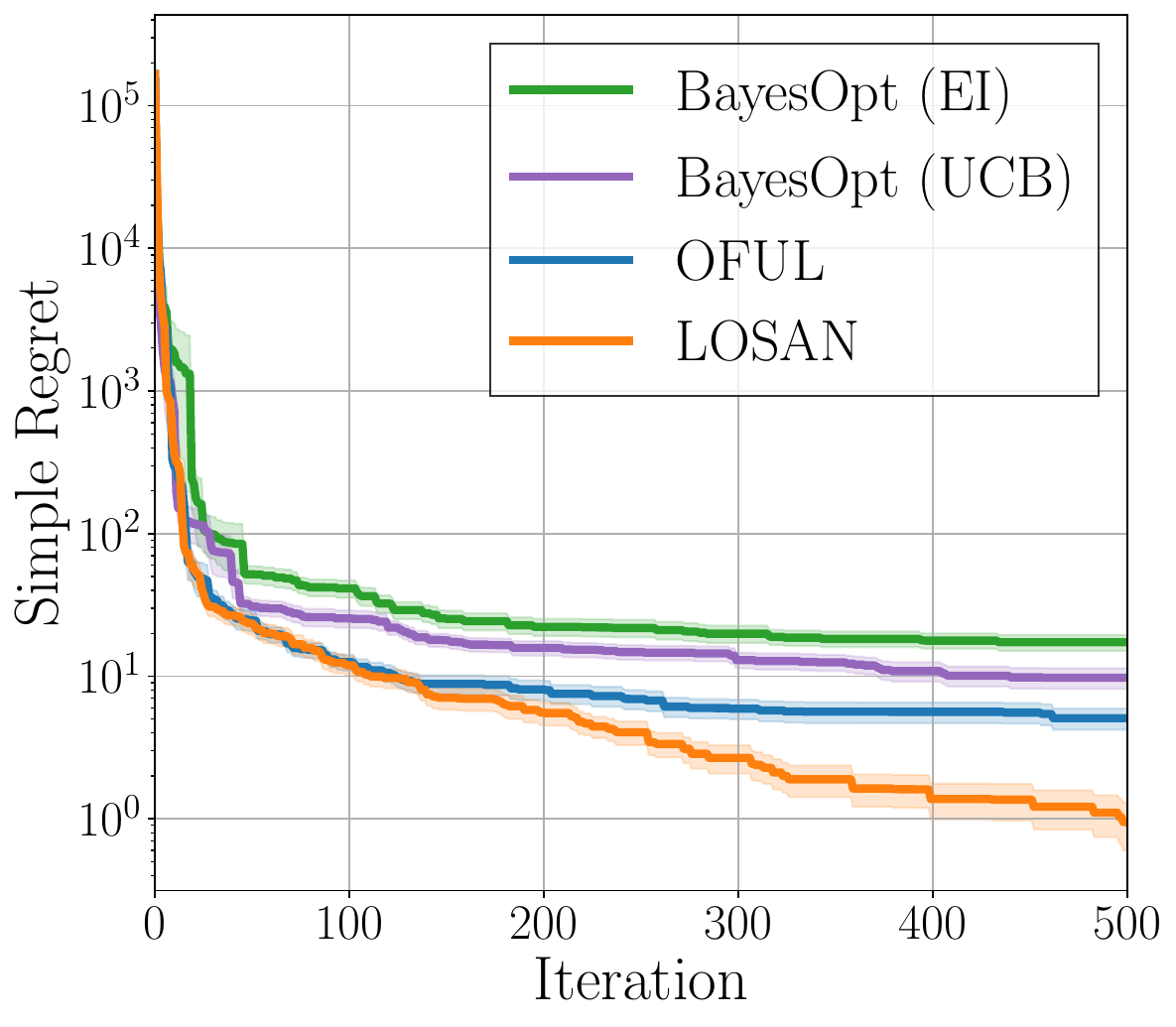}
    }
    \caption{Bayesian optimization results of LOSAN with random Fourier features and sub-Gaussian noises for four benchmark functions. We perform each experiment over 50 rounds where $S= 1.0$, $d = 128$, $|\cX_t| = 512$, and $\sigma_0 = 1.0$.}
    \label{fig:bo_benchmarks_gaussian}
\end{figure*}

\begin{figure*}[t]
    \centering
    \subfigure[Beale]{
        \includegraphics[width=0.225\textwidth]{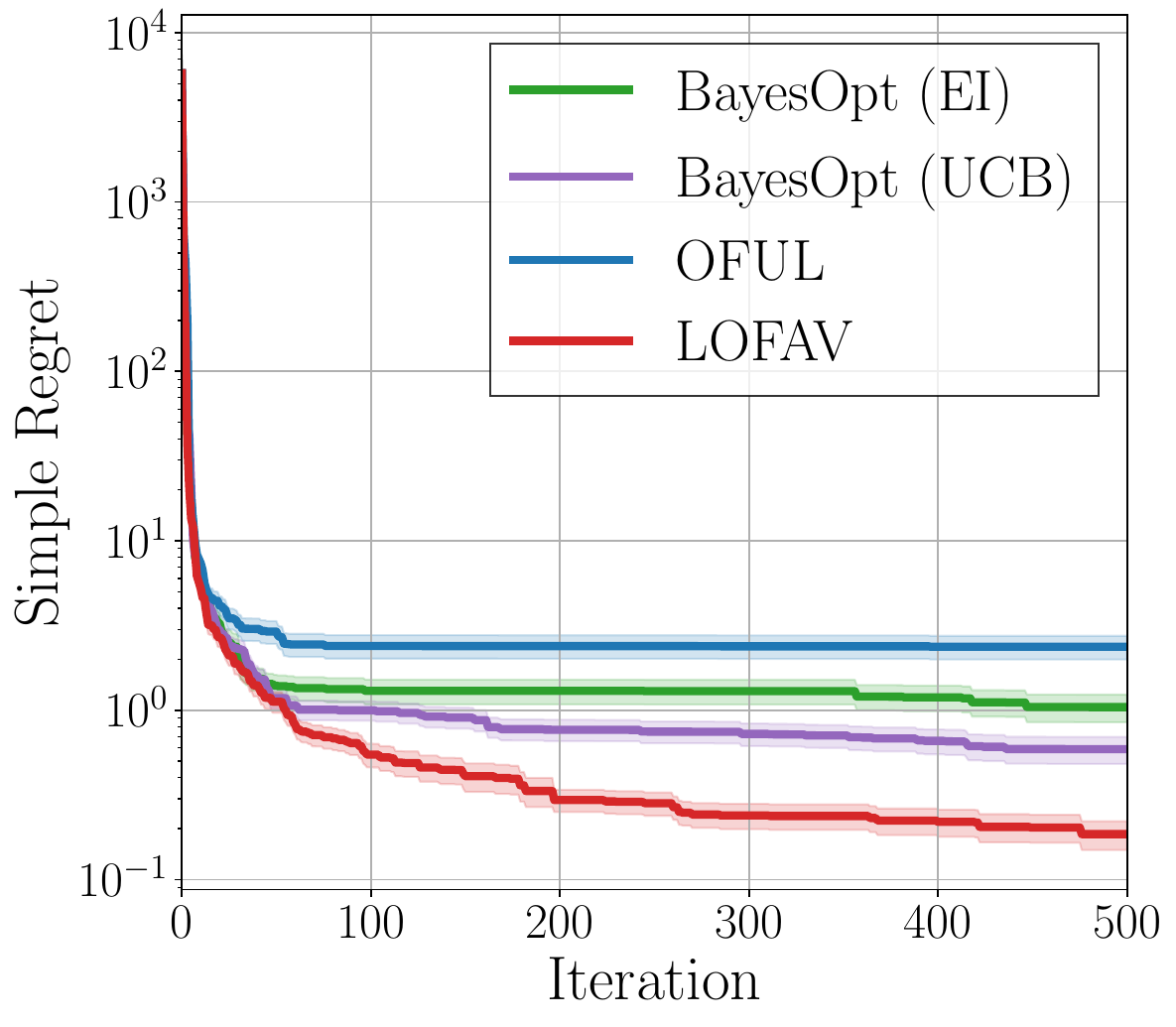}
    }
    \subfigure[Branin]{
        \includegraphics[width=0.225\textwidth]{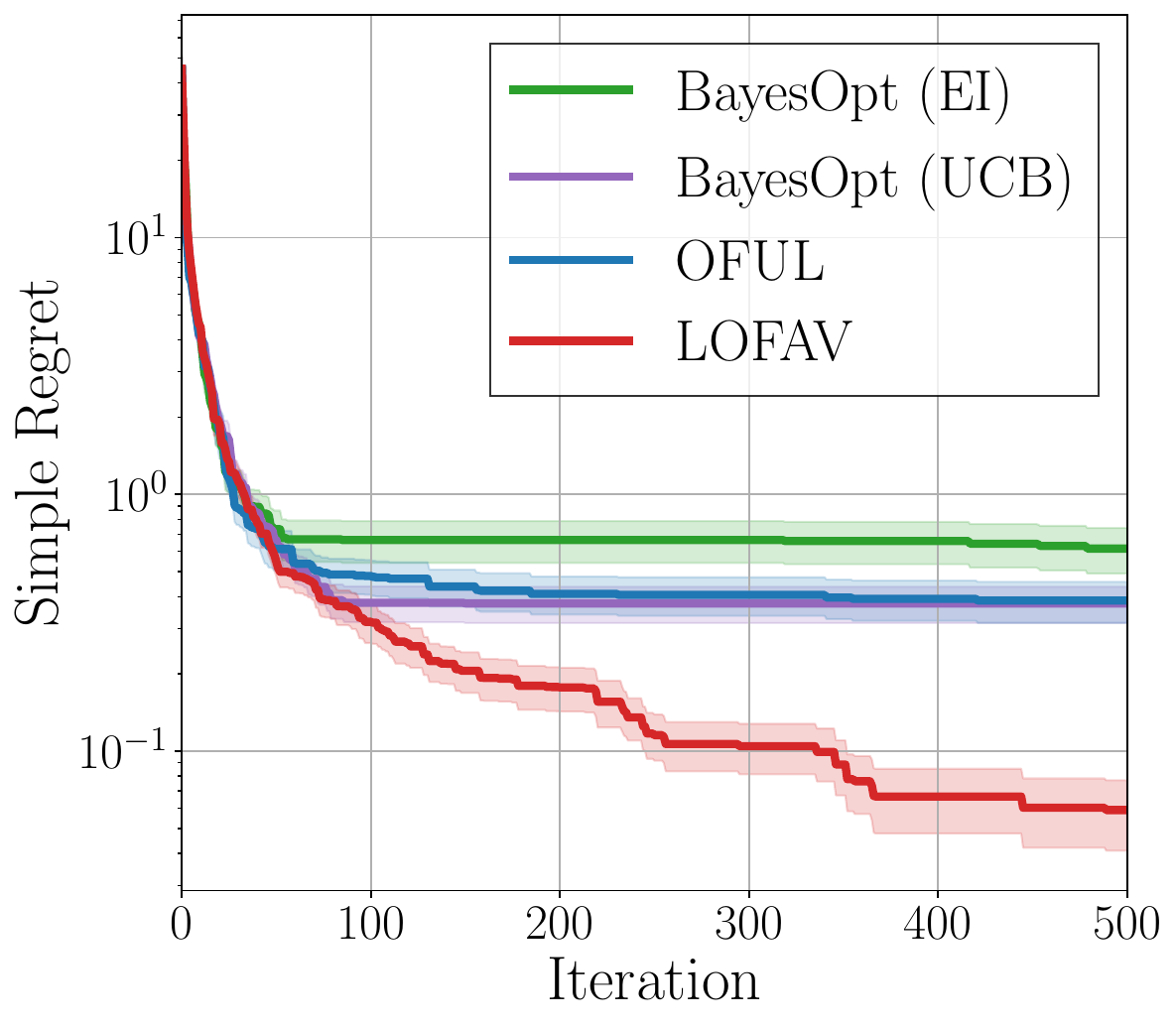}
    }
    \subfigure[Three-Hump Camel]{
        \includegraphics[width=0.225\textwidth]{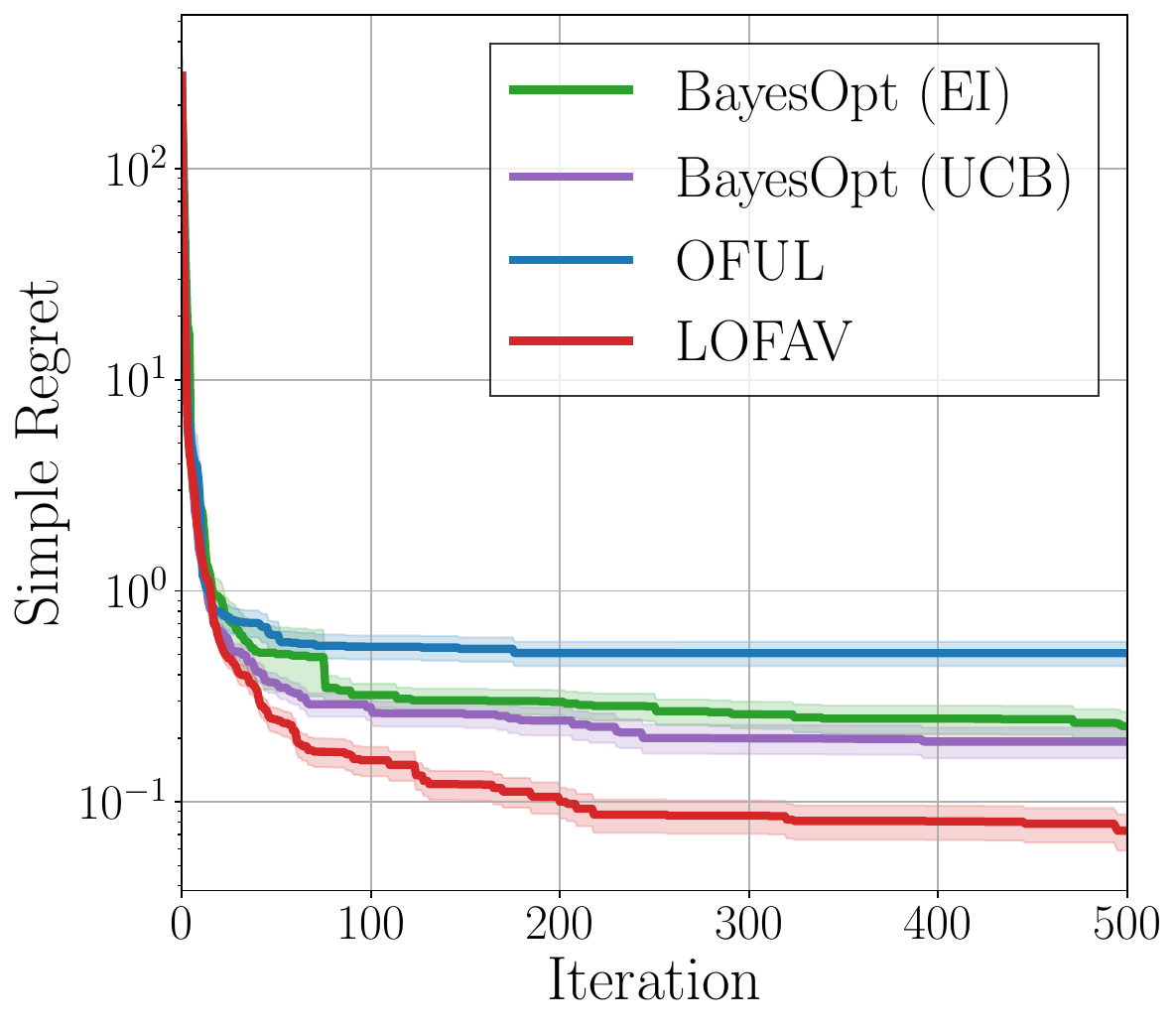}
    }
    \subfigure[Zakharov 4D]{
        \includegraphics[width=0.225\textwidth]{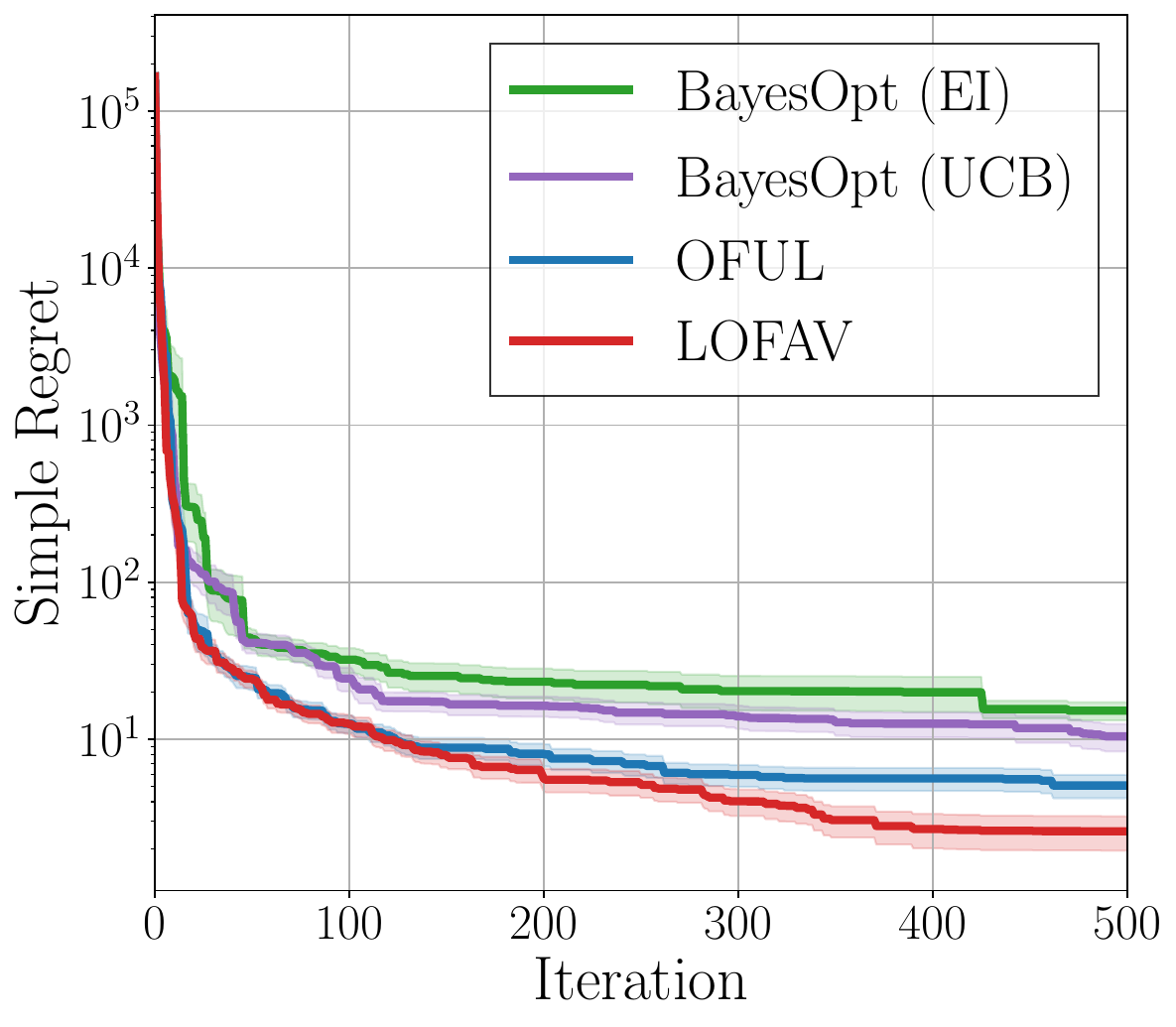}
    }
    \caption{Bayesian optimization results of LOFAV with random Fourier features and bounded noises for four benchmark functions. We perform each experiment over 50 rounds where $S= 1.0$, $d = 128$, $|\cX_t| = 512$, and $R = 1.0$.}
    \label{fig:bo_benchmarks_rademacher}
\end{figure*}

We conduct several experiments with two noise types,
i.e., the sub-Gaussian and bounded noises,
to verify the theoretical analysis of our algorithms.
First,
synthetic experiments are tested to show empirical results of our methods and OFUL in terms of cumulative regrets.
Second, the experiments of Bayesian optimization,
i.e., simple benchmark functions and neural architecture search benchmarks,
are carried out to compare our algorithms to OFUL and simple Bayesian optimization algorithms.
Unless noted otherwise, we perform each experiment with 50 random trials.
Moreover, the sample mean and the standard error of the sample are depicted in~\cref{fig:synthetic,fig:bo_benchmarks_gaussian,fig:bo_benchmarks_rademacher,fig:bo_natsbench}.
For sub-Gaussian and bounded noises, we assume that the true noises are sampled from the distributions with $\sig_*^2=0.01^2$ (Guassian) and $\sig_t^2 = 0.01^2$ ($\eta \sim \mathrm{Uniform}\{-0.01,0.01\}$) respectively. 
For LOFAV, we implement a practical version discussed in~\cref{sec:fully}.

\subsection{Synthetic Experiments}
\label{sec:synthetic}

To generate synthetic experiments, we sample an unknown parameter $\th^*$ and an arm set $\cX_t$ from multivariate normal distributions with zero mean and unit variance and then normalize them in order to locate them on a hypersphere of radius $S$.  
As presented in~\cref{fig:synthetic}, our methods LOSAN and LOFAV exhibit better performance in terms of cumulative regrets than OFUL.
We perform a more comprehensive experiments in Appendix~\ref{sec:expr-more}.

\subsection{Application to Bayesian Optimization}
\label{sec:bayesopt}

We expand the application of our linear bandit models to Bayesian optimization~\citep{GarnettR2023book}.
Along with OFUL, we compare our methods to simple Bayesian optimization algorithms with Gaussian process regression~\citep{RasmussenCE2006book} and either expected improvement acquisition function~\citep{JonesDR1998jgo} or Gaussian process upper confidence bound~\citep{SrinivasN2010icml}.
For the Gaussian process regression, we use a linear kernel where kernel parameters are sought by marginal likelihood maximization.
We select this linear kernel for fair comparison to linear bandit models.
Multi-start L-BFGS-B~\citep{ByrdRH1995siamjsc} is utilized in the process of acquisition function optimization.
To deal with a fixed number of arms with Bayesian optimization,
we modify the standard Bayesian optimization algorithm to choose the nearest arm after determining the next point through Bayesian optimization~\citep{GarridoEC2020neucom}.

For LOSAN, LOFAV, OFUL, and Bayesian optimization approaches,
we make use of random Fourier features~\citep{RahimiA2007neurips}
in order to solve Bayesian optimization problems using linear models or Gaussian process regression with the linear kernel.
Each original point is transformed into a 128-dimensional random feature
following the work by~\citet{RahimiA2007neurips}.
Other configurations for these algorithms are the same as the configurations used in~\cref{sec:synthetic}.
For arm selection,
we uniformly sample a fixed number of arms from a specific search space depending on benchmarks.
Moreover,
instead of cumulative regret,
we use simple regret as the performance measure in the experiments of Bayesian optimization tasks
where the simple regret at time $t$ is set to be the best instantaneous regret until $t$.

\paragraph{Benchmark functions.}

We test four benchmark functions:
Beale, Branin, Three-Hump Camel,
and Zakharov 4D functions.
As illustrated in~\cref{fig:bo_benchmarks_gaussian,fig:bo_benchmarks_rademacher},
our LOSAN and LOFAV are better than OFUL and Bayesian optimization.

\begin{figure*}[t]
    \centering
    \subfigure[CIFAR-100 (Sub-Gaussian)]{
        \includegraphics[width=0.225\textwidth]{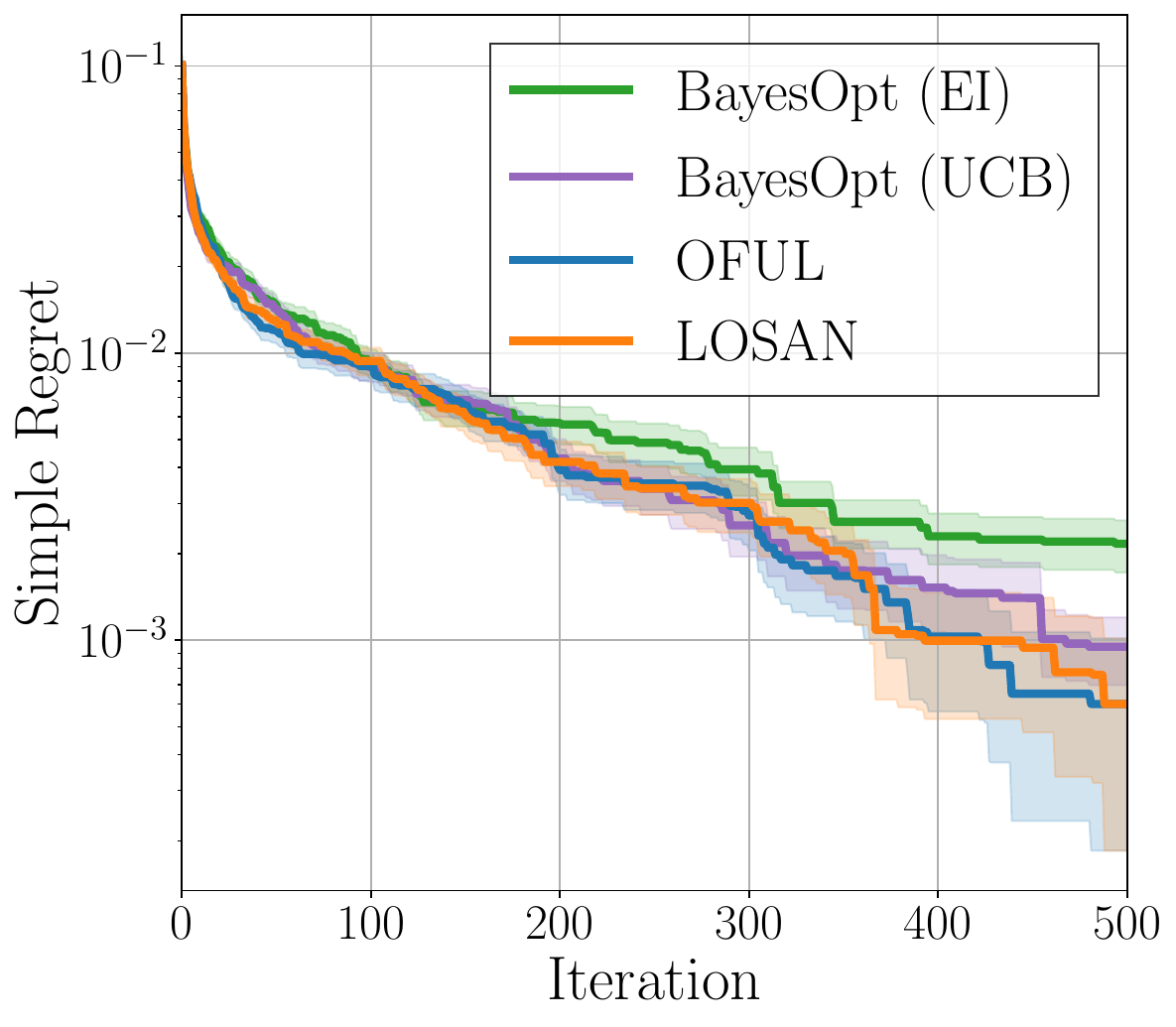}
    }
    \subfigure[CIFAR-100 (Bounded)]{
        \includegraphics[width=0.225\textwidth]{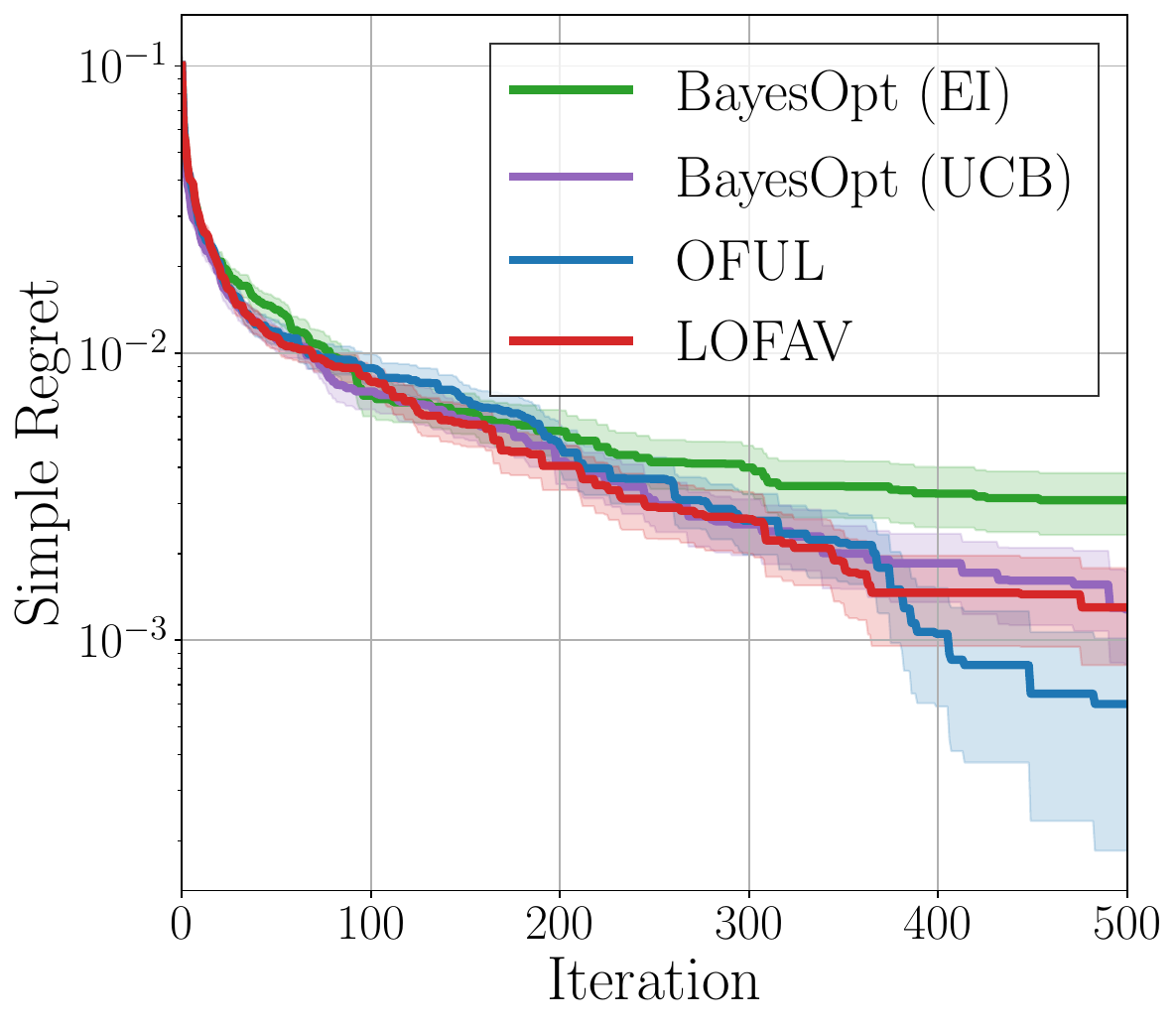}
    }
    \subfigure[ImageNet (Sub-Gaussian)]{
        \includegraphics[width=0.225\textwidth]{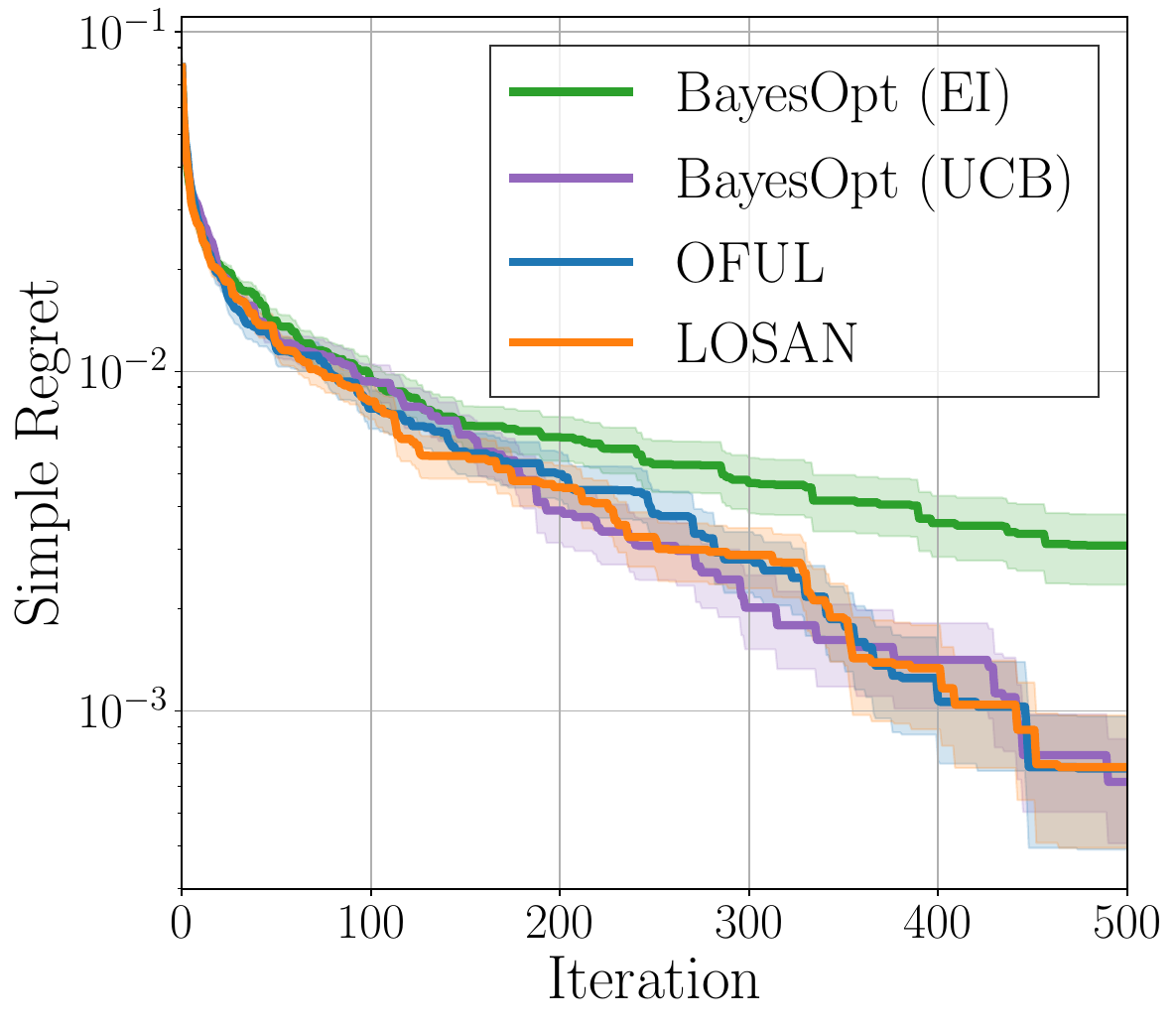}
    }
    \subfigure[ImageNet (Bounded)]{
        \includegraphics[width=0.225\textwidth]{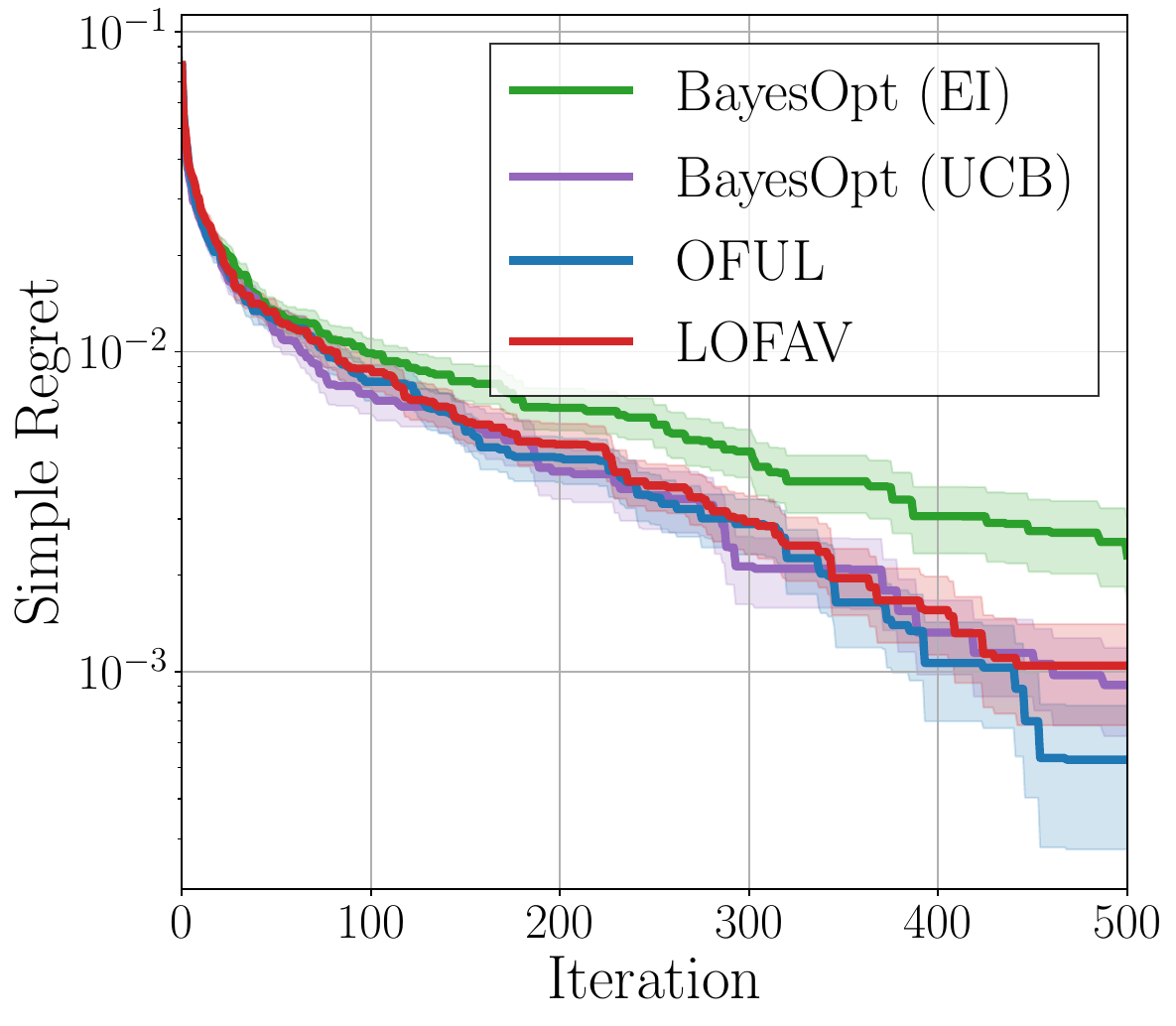}
    }
    \caption{Bayesian optimization results of LOSAN and LOFAV with random Fourier features and sub-Gaussian or bounded noises for NATS-Bench. We perform each experiment over 50 rounds where $S= 1.0$, $d = 128$, $|\cX_t| = 512$, and $\sigma_0 \ \textrm{or} \ R = 1.0$.}
    \label{fig:bo_natsbench}
\end{figure*}

\paragraph{NATS-Bench.}

We utilize NATS-Bench~\citep{DongX2021ieeetpami},
which is a set of benchmarks on neural architecture search~\citep{ZophB2017iclr},
in order to examine our algorithms and baseline methods.
We employ the size search space of NATS-Bench
for CIFAR-100 and ImageNet-16-120.
Since the number of possible architecture candidates, i.e., 32,768, is significantly greater than the number of arms often used in the linear bandits,
we sample 512 arms from the continuous space of the size search space.
\cref{fig:bo_natsbench} demonstrates that there is no single method that dominates the rest.
LOSAN shows comparable results for CIFAR-100 (sub-Gaussian) and ImagenNet (sub-Gaussian).

\section{Related Work}
\label{sec:related}

\paragraph{Heteroscedastic linear bandits.}
For the sub-Gaussian noise, \citet{kirschner18information} first considered the linear bandit problem where the noise at time $t$ is assumed to be $\sig_t^2$-sub-Gaussian where $\sig_t^2$ is known to the learner.
\citet{zhou21nearly} and \citet{zhou22computationally} consider the setup where the noise $\sig_t^2$ is known only after the learner chooses the arm at time $t$.

For the bounded \textit{noise} setup $\eta_t\in[-1,1]$ (often appear as \textit{rewards} being bounded), the seminal work of~\citet{zhang21variance} proposes a linear bandit algorithm called VOFUL whose regret bound was shown to be adaptive to the unknown variances $\{\sig^2_t\}_{t=1}^n$.
They achieved a regret upper bound of order $d^{4.5}\sqrt{1+\sum_{t=1}^n \sig_t^2} + d^5$. 
\citet{kim22improved} then improved the regret bound of VOFUL to $d^{1.5}\sqrt{1+\sum_{t=1}^n \sig_t^2} + d^2$.
However, VOFUL has an exponential time complexity.
\citet{zhao23variance} has made a breakthrough by achieving the optimal variance-adaptive worst-case regret bound of $d\sqrt{1+\sum_{t=1}^n \sig_t^2} + d$ with a computationally efficient algorithm called SAVE.
However, as detailed in \cref{sec:fully}, SAVE makes a limiting assumption on the variance, and the algorithm is not practical. 
Our algorithm LOFAV overcomes these two limitations, which we claim to be the first practical variance-adaptive algorithm.
\citet{xu23noise} further extends SAVE to a Thompson sampling style algorithm, which adds an extra factor of $\sqrt{d}$ in the regret bound -- this is known to be unavoidable for linear Thompson sampling style algorithms \cite{hamidi20worst}.
Unfortunately, they also use the inefficient SupLinRel-style algorithm.
This also means that their time complexity per round w.r.t. the arm set size $|\cX_t|$ scales with $d^2 |\cX_t|$ rather than $d|\cX_t|$ of the standard linear Thompson sampling~\citet{agrawal13thompson}, which is one of the main benefits of linear Thompson sampling.

\paragraph{Improved confidence sets.}
Recently, there have been quite a few studies that improve confidence sets for linear models upon the standard self-normalized confidence set (SNCS)~\cite[Theorem 2]{ay11improved}.
\citet{emmenegger23likelihood} propose a novel confidence set construction based on weighted sequential likelihood ratio tests, which is empirically shown to be tighter than SNCS.
However, the analysis does not show a tighter bound than SNCS due to an extra factor of $S$ (the bound on $\|\th^*\|_2$), which requires further investigation.
\citet{flynn23improved} propose a novel adaptive martingale mixture to construct an improved confidence set that is both numerically and provably tighter than SNCS.
However, the degree of tightness was not precisely quantified as an orderwise improvement.
In stark contrast, our confidence set for LOSAN enjoys an improvement that is precisely quantified in terms of the true noise level $\sig^2_*$ and the specified noise level $\sig^2_0$.
We are not aware of any comparable results in prior work.

\paragraph{Confidence sets via online learning regret bounds.} 
In one way or another, regret bounds of online learning (OL) algorithms play an important role in constructing or analyzing novel confidence bounds or sets.
The seminal work of \citet{rakhlin17on} makes a strong case by showing that the existence of an OL regret bound implies a confidence bound.
For (generalized) linear models, we have found that there have been three types of results that leverage OL regret bounds to construct confidence sets.
The first is to construct a confidence set by running a specific OL algorithm (e.g., online Newton step~\cite{hazan07logarithmic}) and leverage its regret bound to quantify the confidence width, which relies critically on a negative term in the regret bound~\cite{dekel12selective,crammer13multiclass,gentile14onmultilabel,zhang16online}.
The second is the so-called online-to-confidence-set conversion~\cite{ay12online,jun17scalable}, which constructs a confidence set by \textit{regressing} on the prediction made by running an OL algorithm $\cA$ whose confidence width becomes a function of the regret bound of $\cA$.
The advantage of this method is that we are not married to those OL algorithms that have a particular negative term in the regret bound, which provides more flexibility.
The third is the so-called regret-to-confidence-set conversion~\cite{lee24improved}, which constructs a confidence set with the maximum likelihood estimator but characterizes the confidence width with an achievable OL regret bound.
Similar to \citet{rakhlin17on}, this technique only requires the \textit{existence} of a regret bound achieved by an online learner, and thus one can even use the regret bounds of computationally intractable OL algorithms such as those in~\citet{mayo22scale}.
Interestingly, the role of OL regret bounds here is solely an analysis tool, perhaps providing a shortcut to otherwise complicated analysis.
Our confidence set for LOSAN belongs to the first technique above since we leverage the negative term in the regret bound directly.
In this regard, we remark that \citet{emmenegger23likelihood} also use OL regret bounds for the analysis of their confidence set algorithm.

\section{Conclusion}
\label{sec:conclusion}

Our advances in noise-adaptive confidence sets and applications to linear bandits and Bayesian optimization open up numerous exciting future directions.
First, investigating whether similar variance-adaptive worst-case regret bounds are possible in more generic hypothesis classes and various noise models is an open question.
Second, the weighted ridge regression considered in this paper computes the weights in a sequential manner.
It would be interesting to study if there exists a batch counterpart of the weights, which could be more useful for the fixed design case.
Finally, it would be interesting to develop a Thompson sampling version of the variance-adaptive algorithm whose time complexity w.r.t. the arm set size $|\cX_t|$ is $O(d|\cX_t|)$ per iteration.

\section*{Acknowledgements}

This work used, in part, the H2P cluster of the University of Pittsburgh Center for Research Computing, which is supported by National Science Foundation award number OAC-2117681.
Kwang-Sung Jun was supported in part by the National Science Foundation under grant CCF-2327013.

\section*{Impact Statement}

This paper presents work whose goal is to advance the field of machine learning. There are many potential societal consequences of our work, none which we feel must be specifically highlighted here.

\bibliography{library-overleaf,bo}
\bibliographystyle{icml2024}

\newpage
\appendix
\onecolumn

\addcontentsline{toc}{section}{Appendix} 
\part{Appendix} 
\parttoc 

\section{Details for Semi-Adaptation}
\label{sec:details-semi} 

For convenience, we introduce the following model assumption, which helps us use simpler notations for both Section~\ref{sec:details-semi} and~\ref{sec:details-fully}.
\begin{assumption}\label{ass:model}
  Let $\{(z_t \in \RR^d, u_t \in \RR)\}_{t=1}^\infty$ be a sequence of feature vectors and observed labels.
  Let $\cF_t = \sig(z_1,u_1,\ldots,z_t,u_t,z_{t+1})$.
  Assume that $u_t = z_t^\T \th^* + \nu_t$ where $\nu_t \mid \cF_{t-1}$ is $\blue{\sig_*^2}$-sub-Gaussian (i.e., $\forall \lam\in\RR, \EE[\exp(\lam\nu_t) \mid \cF_{t-1}] \le e^{\lam^2\sig^2/2}$). 
  Let $\blue{\sig^2_t} = \EE[\nu_t^2 \mid \cF_{t-1}]$.
  Let $\blue{L_t(\th)} = \sum_{s=1}^t f_s(\th) + \fr{\lam}{2}\normz{\th}^2$ where $\blue{f_s(\th)} = \fr12(z_s^\T \th - u_s)^2$ and define
  \begin{align*}
    \blue{\hth_t} = \arg \min_{\th} L_t(\th) \text{~~ and ~~}  \blue{\Sig_t} = \lam I + \sum_{s=1}^t z_s z_s^\T ~.
  \end{align*}
\end{assumption}
Specifically, the model assumption in \cref{sec:semi} can be reduced to \cref{ass:model} by setting $z_t = w_{t} x_t$, $u_t = w_{t} y_t$, $\nu_t = w_t \eta_t$, $f_s(\th) = \ellw_s(\th)$ with the same sub-Gaussian parameter $\sig^2_*$.
With this, other notations like $\hth_{t-1}$, $\Sig_t = \lam I + \sum_{s=1}^t w_s^2 x_s x_s^\T = \lam I + \sum_{s=1}^t z_s z_s^\T$, and $D_s^2 = \normz{w_t x_t}^2_{\Sig_t^{-1}} = \normz{z_t}^2_{\Sig_t^{-1}}$ remain the same.
Throughout this section, we use this set of notations to avoid clutter.

\subsection{Proof of the Confidence Set (\cref{thm:confset-semi})}

Let $r_s := z_s^\T(\hth_{s-1} - \th^*)$.
The regret equality of FTRL (\cref{lem:regret_equality}) states that
\begin{align}
  \fr12 \normz{\hth_t - \th^*}^2_{\Sig_t} 
  &= \fr \lam 2 \normz{\th^*}^2 + \sum_{s=1}^{t}  f_s(\hth_{s-1}) D_s^2 + \sum_{s=1}^t f_s(\th^*) - f_s(\hth_{s-1}) ~. \notag
\end{align}
The negative regret bound (\cref{lem:negregret}) implies that, with probability at least $1-\dt$,
\begin{align*}
  \forall t\ge1, \sum_{s=1}^t (f_s(\th^*) - f_s(\hth_{s-1})) 
  &\le \sig_*^2 \ln(1/\dt)~.
\end{align*}
Thus, we have
\begin{align*}
  \forall t\ge1, 
  \fr12 \normz{\hth_t - \th^*}^2_{\Sig_t}
  &\le \fr \lam 2 \normz{\th^*}^2 + \sum_{s=1}^{t} f_s(\hth_{s-1}) D_s^2 + \sig_*^2\ln(1/\dt) 
  \\&\le \fr \lam 2 S^2 + \sum_{s=1}^{t} f_s(\hth_{s-1}) D_s^2 + \sig_0^2\ln(1/\dt)~.  \tag{\cref{ass:semi}}
\end{align*}
Therefore, it is easy to see that with probability at least $1-\dt$, 
\begin{align*}
  \forall t\ge1, \th^* \in \cC^{\semi}_t ~.
\end{align*}
This concludes the proof of \cref{thm:confset-semi}.

\subsection{Analysis of the (Normalized) Confidence Width \texorpdfstring{$\sqrt\gam_t$}{}}

\begin{proposition}\label{prop:gam_t-bound}
  With probability at least $1-3\dt$, 
  \begin{align*}
    \gam_t \lsim \lam S^2 + \sig_*^2 d\ln\del{1 + \fr{t}{d\lam}} + \sig_0^2 \ln(1/\dt)~.
  \end{align*}
\end{proposition}

\begin{proof}
  By \cref{lem:empirical-var}, we have, with probability at least $1-2\dt$, 
\begin{align*}
  \sum_{s=1}^{t} f_s(\hth_{s-1})D_s^2 
  &\le \fr32 \lam \normz{\th^*}_2^2 + 3\sum_{s=1}^t f_s(\th^*)D_s^2 +  10\sig_*^2 \ln(1/\dt)~.
\end{align*}
Thus,
\begin{align*}
  \gam_t
  &\lsim \lam S^2 + 3\sum_{s=1}^t f_s(\th^*)D_s^2 + \sig_0^2 \ln(1/\dt)~.  \tag{$\sig_*^2\le\sig_0^2$ }
\end{align*}

It remains to study how $\sum_{s=1}^t f_s(\th^*) D_s^2$ scales.

Since $\nu_s$ is $\sig_*^2$-sub-Gaussian, $\nu_s^2$ is ($v=2\sig_*^4$,$c=2\sig_*^2$)-subgamma~\citet[Section 2.4]{boucheron13concentration}.
Then, we can use the concentration of subgamma random variables~\citet[Theorem 2.3]{boucheron13concentration} to obtain
\begin{align*}
  \sum_{s=1}^t f_s(\th^*) D_s^2
  &= \sum_{s=1}^t \nu_s^2 D_s^2
  \\&\le \sum_{s=1}^t \sig_*^2 D_s^2 + \sqrt{4 \sig_*^4\sum_{s=1}^t D_s^4 \ln(1/\dt)} + 2\sig_*^2 \ln(1/\dt)
  \\&\sr{(a)}{\le} \sig_*^2 \cd d\ln(1 + \fr{t }{d\lam}) + \sig_*^2\sqrt{4 d\ln(1 + \fr{t }{d\lam}) \ln(1/\dt)} + 2\sig_*^2 \ln(1/\dt)
  \\&\le 2\sig_*^2 \cd d\ln(1 + \fr{t }{d\lam}) + 3\sig_*^2 \ln(1/\dt)
\end{align*}
where $(a)$ is by $D_s^4\le D_s^2$ (since $D_s^2\le1$) and the standard elliptical potential lemma (see \cref{lem:epl}).
This concludes the proof.
\end{proof}

\subsection{Regret Analysis (Proof of \cref{thm:semi-regret-bound})}
Define the event
\begin{align*}
  \blue{\cE_1} := \{\forall t\ge1, \th^* \in \cC_t\}~.
\end{align*}
Let $\blue{\reg_t} = \inp{x_t^*}{\th^*} - \inp{x_t}{\th^*}$ where $x_t^* = \arg \max_{x\in\cX_t} \inp{x}{\th^*}$.
The following lemma is standard in linear bandit analysis.
\begin{lemma}\label{lem:24-0112-basic}
  Let $\blue{\tth_t} = \arg \max_{\th\in \cC_t} \la x_t, \th\ra$.
  Under the event $\cE_1$,
  \begin{align*}
  \reg_t &\le \normz{x_t}_{\Sig^{-1}_{t-1}} \sqrt{8\gam_{n}} ~.
\end{align*}
\end{lemma}
\begin{proof}
  We have
  \begin{align*}
  \reg_t 
    &\le \la x_t, \tth_t - \th^*\ra  \tag{def'n of $x_t$}
  \\&\le \normz{x_t}_{\Sig^{-1}_{t-1}} \normz{\tth_t - \th^*}_{\Sig_{t-1}}  \tag{Cauchy-Schwarz}
  \\&\le \normz{x_t}_{\Sig^{-1}_{t-1}} (\normz{\tth_t - \hth_{t-1}}_{\Sig_{t-1}}   + \normz{\hth_{t-1} - \th^*}_{\Sig_{t-1}} ) \tag{triangle inequality}
  \\&\le \normz{x_t}_{\Sig^{-1}_{t-1}} \sqrt{8\gam_{t-1}} \tag{by $\cE_1$}
  \\&\le \normz{x_t}_{\Sig^{-1}_{t-1}} \sqrt{8\gam_{n}}~. \tag{monotonicity}
  \end{align*}
\end{proof}
For the regret analysis, using the fact that $\reg_t \le 2B$,
\begin{align*}
  \sum_t^n \reg_t
  \le \sum_{t=1}^n \onec{w_t \neq 1} 2B + \sum_{t=1}^n \onec{w_t = 1} \reg_t~.
\end{align*}
For the first term,
\begin{align*}
  \sum_{t=1}^n \onec{w_t \neq 1} 2B
  &= 2B\sum_{t=1}^n \onec{w_t \neq 1, \normz{w_t x_t}^2_{\Sig^{-1}_{t-1}} = 1 }  \tag{def'n of $w_t$}
  \\&\le 2B\sum_{t=1}^n \onec{\normz{w_t x_t}^2_{\Sig^{-1}_{t-1}} \ge 1 } 
  \\&\le 6B\cd d \ln\del{1 + \fr{2}{\lam}} ~. \tag{by EPC (\cref{lem:epc})}
\end{align*}
For the second term,
\begin{align*}
  \sum_{t=1}^n \onec{w_t = 1} \reg_t
  &\le \sum_{t=1}^n \onec{w_t = 1} \normz{x_t}_{\Sig_{t-1}^{-1}} \sqrt{8\gam_n}
\\&=   \sum_{t=1}^n \onec{w_t = 1} \normz{w_t x_t}_{\Sig_{t-1}^{-1}} \sqrt{8\gam_n}
\\&\le \sqrt{8\gam_n}\sqrt{n \sum_{t=1}^n \normz{w_t x_t}^2_{\Sig_{t-1}^{-1}}}
\\&\le \sqrt{8\gam_n}\sqrt{n 2d\ln(1 + \fr{n}{d\lam})}
\end{align*}
where the last line is due to the elliptical potential lemma~\cite[Lemma 11]{ay11improved} and the fact that $\normz{w_t x_t}^2_{\Sig_{t-1}^{-1}} \le 1$ by the definition of $w_t$.
Applying \cref{prop:gam_t-bound} and noting that the event $\cE_1$ happens with probability at least $1-\dt$ conclude the proof.

\section{Details for Full Adaptation}
\label{sec:details-fully}

\subsection{Proof of \cref{thm:confset-fa}}

The confidence set is an intersection of $\cC_{t,\ell}$'s.
Therefore, it suffices to prove that the confidence set $\cC_{t,\ell}$ contains $\th^*$ with probability at least $1-\fr{\delta}{L}$.
This is a direct consequence of \cref{thm:fully-conf-set-each} below where we set $z_t = w_{t,\ell} x_t$, $u_t = w_{t,\ell} y_t$, and $\lam = \lam_\ell$, and replace $\dt$ with $\dt / (2L)$.
Then, $f_s(\th) = \ellw_{s,\ell}(\th)$ and other symbols becomes just a matter of adding the extra subscript $\ell$ (e.g., $D_s^2 = D_{s,\ell}^2$, $\xi_t = \xi_{t,\ell}$, $k_t = k_{t,\ell}$, $\beta_t = \beta_{t,\ell}$, $\hth_t = \hth_{t,\ell}$, etc.)

\begin{theorem}
  \label{thm:fully-conf-set-each}
  Take \cref{ass:model} with the added assumption of $\nu_t \in[-R,R]$.
  Define
  \begin{align*}
    \blue{K_t(\th)} &= \sum_{s=1}^t f_s(\th) + \sum_{s=1}^t \fr12 (z_s^\T(\th - \hth_{s-1}))^2 + \fr\lam2\normz{\th}^2
    \\ \text{and ~~~} \blue{\barth_{t}} &= \arg \min_{\th} K_t(\th) ~.
  \end{align*}
  Furthermore, define $\blue{\barSig_t} = 2\sum_{s=1}^t z_s z_s^\T + \lam I$.

  Let $\blue{D^2_s} := \normz{z_s}_{\Sig_{s}^{-1}}^2$ and define $\beta_t$ recursively as follows:
  \begin{align*}
    \blue{\beta_{t}} 
    &= L_{t}(\hth_{t}) - K_{t}(\barth_{t}) +\fr{\lam}{2}S^2 + \sum_{s=1}^{t} f_s(\hth_{s-1}) D_s^2
      + \rho\sqrt{\barbeta_{t-1}}\sqrt{ 8 \del[2]{\sum_{s=1}^t f_s(\hth_{s-1}) +  R^2 \ln(1/\dt)} \xi_t} + 2^{k_t} \rho R\sqrt{2\beta_0}\xi_t  
  \end{align*}
  where $\blue{\bar\beta_{t-1}} = \max_{s=1}^{t-1} \beta_{s}$, $\blue{\beta_0} = \fr{\lam}{2}S^2$, $\blue{\xi_{t}} := \ln( \sqrt{\pi(t+1)} \cd \fr{ 3.4\cd k_{t} \ln^2(1+k_{t}) }{\dt})$ with $\blue{k_t} = 1 \vee \lcl \log_2(\sqrt{\barbeta_{t-1}/\beta_0}) \rcl$.
  Define the confidence set 
  \begin{align*}
    \blue{\cC_t} = \cbr{\th\in\RR^d: \fr12 \normz{\th - \barth_t}_{\barSig_{t}}^2 \le \beta_{t}}~.
  \end{align*}
  Assume $\forall s\ge1, \blue{D^2_s} \le \blue{\rho^2}$ for some $\rho^2 > 0$.
  Then,
  \begin{align*}
    \PP(\forall t\ge1, \th^* \in \cC_t) \ge 1- 2\dt
  \end{align*}
\end{theorem}
\begin{proof}
  This proof is inspired by \citet{zhao23revisiting}, but details differ since we leverage the regret equality (\cref{lem:regret_equality}), which helps shorten the proof and provides a numerically tight derivation.
  
To describe the plan, we will show $\blue{\barM_t^2} := \fr12 \normz{\th^* - \barth_t}_{\barSig_{t}}^2 \le \beta_t, \forall t \ge 0$ under an event $\cE_1$ (defined below) that holds with probability at least $1-\dt$.
For this, we use induction.
Define $\blue{I_{s}} = \onec{\barM_{s}^2 \le \beta_{s}}$ and let $\blue{\barI_s} = 1-I_s$.

First, we show the base case of $I_0 = 1$.
This is trivial since $\barM_0^2 = \fr12 \|\th^*\|^2_{\lam I} \le \fr\lam2 S^2 = \beta_0$.
It remains to prove that, assuming  $I_0 = \cdots=I_{t-1}=1$, we have $I_t = 1$.

So, let us assume the inductive hypothesis $I_0 = \cdots=I_{t-1}=1$.
Define $\blue{r_t} = z_t^\T(\hth_{t-1} - \th^*)$, which implies that $f_s(\hth_{s-1}) - f_s(\th^*) = \fr12 r_s^2 - r_s\nu_s $.
Recall the regret equality (\cref{lem:regret_equality}):
\begin{align*}
  M_t^2 := \fr12\normz{\th^* - \hth_t}^2_{\Sig_t}  
  &= \fr{\lam}{2}\|\th^*\|^2 + \sum_{s=1}^t f_s(\hth_{s-1}) D_s^2 - \sum_{s=1}^t \del{f_s(\hth_{s-1}) - f_s(\th^*)} 
  \\&= \fr{\lam}{2}\|\th^*\|^2 + \sum_{s=1}^t f_s(\hth_{s-1}) D_s^2 + \sum_{s=1}^t(-\fr12 r_s^2) + \sum_{s=1}^t r_s \nu_s
\\ \implies
  M_t^2 + \sum_{s=1}^t \fr12 r_s^2 &\le \fr{\lam}{2}\|\th^*\|^2 + \sum_{s=1}^t f_s(\hth_{s-1}) D_s^2 + \sum_{s=1}^t r_s \nu_s~.
\end{align*}
Note that 
\begin{align*}
  \sum_{s=1}^{t}  r_s \nu_s
  = \sum_{s=1}^{t} I_{s-1} r_s \nu_s + \sum_{s=1}^t \bar I_{s-1}  r_s \nu_s
  = \sum_{s=1}^{t} I_{s-1} r_s \nu_s \tag{inductive hypothesis}
\end{align*}  

We assume the event in Corollary~\ref{cor:23-1118-cb-v3} with
\begin{itemize}
  \item $\blue{b_k} = 2^{k} \rho R\sqrt{2\beta_0}.$ 
  \item $\blue{X_s} = I_{s-1} r_s \nu_s$~,
\end{itemize}
which holds with probability at least $1-\dt$.
Recall that $k_t = 1 \vee \lcl \log_2(\sqrt{\barbeta_{t-1}/\beta_0}) \rcl$, which means that $k_t = \min\{k\in \NN_+: b_k \ge \rho R\sqrt{2\barbeta_{t-1}}\}$.
Then, since 
\begin{align*}
  |X_s| 
   \le R I_{s-1}|r_s|
  &\le R\normz{z_s}_{\barSig_{s-1}^{-1}} \normz{\barth_{s-1} - \th^*}_{\barSig_{s-1}} \tag{Cauchy-Schwarz}
\\&\le R\normz{z_s}_{\barSig_{s-1}^{-1}} \sqrt{2\beta_{s-1}} \tag{$I_{s-1} = 1$ }
\\&\le R\rho \sqrt{2\beta_{s-1}} \tag{assumption in the theorem}
\\&\le R\rho \sqrt{2\barbeta_{t-1}}
  ~\le b_{k_t}~.
\end{align*}
Thus, we have $\clip{X_s}_{b_k} = X_s, \forall s\in [t]$, so we have
\begin{align}
  \forall t \ge 1, \sum_{s=1}^{t} I_{s-1} r_s \nu_s 
  &\le \sqrt{2\sum_{s=1}^t I_{s-1} r_s^2 \nu_s^2 \xi_t} + 2^{k_t} \rho R\sqrt{2\beta_0} \xi_t~.
\end{align}

Note that, with a similar derivation as the bound on $|X_s|$, we have
\begin{align*}
  I_{s-1} r_s^2 
  &\le 2\rho^2 \barbeta_{t-1}~.
\end{align*}
Then, using \cref{lem:negregret}, with probability at least $1-\dt$, we have
\begin{align*}
\forall t\ge1,   \sum_{s=1}^t \nu_s^2 
  = \sum_{s=1}^t 2 f_s(\th^*) 
  = 2\sum_{s=1}^t f_s(\hth_{s-1}) + 2\sum_{s=1}^{t} f_s(\th^*) - f_s(\hth_{s-1}) 
  \le 2\sum_{s=1}^t f_s(\hth_{s-1}) +  2 R^2 \ln(1/\dt)  ~. 
\end{align*}
Thus,
\begin{align*}
  M_t^2 + \sum_{s=1}^t \fr12 r_s^2 
  \le \fr{\lam}{2}\normz{\th^*}^2 + \sum_{s=1}^{t} f_s(\hth_{s-1}) D_s^2
  + \sqrt{ 8\rho^2  \barbeta_{t-1}\del[3]{\sum_{s=1}^t f_s(\hth_{s-1}) +  R^2 \ln(1/\dt)}  \xi_t} 
  + 2^{k_t} \rho R\sqrt{2\beta_0} \xi_t =: \blue{\beta_t}~.
\end{align*}
Recall that $L_t(\th) = \sum_{s=1}^{t} f_s(\th) + \fr\lam2\normz{\th}^2$.
Since $M_t^2 = L_t(\th^*) - L_t(\hth_t)$ (verify this with Taylor's theorem), the LHS above can be rewritten as
\def\barSig{{\wbar{\Sig}}}
\begin{align*}
  &\sum_{s=1}^t f_s(\th^*) + \fr\lam2\normz{\th^*}^2 - L_t(\hth_t) + \sum_{s=1}^t \fr12(z_s^\T(\hth_{s-1} - \th^*) )^2
  \\&= K_t(\th^*) - L_t(\hth_t)
  \\&= K_t(\th^*) - K_t(\barth_{t}) + K_t(\barth_{t}) - L_t(\hth_t)
  \\&= \fr12 \normz{\th^* - \barth_t}_{\barSig_t}^2 + K_t(\barth_t) - L_t(\hth_t)
\end{align*}
Rearranging the terms and using the bound $\normz{\th^*}^2 \le S^2$ prove $I_t = 1$, which completes the inductive proof. 
\end{proof}
Next, we find a nondecreasing upper bound on $\beta_t$ that directly depends on the true variances $\{\sig_s^2\}_{s=1}^{t}$.
\begin{lemma} \label{lem:beta_t-true-variance} 
  Take the assumptions of \cref{thm:fully-conf-set-each}.
  Suppose $\lam = \fr{R^2}{S^2}\rho^2 \lam_0$ for some absolute constant $\lam_0 > 0$.
  Then,  with probability at least $1-O(\dt)$, 
\begin{align*}
  \forall t\ge 1, \barbeta_t \le \beta^*_t := c\rho^2(R^2 + \sum_{s=1}^{t} \sig_s^2)\ln^2(t/\dt) 
\end{align*}
for some absolute constant $c>0$.
\end{lemma}
\begin{proof}
Using the definition of $k_t$, we have
\begin{align*}
\beta_t &\le
    L_t(\hth_{t}) - K_t(\barth_t) 
\\& +\fr{\lam}{2}\normz{\th^*}^2 + \sum_{s=1}^{t} f_s(\hth_{s-1}) D_s^2
   + \sqrt{ 8\rho^2  \barbeta_{t-1}\del[3]{\sum_{s=1}^t f_s(\hth_{s-1}) +  R^2 \ln(1/\dt)}  \xi_t} 
+ \rho R\sqrt{8\barbeta_{t-1}}\xi_t~.
\end{align*}
First, we notice that we can easily bound the first two terms in $\beta_t$:
\begin{align*}
  L_t(\hth_t) - K_t(\barth_t) 
  \le L_t(\barth_t) - K_t(\barth_t) 
  \le 0 ~.
\end{align*}
Furthermore,
\begin{align*}
  2^{k_t} \rho\sqrt{2\beta_0}\xi_t \le \rho \sqrt{8\barbeta_{t-1}}\xi_t~.
\end{align*}
Let $\blue{A_t} = \sum_{s=1}^t f_s(\hth_{s-1})$.
We use $D_s^2 \le \rho^2$ and the Fenchel-Young inequality $xy \le \fr{1}{2a}x^2 + \fr{a}{2} y^2, \forall a>0$ to obtain
\begin{align*}
  \beta_t 
  &\le \fr{\lam}{2} S^2 + \rho^2 A_t + 8\rho^2  (A_t + R^2\ln(1/\dt)) \xi_t + \fr14 \barbeta_{t-1} + 8\rho^2 R^2 \xi^2_{t} + \fr{1}{4}\barbeta_{t-1}
\\&\le c\lam S^2 + c \rho^2 \xi_t(A_t + R^2 \xi_t) + \fr12 \barbeta_{t-1} 
\end{align*}
for some absolute constant $c>0$.

Let $\xi = \ln(1/\dt)$.
We bound $A_t$ with probability at least $1-O(\dt)$ as follows:
\begin{align*}
  A_t 
&= \sum_{s=1}^t f_s(\th^*)  + \sum_{s=1}^t (f_s(\hth_{s-1}) - f_s(\th^*))
\\&\lsim \sum_{s=1}^{t} \sig_s^2 + \sum_{s=1}^t r_s^2 + R^2\xi \tag{\cref{lem:sum_eta_sq} and \cref{lem:regret-upperbounded-by-rt-sq}}
\\&\lsim \sum_{s=1}^{t} \sig_s^2 + \lam S^2 + \rho^2 \sum_{s}\nu_s^2 + R^2\xi \tag{\cref{lem:sum_rs_sq}; $D_s^2\le\rho^2$}
\\&\lsim \sum_{s=1}^{t} \sig_s^2 + \lam S^2 + R^2\xi ~. \tag{\cref{lem:sum_eta_sq}; $\rho^2 \le 1$}
\end{align*}
Thus, there exists an absolute constant $c_1,c_2>0$ such that, letting $\blue{Q_t} = \lam S^2 + R^2 + \sum_{s=1}^t \sig_s^2$,
\begin{align*}
  \beta_t 
  &\le c_1 \lam S^2 + c_1 \rho^2\xi_t^2 Q_t + \fr12 \barbeta_{t-1} 
\\&\le \underbrace{c_2 \lam S^2 + c_2 \rho^2 Q_t\del{\ln^2(t/\dt) + \barlnln^2(\sqrt{\barbeta_{t-1}/\beta_0})}}_{\tsty =: \beta'_t} + \fr12 \barbeta_{t-1} 
\end{align*}
for some absolute constant $c_2>e$. 
We define $\beta'_0 = \fr{\lam}{2}  S^2$.
Note that $\beta'_t$ can be compactly written as
\begin{align*}
  \beta'_t = a_t + b_t \barlnln^2(\sqrt{\barbeta_{t-1}/\beta_0})
\end{align*}
where
\begin{align*}
  \blue{a_t} &:= c_2(\lam S^2 + \rho^2 Q_t\ln^2(t/\dt) )
\\\blue{b_t} &:= c_2 \rho^2 Q_t~.
\end{align*}
Let 
\begin{align*}
  \blue{\beta^\dagger_t} = a_t + 2b_t(e + \barlnln^2(2a_t/\beta_0) + \barlnln^2(2b_t/\beta_0))~.
\end{align*}
Note that $\beta'_t \le \beta^\dagger_t$ and that $\beta^\dagger_t$ is a nondecreasing function of $t$.

We claim that $\beta_t \le 2\beta^\dagger_t, \forall t\ge0$.
We use induction.
First, we trivially have $\beta_0 = \fr{\lam}{2}S^2 \le \beta^\dagger_0$ using $c_2\ge e$.

For the general case, assume $\beta_{t-1} \le 2 \beta_{t-1}^*$.
If $\barbeta_{t-1} \le e^2 \beta_0$, then the $\barlnln^2(\cd)$ term is 0, so we trivially have that $\beta'_t \le \beta^\dagger_t$ and thus
\begin{align*}
  \beta_t \le \beta^\dagger_t + \fr12 \barbeta_{t-1} \le \beta^\dagger_t + \beta^\dagger_{t-1} \le 2\beta^\dagger_t~.
\end{align*}
If $\barbeta_{t-1} > e^2 \beta_0$, then we use \cref{lem:lnlnsq} with
$x = \barbeta_{t-1}/\beta_0$, $a =  a_t/\beta_0$, and $b = b_t/\beta_0$ 
to derive
\begin{align*}
  a + b \ln^2\ln(\sqrt{\barbeta_{t-1}/\beta_0})
  &\le a + b \ln^2\ln(\barbeta_{t-1}/\beta_0) \tag{$\barbeta_{t-1}/\beta_0\ge e^2 \ge 1$}
\\&\le \barbeta_{t-1}/\beta_0 ~. \tag{\cref{lem:lnlnsq}; $\barbeta_{t-1}/\beta_0\ge e^2 \ge e$ }
\end{align*}
Thus,
\begin{align*}
   \beta'_t 
  &= a_t + b_t \barlnln^2(\sqrt{\barbeta_{t-1}/\beta_0}) 
\\&\le \barbeta_{t-1} 
  \le \beta^\dagger_{t-1} \le \beta^\dagger_t ~,
\end{align*}
which implies $\beta_t \le 2 \beta^\dagger_t$.
Observing $\forall s\in[t],  \beta_s \le 2 \beta^\dagger_s \implies \barbeta_t \le 2 \beta^\dagger_t$ concludes the claim.

Finally, we need to show that $\beta^\dagger_t \le \beta^*_t$.
First, note that $b_t \lsim a_t$, so $\barlnln^2(2b_t/\beta_0) \le \barlnln^2(2a_t/\beta_0)$.
Furthermore, using the fact that $\sig_s^2 \le R^2$, one can show that $a_t / \beta_0 \lsim 1 + \fr{t}{\lam_0} \lsim t \ln^2(t/\dt)$ since $\lam_0$ is an absolute constant.
Thus, $b_t(e + \barlnln^2(2 a_t / \beta_0)) \lsim \rho^2 Q_t \barlnln^2(t\ln^2(t/\dt)) \lsim \lam S^2 + \rho^2 Q_t \ln^2(t/\dt) \lsim a_t$.
This implies that $\beta^\dagger_t \lsim a_t$, and one can easily show that $a_t \lsim \beta^*_t$.
\end{proof}

\subsection{Proof of \cref{thm:fully-regret-bound}}

By \cref{lem:beta_t-true-variance}, we have that $\beta_{t,\ell} \le \beta^*_{t,\ell} = c \rho_\ell^2(R^2 + \sum_{s=1}^t \sig_s^2)\psi^2_t$ where $\blue{\psi^2_t} = \ln^2(t/\dt)$.
Defining $\blue{\beta^*_{t,0}} = c \psi^2_t(R^2 + \sum_{s=1}^t \sig_s^2)$, we can write down 
\begin{align*}
  \beta^*_{t,\ell}  = 2^{-2\ell}\beta^*_{t,0}~.
\end{align*}
Define $\blue{t'}=t-1$ and $\blue{\tth_{t,\ell}} = \arg \max_{\th\in  C_{t',\ell}} \inp{x_t}{\th}$.
Define $\blue{\reg_t} := \inp{x_{t,*} - x_t}{\th^*}$, which is the instantaneous regret at time $t$.
We first show the elementary bound that is useful throughout: 
\begin{align*}
  \blue{\reg_t} 
  &:= \inp{x_{t,*} - x_t}{\th^*}    \tag{$\blue{x_{t,*}}:= \max_{x\in\cX_t} \inp{x}{\th^*} $ }
  \\&\le \min_{\ell \in [\ell^*+1..L]} \inp{x_t}{\tth_{t,\ell}  - \th^*}   
  \\&=   \min_{\ell \in [\ell^*+1..L]}\inp{x_t}{\tth_{t,\ell} - \barth_{t',\ell} + \barth_{t',\ell} - \th^*}  
  \\&\le \min_{\ell \in [\ell^*+1..L]} \sqrt{8}\normz{x_t}_{\barSig^{-1}_{t',\ell}} \sqrt{\beta^*_{t',\ell } }  
  \\&\le \min_{\ell \in [\ell^*+1..L]} \sqrt{8}\normz{x_t}_{\Sig^{-1}_{t',\ell}} 2^{-\ell}  \sqrt{\beta^*_{n,0}} ~. \tag{$\barSig_{t',\ell} \succeq \Sig_{t',\ell}$; monotonicity of $\beta^*_{t,0}$ }
\end{align*}
Also note the trivial upper bound: $\reg_t \le 2B$.

We classify the time steps $[n]$ into the following three:
\begin{itemize}
  \item $\cT_1 := \{t\in[n]: \reg_t > \sqrt{8\beta^*_{n,0}}\}$
  \item $\cT_2 := \{t\in[n]: \reg_t \le 2^{-2L}  \sqrt{8\beta^*_{n,0}}\}$
  \item $\cT_3 := \cbr[2]{t\in[n]: \exists \ell\in[1..L] \suchthat \reg_t \in \left\lparen   2^{-2\ell} \sqrt{8\beta^*_{n,0}},~ \cd 2^{-2(\ell-1)} \sqrt{8\beta^*_{n,0}}\right\rbrack }$ ~.
\end{itemize}
Hereafter, we define $\blue{\Sig_{t,\ell}[\cT]} := \lam_\ell I + \sum_{s\in[t] \cap \cT} w^2_{s,\ell}x_s x_s^\T$.

\textbf{Case 1. } $t \in \cT_1$   \\
We have
\begin{align*}
  \sqrt{8\beta^*_{n,0}} 
  < \reg_t 
  \le \normz{x_t}_{\Sig^{-1}_{t',1}} 2^{-1} \sqrt{8\beta^*_{n,0}} 
\implies
  \normz{x_t}_{\Sig^{-1}_{t',1}} &\ge 2~.
\end{align*}
This means that $\fr{2^{-1}}{\normz{x_t}_{\Sig^{-1}_{t',1}}} \le \fr14$, so $w_{t,1} < 1$.
Then, we have $\normz{w_{t,1} x_t}_{\Sig^{-1}_{t',1}} = \fr12$ .
Thus, using $\reg_t \le 2B$,
\begin{align*}
  \sum_{t\in\cT_1} \reg_t
  &\le \sum_{t\in\cT_1} \onec{\normz{w_{t,1} x_t}_{\Sig^{-1}_{t',1}} = \fr12} 2B
  \\&\le \sum_{t\in\cT_1} \onec{\normz{w_{t,1} x_t}_{\Sig^{-1}_{t',1}[\cT_1]} \ge \fr12} 2B
  \\&\lsim Bd \ln(1 + \fr{S^2}{R^2})  ~. \tag{EPC (\cref{lem:epc}) }
\end{align*}

\textbf{Case 2. } $t \in \cT_2$ \\
Since $\reg_t \le 2^{-2L}\sqrt{8\beta^*_{n,0}}$,
\begin{align*}
  \sum_{t\in\cT_2} \reg_t
  &\le \sum_{t\in\cT_2}   2^{-2L} \sqrt{8\beta^*_{n,0}}
  ~\le n \cd 2^{-2L} \sqrt{8\beta^*_{n,0}}~.
\end{align*}

\textbf{Case 3. } $t \in \cT_3$ \\
Define $\blue{\cT_{3,\ell}} = \cbr{t\in\cT_3: \ell_t = \ell}$.
Note that
\begin{align}
  2^{-2\ell}\sqrt{8\beta^*_{n,0}} < \reg_t 
  \le \normz{x_t}_{\Sig^{-1}_{t',\ell} } 2^{-\ell} \sqrt{8\beta^*_{n,0}}
  ~~~~~&\implies~~~~~
  2^{-\ell} \le \normz{x_t}^2_{\Sig_{t',\ell}^{-1}} \notag
  \\ &\implies~~~~~ \label{eq:T_3-key-property-2-v3} 
  w_{t,\ell} = \fr{2^{-\ell} }{\normz{x_t}_{\Sig_{t',\ell}^{-1}}} ~.
\end{align}
That is, $w_{t,\ell}$ can be written without `$\wed 1$' from its definition.
Using this,
\begin{align*}
  \sum_{t\in\cT_3} \reg_t
  &= \sum_{\ell=1}^{L} \sum_{t\in\cT_{3,\ell}}  \reg_t 
  \\ &= \sum_{\ell=1}^{L} \sum_{t\in\cT_{3,\ell}} \onec{\normz{w_{t,\ell} x_t}_{\Sig^{-1}_{t',\ell}} = 2^{-\ell}  } \reg_t \tag{by \eqref{eq:T_3-key-property-2-v3}}
  \\ &\le \sum_{\ell=1}^{L} \sum_{t\in\cT_{3,\ell}} \onec{\normz{w_{t,\ell} x_t}_{\Sig^{-1}_{t',\ell}} = 2^{-\ell}  } 2^{-2(\ell-1)} \sqrt{8\beta^*_{n,0}} \tag{$t\in\cT_{3,\ell}$} 
  \\ &\le \sum_{\ell=1}^{L} \sum_{t\in\cT_{3,\ell}} \onec{\normz{w_{t,\ell} x_t}_{\Sig^{-1}_{t',\ell}[\cT_{3,\ell}]} \ge 2^{-\ell}  } 2^{-2(\ell-1)} \sqrt{8\beta^*_{n,0}} 
  \\ &= \sum_{\ell=1}^{L} 2^{2\ell} d\ln\del[2]{1+\fr{2^{2\ell}S^2}{R^2}} 2^{-2(\ell-1)} \sqrt{8\beta^*_{n,0}} \tag{EPC (\cref{lem:epc})} 
  \\ &\lsim L \cd d\ln\del[2]{1 + \fr{4^L S^2}{R^2} } \sqrt{\beta^*_{n,0}}  ~.
\end{align*}

Altogether, we have
\begin{align*}
  \sum_{t=1}^n \reg_t 
  \lsim B d \ln\del{1 + \fr{S^2}{R^2}} + \del[2]{n 2^{-2L} + L \cd d\ln\del[2]{1 + \fr{4^L S^2}{R^2} } } \sqrt{\beta^*_{n,0}}~.
\end{align*}
Using the definition $\blue{L} = 1 \vee \lcl \fr12 \log_2 (n/d) \rcl$,
\begin{align*}
  \Reg_n \lsim \ln(n/d)\cd d\sqrt{\del[2]{R^2 + \sum_{t=1}^n \sig_t^2} \ln^2(n/\dt)} + Bd \ln(1 + \fr{S^2}{R^2})~.
\end{align*}

\section{Utility Lemmas}

The following lemma is the well known elliptical potential lemma from linear algebra.
\begin{lemma}(Elliptical potential lemma; e.g., \citet[Lemma 11]{hazan07logarithmic})
  \label{lem:epl}
  Let $x_1,\ldots, x_t\in\RR^d$  be a sequence of vectors with $\normz{x_s}_2\le X, \forall s\in[t]$ for some $X>0$. 
  Let $V_t = \lam I + \sum_{s=1}^t \sum_{s=1}^t x_s x_s^\T$ for some $\tau > 0$.
  Let $|A|$ be the determinant of the matrix $A$.
  Then,
  \begin{align*}
    \sum_{s=1}^t \normz{x_s}_{V_s^{-1}}^{2} \le \ln\del{\fr{|V_t| }{|\lam I| } } 
  \end{align*}
  and $ \ln\del{\fr{|V_t| }{|\lam I| } }  \le d \ln \del{1 + \fr{X^2 t}{d \lam}}$.
\end{lemma}
\begin{proof}
  The following is a well-known identity (e.g., see the proof of \citet[Lemma 11]{ay11improved}) :
  \begin{align*}
    \sum_{s=1}^t \ln(1 + \normz{x_t}^2_{V_{t-1}^{-1}}) = \ln\del{\fr{|V_t| }{|\lam I| } }~.
  \end{align*}
  We now lower bound the left-hand side above.
  Letting $D_s^2 = \normz{x_s}_{V_{s}^{-1}}^2$, we have
  \begin{align*}
    \sum_{s=1}^t \ln(1 + \normz{x_s}^2_{V_{s-1}^{-1}}) 
    &= \sum_{s=1}^t \ln\del[2]{1 + \fr{D_s^2}{1 - D_s^2}}  \tag{Woodbury matrix identity}
  \\&= \sum_{s=1}^t \ln\del[2]{\fr{1}{1 - D_s^2}} 
  \\&\ge \sum_{s=1}^t D_s^2~. \tag{$\ln(1 + x) \le x, \forall x$ }
  \end{align*}
\end{proof}

The following lemma is a simplified version of the elliptical potential count lemma~\cite[Lemma 4]{kim22improved} by using $\ln(1+x) \ge \ln(2) x, \forall x\in[0,1]$, which is a generalization of \citet[Exercise 19.3]{lattimore20bandit}.
\begin{lemma} \label{lem:epc} 
  (Elliptical potential count; \citet[Lemma 4]{kim22improved} and \citet[Exercise 19.3]{lattimore20bandit})  
  Let $x_1,\ldots, x_t\in\RR^d$  be a sequence of vectors with $\normz{x_s}_2\le X, \forall s\in[t]$ for some $X>0$. 
  Let $V_t = \lam I + \sum_{s=1}^t  x_s x_s^\T$ for some $\lam > 0$.
  Let $J = \{s \in [t]: \normz{x_s}_{V_{s-1}^{-1}}^2 \ge L^2 \}$ for some $L^2 \le 1$.
  Then,
  \begin{align*}
    |J| \le 3\fr{d}{L^2}  \ln\del{1 + \fr{2 X^2}{L^2\lam} } ~.
  \end{align*}
\end{lemma}

\begin{lemma}\label{lem:sg}
  Let $(X_t)_{t=1}^\infty$ be a sequence of random variables adapted to filtration $(\cG_t)_{t=0}^\infty$.
  Assume $X_t \mid \cG_{t-1}$ is $\sig_t^2$-sug-Gaussian.
  Then, for any $a>0$, we have
  \begin{align*}
    1-\dt \le \PP\del{\forall t\ge1, \sum_{s=1}^t X_s \le \fr{a}{2}\sum_{s=1}^t \sig_s^2 + \fr{1}{a} \ln(1/\dt)}~.
  \end{align*}
\end{lemma}
\begin{proof}
  Define $M_t = \exp( a \sum_{s=1}^t X_s - \fr{a^2}{2}\sum_{s=1}^t \sig_s^2)$ and verify that $M_t$ is a supermartingale.
  Applying Ville's inequality~\cite{ville39etude} concludes the proof.
\end{proof}

\begin{lemma}\label{lem:cb-bounded-rv}
  Let $(X_t)_{t=1}^{\infty} $ be a sequence of random variables adapted to a filtration $(\cG_t)_{t=1}^{\infty} $ such that $X_t \mid \cG_{t-1} \in [-R,R]$ and $\EE[X_t \mid \cG_{t-1}] = 0$ almost surely.
  Then,
  \begin{align}\label{eq:cb-bounded-rv-intermediate-form} 
    1-\dt \le \PP\del[2]{ \forall t\ge1, \forall \alpha \in [-1/R,1/R], \sum_{s=1}^t \ln(1 + \alpha X_s) \le \ln\del[2]{\fr{\sqrt{\pi(t+1)}}{\dt}} } ~.
  \end{align}
  Furthermore, let $q(x) := \fr{-\ln(1-x) - x}{x^2}$ and $\psi_t := \ln(\sqrt{\pi(t+1)}/\dt)$.
  Then,
  \begin{align*}
    1-\dt \le \PP\del[2]{\forall t\in \NN_+, \forall \alpha\in[0,1/R), \abs[2]{\sum_{s=1}^t X_s}  \le q(R\alpha) \alpha \sum_{s=1}^t X_s^2 + \fr{1}{\alpha} \psi_t } ~.
  \end{align*}
  Finally, tuning $\alpha$ implies 
  \begin{align*}
    1-\dt \le \PP\del[2]{\forall t \ge 1, \abs[2]{\sum_{s=1}^t X_s} \le \sqrt{2\sum_{s=1}^t X_s^2 \psi_t} + R\psi_t} ~.
  \end{align*}
\end{lemma}
\begin{proof}
  To show the third statement, we choose $\alpha = \del[2]{R + \sqrt{\fr{\sum_{s=1}^t X_s^2}{\psi_t}}}^{-1}$.
  One can similarly obtain an upper bound on $\sum_{s=1}^t -X_s$.

  The proof of the first statement requires modifying \citet[Theorem 1]{orabona23tight} in two ways: we deal with  (i) a sequence of random variables that are not necessarily i.i.d. and (ii) the range of $X_t$ that is $[-R,R]$ rather than $[0,1]$.

  The rest of the proof requires the background on the coin betting problem; with the full details, refer to~\citet{orabona23tight}.
  Suppose we have an algorithm $\cA$ with the initial wealth of $W_0 = 1$.
  At each time step $t$, the algorithm $\cA$ commits to a betting fraction $\alpha_t \in [-1/R,1/R]$ based on the past observations $c_1, \ldots, c_{t-1} \in [-R,R]$.
  Once the coin outcome $c_t \in [-R,R]$ is revealed, the wealth of $\cA$ denoted by $W_t$ becomes $W_t = (1 + \alpha_t c_t)W_{t-1}$.
  Suppose we set $c_t = X_t$.
  Since $(W_t)$ forms a nonnegative martingale, we can apply Ville's inequality~\cite{ville39etude} to have
  \begin{align}\label{eq:ville-for-betting}
    1-\dt \le \PP\del{ \forall t\ge1, \ln(W_t) < \ln(1/\dt)} ~.
  \end{align}
  To maximize the wealth, it is natural to choose an algorithm that has a small log-wealth regret w.r.t. the best betting fraction $\alpha \in [-1/R,1/R]$ in hindsight, which is defined as
  \begin{align}\label{eq:def-regret-betting} 
    \Regret_t := \max_{\alpha\in[-1/R,1/R]} \sum_{s=1}^t \ln(1 + \alpha X_t) - \ln(W_t)~.
  \end{align}
  One can construct an efficient $\cA$ via a reduction of coin betting problem to the two-stock online portfolio problem.
  In this problem, an algorithm $\cB$ starts from wealth $W_0^\portfolio = 1$.
  At each time step $t$, the algorithm $\cB$ determines the fraction $(b_t,1-b_t)$, which means that it will invest $b_t$ fraction of the current wealth $W^\portfolio_{t-1}$ to the first stock and the rest to the second stock.
  Then, the price change ratios $(w_{t,1}, w_{t,2})$ are revealed, and the wealth of $\cB$ changes: $W^\portfolio_t = b_t w_{t,1} W^\portfolio_{t-1} + (1-b_t) w_{t,2} W^\portfolio_{t-1}$.
  
  Similar to~\citet[Lemma 1]{orabona23tight}, one can easily construct a reduction of coin betting to online portfolio: given a coin outcome $c_t \in [-\ell,u]$, we can set the price change ratio of the first stock and the second stock as $w_{t,1} = 1 + \fr{c_t}{\ell}$ and $w_{t,2} = 1 - \fr{c_t}{u}$ respectively.
  The online portfolio algorithm $\cB$ will then produce the next investment fraction $(b_{t+1}, 1- b_{t+1})$.
  We can then set $\alpha_{t+1}$ in the coin betting problem as $\alpha_{t+1} = -\fr{1}{u} + b_{t+1}\del{\fr{1}{\ell} + \fr{1}{u}}$.
  Performing this conversion at every iteration $t$ with $\ell = u = R$ satisfies that $W_t = W^{\portfolio}_t$.
  This implies that coin betting is a special case of the online portfolio problem, and the log-wealth regret in online portfolio coincides with the log-wealth regret in coin betting.
  Our choice of online portfolio algorithm is universal portfolio equipped with the Dirichlet$(\fr12,\fr12)$ prior~\cite{cover96universal}.
  This gives the regret bound of $\Regret_t \le \ln\fr{\sqrt{\pi}\Gam(t+1)}{\Gam(t+\fr12)}\le\ln(\sqrt{\pi(t+1)})$.
  
  Then, the event in~\eqref{eq:ville-for-betting} implies
  \begin{align*}
    \ln(1/\dt) 
    > \ln(W_t)
    =   \max_{\alpha\in[-1/R,1/R]} \sum_{s=1}^t \ln(1 + \alpha  X_t) - \Regret_t
    \ge   \max_{\alpha\in[-1/R,1/R]} \sum_{s=1}^t \ln(1 + \alpha  X_t) - \ln(\sqrt{\pi(t+1)})
  \end{align*}
  This completes the proof of \eqref{eq:cb-bounded-rv-intermediate-form}.
  
  To prove the second statement, assume the event inside the probability probability statement inside \eqref{eq:cb-bounded-rv-intermediate-form}.
  We use the following inequality from \citet[Eq. 4.11]{fan15exponential}:
  \begin{align*}
    \forall \lam \in [0,1), \xi\ge-1, \ln(1 + \lam \xi) \ge \lam \xi + \xi^2 (\lam + \ln(1 - \lam))~.
  \end{align*}
  For $\alpha \in [0,1/R)$, we apply this inequality with $\lam = R\alpha$ and $\xi = X_s/R$ to obtain
  \begin{align*}
    \sum_{s=1}^t X_s \le q(R\alpha) \alpha \sum_{s=1}^t X_s^2 + \fr{1}{\alpha}\psi_t~.
  \end{align*}
  Obtaining an upper bound on $\sum_{s=1}^{t} X_s$ is symmetric to the proof above.

\end{proof}

\begin{corollary}\label{cor:23-1118-cb-v3}
  Define the clipping operator $\clip{x}_y = (\fr{y}{|x|} \wedge 1) x$ if $x\neq0$ and $\clip{x}_y = 0$ otherwise. 
  Let $(b_k)_{k=1}^\infty$ be a sequence of positive integers.
  Then, under the same setting of Lemma~\ref{lem:cb-bounded-rv}, with probability at least $1-\dt$, 
  \begin{align*}
    \forall k\in \NN_+, \forall t \ge 1, \sum_{s=1}^t \clip{X_s}_{b_k} - \EE[\clip{X_s}_{b_k}\mid \cG_{s-1}] \le \sqrt{2\sum_{s=1}^t (\clip{X_s}_{b_k} -  \EE[\clip{X_s}_{b_k}\mid \cG_{s-1}])^2 \xi_{t,k}} + b_k \xi_{t,k}
  \end{align*}
  where $\blue{\xi_{t,k}} = \ln( \sqrt{\pi(t+1)} \cd \fr{ 3.39\cd k \ln^2(1+k) }{\dt}  )$.
\end{corollary}
\begin{proof}
  The statement follows by a simple union bound argument applied to \cref{lem:cb-bounded-rv}.
\end{proof}

\begin{lemma}\label{lem:lnlnsq} 
  Define $\blue{\barlnln(x)} := \ln\ln(e \vee x)$ and let $\barlnln^2(x) = (\barlnln(x))^2$.
  Let $x \ge e$.
  Throughout, we take $\ln\ln(x) = 0$ for $x \le e$.
  Then, for $a,b>0$,
  \begin{align*}
    x \ge a + 2b(e + \barlnln^2(2a) + \barlnln^2(2b))
    \implies x \ge a + b \ln^2\ln(x)~.
  \end{align*}
\end{lemma}
\begin{proof}
  We prove the contraposition: 
  \begin{align*}
    x < a + b \ln^2\ln(x) \implies x < a + 2b(e + \barlnln^2(2a) + \barlnln^2(2b))~.
  \end{align*}
  If $x < a + be$, then the statement follows trivially.
  
  If $x \ge a + be$, then $z := \fr{x-a}{b} \ge e$.
  Then,
  \begin{align*}
    z 
    &< \ln^2\ln(bz + a)
    \\&\le \ln^2\ln(2bz \vee 2a))
    \\&\le \barlnln^2(2bz) \vee \barlnln^2(2a)  
  \end{align*}
  where the last line can be shown by the case-by-case reasoning on whether $2bz \ge e$ or not and whether $2a \ge e$ or not.
  
  If $2bz \ge e$, then
  \begin{align*}
    \barlnln^2(2bz) 
    &\le \ln^2(2\ln(z) \vee 2\ln(2b))   
    \\&\sr{(a)}{\le} \ln^2(2\ln(z)) + 2\ln^2(2) + 2\barlnln^2(2b)  \tag{$ z \ge e$}
    \\&\le \fr{z}{2} + 1 + \barlnln^2(2b)  ~.  \tag{$ z \ge e$; $2\ln^2(2)\le1$}
  \end{align*}
  If $2bz < e$, then the inequality above is also true, trivially.
  
  The two displays above imply
  \begin{align*}
    z < 2(1 + \barlnln^2(2b) + \barlnln^2(2a))~.
  \end{align*}
  Using definition of $z$ concludes the proof.
\end{proof}

\subsection{Online Ridge Regression}

The following regret equality provides an essential tool for analyses in this paper.
\begin{lemma}\label{lem:regret_equality}
  (Regret Equality; \citet[Lemma 7.1]{orabona19modern}) 
  Take \cref{ass:model} except for the stochastic modeling of $u_t$.
  Let $\Sig_t = \lam I + \sum_{s=1}^t z_s z_s^\T$.
  Then,
  \begin{align*}
    \sum_{s=1}^t f_s(\hth_{s-1}) - f_s(\th^*) 
    &= \fr{\lam}{2}\|\th^*\|^2 + \sum_{s=1}^t f_s(\hth_{s-1}) \normz{z_s}^2_{\Sig_t^{-1}}  - \fr12\normz{\th^* - \hth_t}^2_{\Sig_t}~.
  \end{align*}
\end{lemma}

\begin{lemma}\label{lem:negregret}
  Take \cref{ass:model}.
  Then,
  \begin{align*}
    1-\dt \le \PP\del[2]{ \sum_{s=1}^t (f_s(\th^*) - f_s(\hth_{s-1})) \le \sig_*^2 \ln(1/\dt)}~.
  \end{align*}
\end{lemma}
\begin{proof}
  Let $r_s := z_s^\T(\hth_{s-1} - \th^*)$.
  We have the following identity.
  \begin{align}
    f_s(\hth_{s-1}) - f_s(\th^*) 
    &= \fr12 r_s^2 - r_s\nu_s  ~. \label{eq:24-0112-instregret2}
  \end{align}
  Using~\eqref{eq:24-0112-instregret2} and the sub-Gaussianity of $(\nu_s)_{s=1}^{t}$, we can use \cref{lem:sg} to have that, with probability at least $1-\dt$,
  \begin{align*}
    \forall t\ge1, \sum_{s=1}^t f_s(\th^*) - f_s(\hth_{s-1}) \le \fr{a}{2}  \sum_{s=1}^t r_s^2 \sig_*^2 + \fr 1a \ln(1/\dt) - \fr12 \sum_{s=1}^{t} r_s^2 ~.
  \end{align*}
  Choosing $a = \fr{1}{\sig_*^2}$ concludes the proof.
\end{proof}

\begin{lemma} \label{lem:regret-upperbounded-by-rt-sq}
  Take \cref{ass:model}.
  Then, with probability at least $1-\dt$,
  \begin{align*}
  \forall t\ge1,      \sum_{s=1}^t (f_s(\hth_{s-1}) - f_s(\th^*)) \le \sum_{s=1}^t r_s^2 + \sig_*^2\ln(1/\dt) ~.
  \end{align*}
\end{lemma}
\begin{proof}
  Note that, with probability at least $1-\dt$,
  \begin{align*}
    \sum_{s=1}^t (f_s(\hth_{s-1}) - f_s(\th^*)) 
    &= \sum_{s=1}^t \fr12 r_s^2 - r_s\nu_s
  \\&\le \sum_{s=1}^t \fr12 r_s^2 + \fr{a}{2} \sum_{s=1}^{t} r_s^2 \sig_*^2 + \fr{1}{a} \ln(1/\dt) 
  \\&= \sum_{s=1}^t r_s^2 + \sig_*^2 \ln(1/\dt) ~. \tag{choose $a = 1/\sig_*^2$}
  \end{align*}
\end{proof}

\begin{lemma}\label{lem:empirical-var}
  Take \cref{ass:model}.
  If $D_t := \normz{z_t}^2_{\Sig_t^{-1}} \le \fr12, \forall t\ge1$, then with probability at least $1-2\dt$, 
  \begin{align*}
    \forall t, \sum_{s=1}^t f_s(\hth_{s-1})D_s^2 
    &\le 10\sig_*^2 \ln(1/\dt) + \fr32 \lam \normz{\th^*}_2^2 + 4\sum_{s=1}^t f_s(\th^*)D_s^2 ~.
  \end{align*}
\end{lemma}
\begin{proof}
  Let $\blue{r_s} := z_s^\T(\hth_{s-1} - \th^*)$.
  
  By the regret equality (\cref{lem:regret_equality}), we have
  \begin{align} \label{eq:24-0112-orr-regret}
    \sum_{s=1}^t \fr12 r_s^2 - r_s \nu_s
    = \sum_{s=1}^t f_s(\hth_{s-1}) - f_s(\th^*) 
    \le \fr{\lam}{2}\normz{\th^*}^2  + \sum_{s=1}^t f_s(\hth_{s-1}) D_s^2 ~.
  \end{align}
  Then, using \cref{lem:sg}, with probability at least $1-\dt$, we have, $\forall t\ge1$,
  \begin{align*}
    \sum_{s=1}^t \fr12 r_s^2 
    &\le \fr14\sum_{s=1}^t r_s^2 + 2\sig_*^2 \ln(1/\dt) + \fr{\lam}{2}\normz{\th^*}^2  + \sum_{s}^t f_s(\hth_{s-1}) D_s^2 ~,
    \\ \implies
    \fr14\sum_{s=1}^t  r_s^2 
    &\le 2\sig_*^2 \ln(1/\dt) + \fr{\lam}{2}\normz{\th^*}^2  + \sum_{s}^t f_s(\hth_{s-1}) D_s^2 ~.
  \end{align*}
  Therefore,
  \begin{align*}
    \sum_{s=1}^t (f_s(\hth_{s-1}) - f_s(\th^*) )D_s^2
    &=   \sum_{s=1}^t \fr12 r_s^2D_s^2 - \sum_{s=1}^t r_s \nu_s D_s^2
    \\&\le   \fr14\sum_{s=1}^t r_s^2 - \sum_{s=1}^t r_s \nu_s D_s^2 \tag{$D_s^2 \le \fr12$ }
    \\&\le  \sum_{s=1}^t r_s \nu_s(\fr12 - D_s^2) + \fr{\lam}{4}\normz{\th^*}^2  + \fr12\sum_{s}^t f_s(\hth_{s-1}) D_s^2 ~. \tag{by \eqref{eq:24-0112-orr-regret}}
  \end{align*}
  Using \cref{lem:sg}, with probability at least $1-\dt$, we have, $\forall t\ge1$,
  \begin{align*}
    \sum_{s=1}^t r_s \nu_s(\fr12 - D_s^2) 
    \le \fr{\xi}{2} \sum_{s=1}^t r_s^2\underbrace{(\fr12 - D_s^2)^2}_{\le \fr14} + \fr{\sig_*^2}{\xi}\ln(1/\dt)
  \end{align*}
  for some $\xi > 0$.
  Thus, 
  \begin{align*}
    \sum_{s=1}^t (f_s(\hth_{s-1}) - f_s(\th^*) )D_s^2 
    &\le \fr{\xi}{8} \sum_{s=1}^t r_s^2 + \fr{\sig_*^2}{\xi}\ln(1/\dt) + \fr{\lam}{4}\normz{\th^*}^2  + \fr12\sum_{s}^t f_s(\hth_{s-1}) D_s^2 
    \\&\le (\xi + \fr{1}{\xi})\sig_*^2\ln(1/\dt) + (\xi + 1)\fr{\lam}{4}\normz{\th^*}^2 + (\fr{\xi}{2} + \fr12) \sum_{s=1}^t f_s(\hth_{s-1}) D_s^2  \tag{use the bound on $\fr14\sum_s r_s^2$}
    \\&=   \fr{5}{2} \sig_*^2\ln(1/\dt) + \fr38 \lam\normz{\th^*}^2 + \fr34\sum_{s=1}^t f_s(\hth_{s-1}) D_s^2  \tag{choose $\xi = \fr12$}
    \\\implies
    \sum_{s=1}^t f_s(\hth_{s-1}) D_s^2
    &\le 10 \sig_*^2\ln(1/\dt) + \fr32 \lam\normz{\th^*}^2 + 4\sum_{s=1}^t f_s(\th^*) D_s^2 ~.
  \end{align*}
\end{proof}

\begin{lemma}\label{lem:sum_rs_sq}
  Take \cref{ass:model}.  
  Assume $\normz{z_s}_{\Sig_{s-1}^{-1}} \le 1, \forall s$. Then, with probability at least $1-\dt$
  \begin{align*}
    \forall t\ge1, \sum_{s=1}^t (z_s^\T(\hth_{s-1} - \th^*))^2 \le 8\del{\fr{\lam}{2}\normz{\th^*}^2 - \fr12 \normz{\th^* - \hth_{t}}^2_{\Sig_t} + \fr{1}{2} \sum_{s=1}^t \nu_s^2 D_s^2}  + 16\sig_*^2 \ln(1/\dt) ~.
  \end{align*}
\end{lemma}
\begin{proof}
  Let $r_s = z_s^\T(\hth_{s-1} - \th^*)$
  From the regret equality (\cref{lem:regret_equality})
  \begin{align*}
    \fr12 \sum_{s=1}^t r_s^2 - \sum_{s=1}^t \nu_s r_s 
    &= \fr{\lam}{2}\normz{\th^*}^2 - \fr{1}{2}\normz{\hth_t- \th^*}^2_{\Sig_t} + \fr12\sum_{s=1}^t r_s^2 D_s^2 - \sum_{s=1}^t \nu_s r_s D_s^2 + \fr{1}{2} \sum_{s=1}^t \nu_s^2 D_s^2
    \\ \implies
    \fr12 \sum_{s=1}^t r_s^2(1-D_s^2) 
    &= \underbrace{\fr{\lam}{2}\normz{\th^*}^2 - \fr{1}{2}\normz{\hth_t- \th^*}^2_{\Sig_t}  + \fr{1}{2} \sum_{s=1}^t \nu_s^2 D_s^2}_{\tsty =: A} + \sum_{s=1}^t \nu_s r_s (1-D_s^2) 
  \end{align*}
  where $(a)$ is obtained by tuning the scalar in the exponential martingale carefully.
  
  With \cref{lem:sg}, one can obtain, with probability at least $1-\dt$,
  \begin{align*}
    \forall t\ge1, \sum_{s=1}^{t}  \nu_s r_s (1-D_s^2) \le \fr14 \sum_{s=1}^t r_s^2(1-D_s^2)^2 +  2\sig_*^2 \ln(1/\dt) ~.
  \end{align*}
  Thus,
  \begin{align*}
    \sum_{s=1}^t r_s^2(1-D_s^2) 
    &\le 4A + 8\sig_*^2\ln(1/\dt)
    \\ \implies
    \sum_{s=1}^t r_s^2 &\le 8A + 16\sig_*^2\ln(1/\dt)
  \end{align*}
  where the last line is due to the fact that $\normz{x_s}_{\Sig_{s-1}^{-1}} \le 1 \implies D_s^2 \le \fr12$ using the Woodbury matrix identity.
\end{proof}

\begin{lemma}
  \label{lem:sum_eta_sq}
  Take \cref{ass:model} with the added assumption of $\nu_t \in [-R,R], \forall t$, with probability 1.
  \begin{align*}
    1-\dt \le \PP\del[2]{\forall t\ge1, \sum_{s=1}^t \nu_s^2 \le \fr32\sum_{s=1}^t \sig_s^2 + R^2\ln(1/\dt)}  ~.
  \end{align*}
\end{lemma}
\begin{proof} 
  Using the standard sub-Gaussian inequality (\cref{lem:sg}) with the fact that $\nu_s^2 - \sig_s^2$ is $R^2$-sub-Gaussian, we have, with probability at least $1-\dt$, 
  \begin{align*}
    \forall t\ge1, \sum_{s=1}^t \nu_s^2
    &=   \sum_{s=1}^t \sig_s^2 + \sum_{s=1}^t (\nu_s^2 - \sig_s^2)
    \\&\le \sum_{s=1}^t \sig_s^2 + \fr{a}{2} \sum_{s=1}^t \EE[(\nu_s^2 - \sig_s^2)^2 \mid \cF_{s-1}] + \fr{R^2} a \ln(1/\dt) 
    \\&\le \sum_{s=1}^t \sig_s^2 + \fr{a}{2} \sum_{s=1}^t \sig_s^2 + \fr{R^2}a \ln(1/\dt)
    \\&\le \fr32\sum_{s=1}^t \sig_s^2 + R^2\ln(1/\dt)    ~. \tag{choose $a = 1$} 
  \end{align*}
  Note that the above holds true for every $t$ simultaneously, with probability at least $1-\dt$.
\end{proof}

\section{Implementation Details of LOFAV}
\label{sec:lofav-implementation} 

For LOFAV, we can define an extra confidence set:
\begin{align*}
\forall \ell \in [L],   \cC^{\full}_{t,L+\ell} := \cbr[3]{\th:
  \fr12 \normz{\hth_{t,\ell} - \th}^2_{\Sig_{t,\ell}} \le \fr{\lam_\ell}{2}S^2 + \sum_{s=1}^t \ellw_{s,\ell}(\hth_{s-1,\ell})D_{s,\ell}^2 + R^2 \ln(2L/\dt) =: \blue{\gam_{t,\ell}}} ~.
\end{align*}
Then, change the algorithm so we pull $x_t 
= \arg \max_{x\in \cX_t} \max_{\th \in \wbar{\cC}^\full_{t-1}}~ \la x, \th \ra~$
where $\wbar{\cC}^\full_{t-1} := \cap_{\ell=1}^{2L} \cC^\full_{t-1,\ell}$.
Specifically,
\begin{align*}
  \max_{\th \in \wbar\cC^\full_{t-1}} \la x, \th \ra 
  = \min\cbr[2]{\min_{\ell \in [L]} \inp{x}{\barth_{t-1,\ell}} + \sqrt{2\beta_{t-1,\ell}} \normz{x}_{\barSig^{-1}_{t-1,\ell}},
  \min_{\ell \in [L]} \inp{x}{\hth_{t-1,\ell}} + \sqrt{2\gam_{t-1,\ell}} \normz{x}_{\Sig^{-1}_{t-1,\ell}}
  } ~.
\end{align*}
One can verify easily that
\begin{align*}
  1-\dt \le \PP(\forall t\in[n], \th^* \in \cC^\full_{t}  )
\end{align*}
since the event we need for $\{\cC^\full_{t,\ell}\}_{\ell=L+1}^{2L}$ to be true (i.e., the upper deviation of the negative regret) is already assumed in the proof of the correctness of $\{\cC^\full_{t,\ell}\}_{\ell=1}^L$

\clearpage
\section{Additional Experiments for LOSAN}
\label{sec:expr-more}

In this section, we provide additional experiments on toy dataset to verify our theory and identify strengths and weaknesses of LOSAN.

In addition to OFUL, we accommodate recent advances of confidence sets by using OFUL with the improved confidence set proposed in \citet[Appendix F]{chowdhury23bregman}, which coincides with the version of AMM-UCB in \citet[Appendix C.2]{flynn23improved}.
Specifically, instead of the confidence set of OFUL~\cite{ay11improved,lattimore20bandit}, i.e., 
\begin{align*}
  \cC_t = \cbr{\th \in \RR^d: \normz{\th - \hth_t}_{V_t}^2 \le \del{\sqrt{\lam}S + \sqrt{\sig_0^2\ln\del{\fr{|V_t|}{|\lam I|}} + 2\sig_0^2\ln(1/\dt) }}^2 }
\end{align*}
where $V_t = \lam I + \sum_{s=1}^t x_s x_s^\T$, $\hth_t = V_t^{-1} X_t^\T y_t$, $X_t \in \RR^{t\times d}$ is the design matrix, and $y_t \in \RR^t$ is the vector of rewards from time step 1 to $t$,
we use
\begin{align*}
  \cC_t = \cbr{\th \in \RR^d: \normz{\th - \hth_t}_{V_t}^2 \le \lam S^2 + \sig_0^2\ln\del{\fr{|V_t|}{|\lam I|}} + 2\sig_0^2\ln(1/\dt) }~.
\end{align*}
We call this variation OFUL-C.
We exclude any methods that are computationally more demanding (i.e., orderwise more computations) than OFUL such as CMM-UCB of \citet{flynn23improved}.
For all methods, we set $\lambda = 10\cd \fr{\sig_0^2}{S^2}$ where $\sig_0^2$ is the specified noise and $S$ is the known upper bound of the norm of the unknown parameter $\|\th^*\|_2$.

\textbf{More experiments for LOSAN.}
We consider a set of hard instances where the suboptimality gap of all the suboptimal arms is $\Delta := 4\sqrt{\sig_0^2 d^2/n}$ as we describe below, which is meant to achieve the worst-case regret bound for OFUL, roughly.
We set $d=20$, $n=$ 50,000, $\sig_0^2 = 1.0$, and $\th^* = (S,0,\ldots,0)\in\RR^d$. 
We take a fixed arm set $\cX$ and use it throughout the time steps (i.e., $\cX_t = \cX, \forall t$) with $|\cX| = 400$.
The arm set consists of a single best arm $x^* = \fr1S\th^*$ and the rest of the arms whose suboptimality gap is all equal to $\Delta$.
To ensure this, each suboptimal arm $x$ has its first coordinate $x(1)$ as $1-\fr{\Delta}{S}$ and the other coordinates as a vector $(x(2), \ldots, x(d))\in\RR^{d-1}$ uniformly drawn from a sphere of radius $\sqrt{1 - (1 - \fr{\Delta}{S})^2}$.
This ensures that every arm has a unit Euclidean norm and the suboptimality gap is equal to $\Delta$.

We set the parameter $S$ of all the baseline algorithms as $\|\th^*\|$ and the specified noise as  $\sig_0^2 = 1$.
We try $\|\th\|_2 \in \{1, 10\}$ and $\sig_* \in \{10^{-1}, 10^{-1/2}, 10^0\}$ and draw the reward noise by $\eta_t \sim \cN(0, \sig_*^2)$, which results in total 6 experiments.
We repeat each experiment 20 times and report their average regret along with their twice standard error as the error band in Figure~\ref{fig:expr-more-losan-success} 
As one can see, LOSAN significantly outperforms all the other methods when the noise is over-specified (see (a-b) and (d-e)).
Interestingly, in the just-specified setting of (c) and (f), LOSAN outperforms OFUL and is on par with LOSAN, which shows that LOSAN is not only adaptive to the true noise but also numerically tight enough to be competitive with the state-of-the-art confidence set.
Note that, when $\|\th^*\|$ is large ($\|\th^*\|=10$), LOSAN is slightly worse than OFUL-C.

\begin{figure}[t]
  \begin{tabular}{ccc}
    \includegraphics[width=0.3\linewidth]{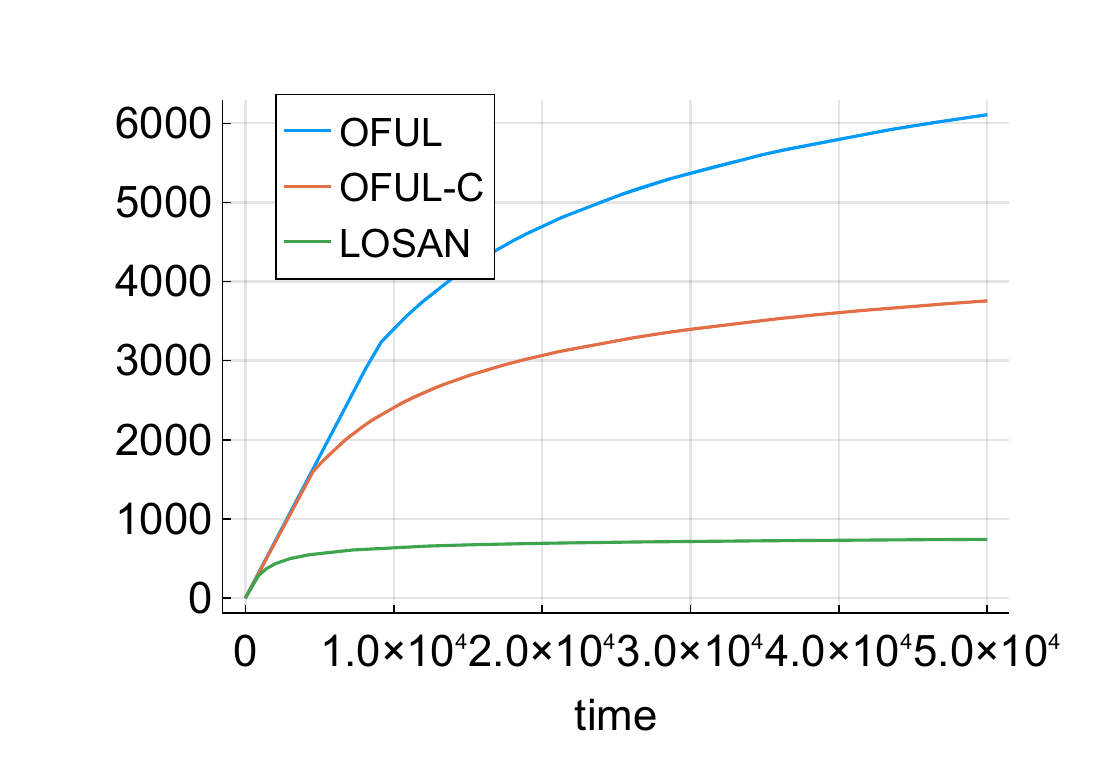}
  & \includegraphics[width=0.3\linewidth]{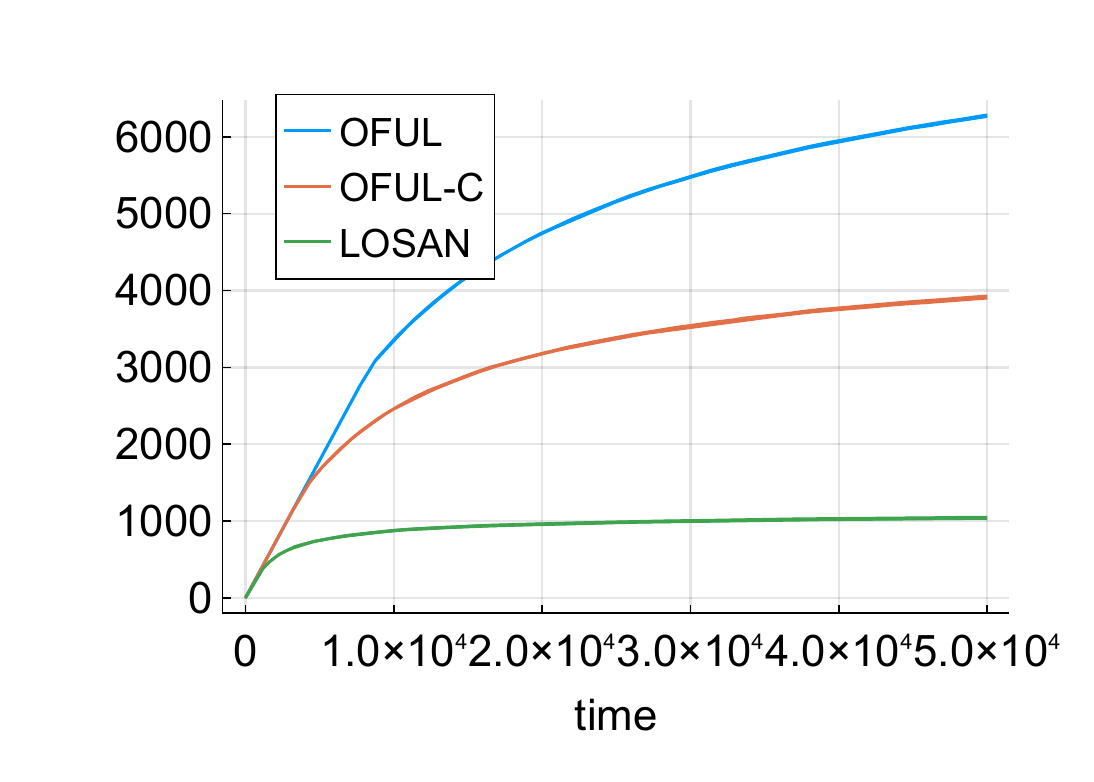}
  & \includegraphics[width=0.3\linewidth]{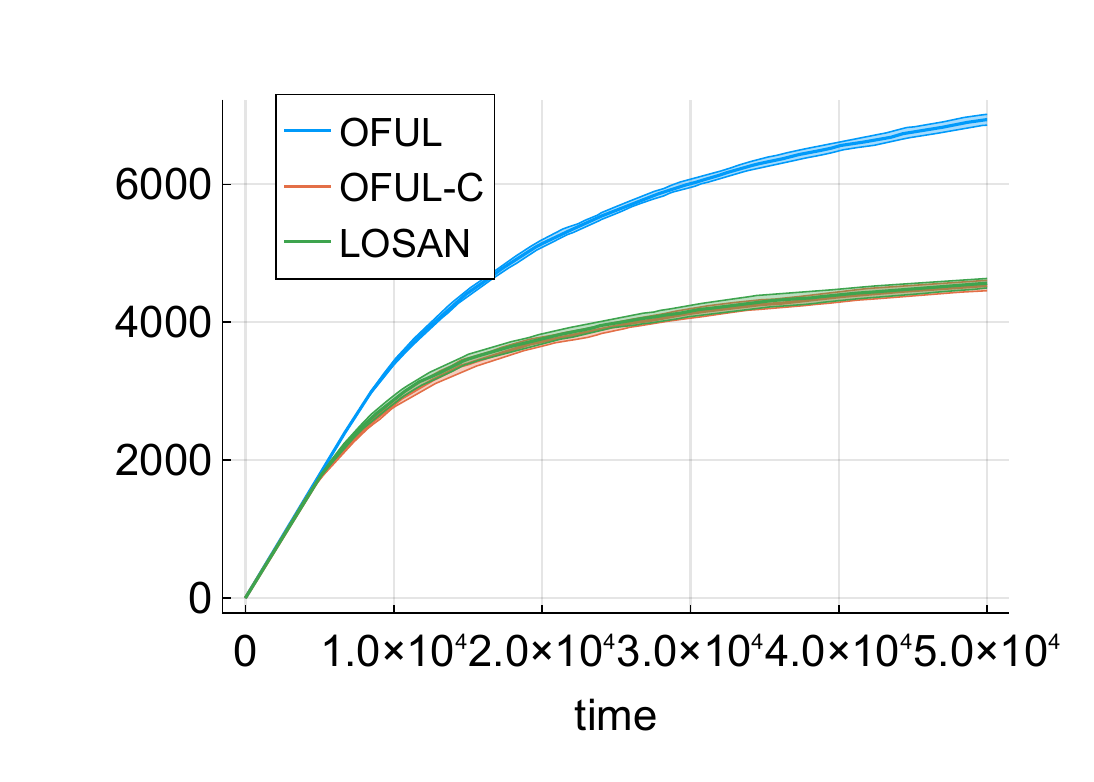} \\
   (a) $\|\th^*\|=1$, $\sig_* = 10^{-1}$    & (b) $\|\th^*\|=1$, $\sig_* = 10^{-1/2}$  & (c) $\|\th^*\|=1$, $\sig_* = 10^{0}$ \\
    \includegraphics[width=0.3\linewidth]{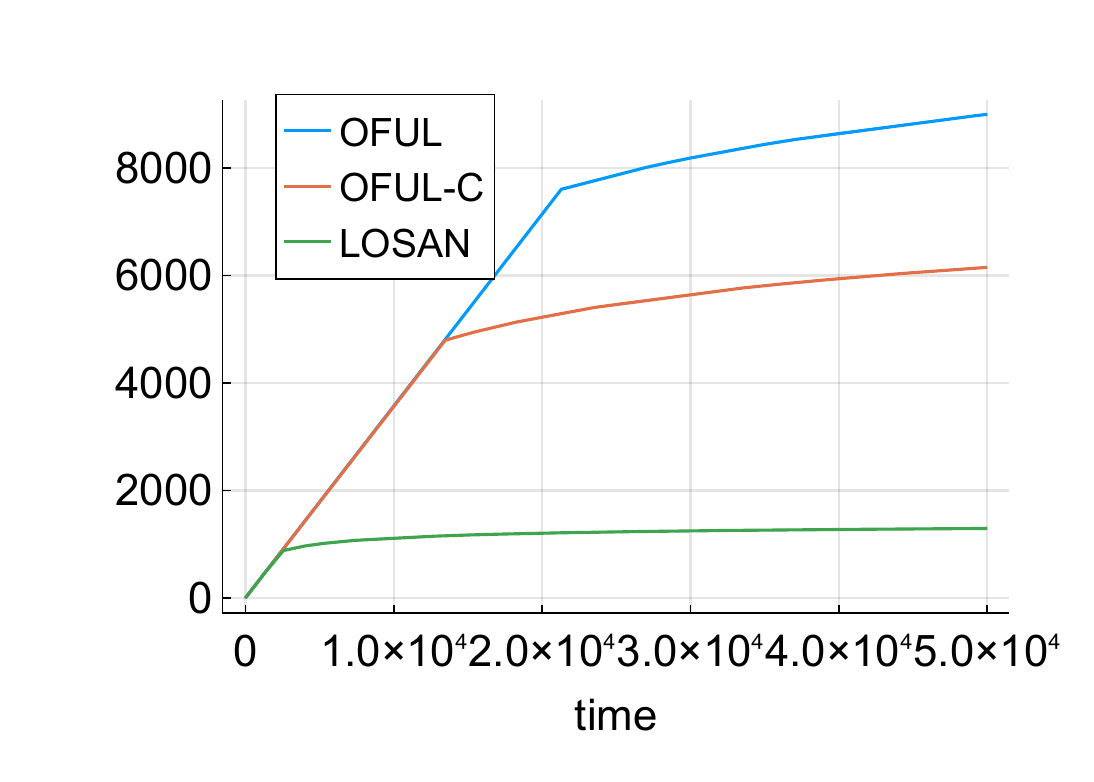}
  & \includegraphics[width=0.3\linewidth]{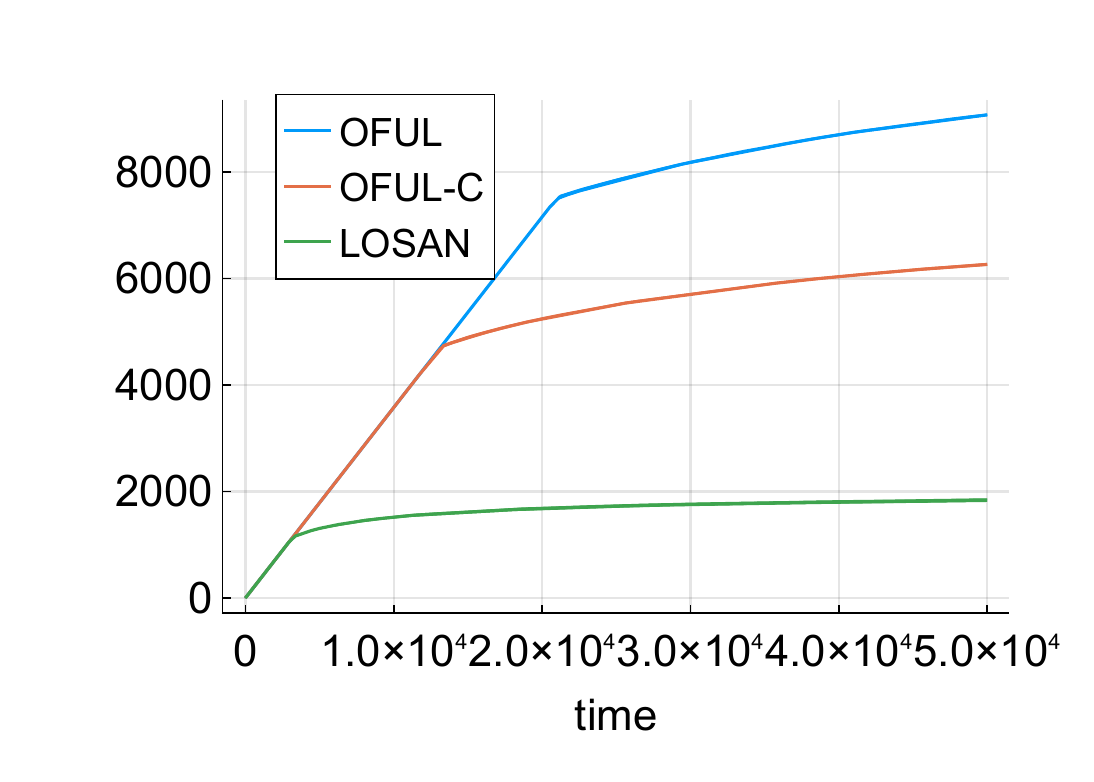}
  & \includegraphics[width=0.3\linewidth]{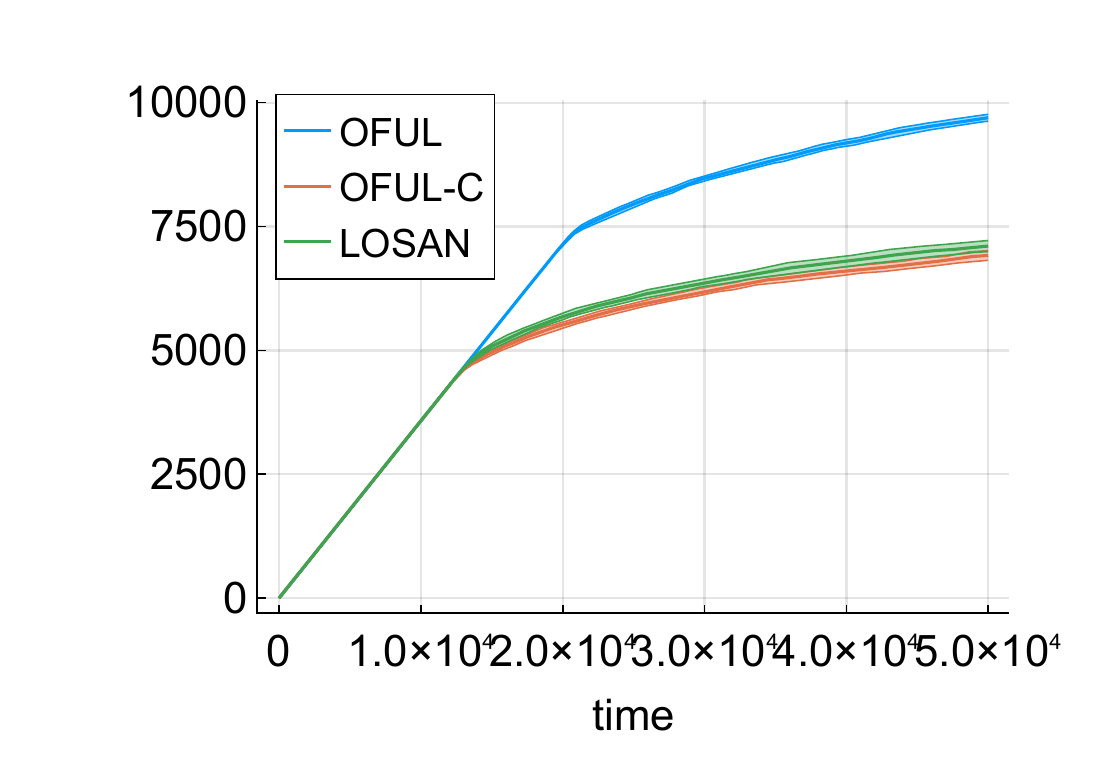} \\
   (d) $\|\th^*\|=10$, $\sig_* = 10^{-1}$ & (e) $\|\th^*\|=10$, $\sig_* = 10^{-1/2}$ & (f) $\|\th^*\|=10$, $\sig_* = 10^{0}$
  \end{tabular}
  \caption{Toy Experiments with Gaussian noise.}
  \label{fig:expr-more-losan-success}
\end{figure}

\textbf{A just-specified setting where LOSAN is worse than OFUL and OFUL-C.}
Note that LOSAN is not without a weakness in the just-specified setting.
We here report a case where LOSAN performs worse than OFUL and OFUL-C and explain why.

The problem instance considered here is an easy problem case in the sense that not many samples are required to have small regret.
This instance has a high signal-noise ratio in the sense that the smallest suboptimality gap $\Delta_{\min} := \min_{x\in \cX, x\neq x^*} \la x^* - x,\th^*\ra$, where $x^*$ is the best arm, is large compared to the noise level $\sigma_*^2$.
Specifically, we draw $\th^*$ uniformly at random from a ($d=20$)-dimensional sphere of radius 15 (thus $\|\th^*\| = 15$), and each arm is drawn uniformly at random from the unit sphere.
We sample 800 such arms ($|\cX| = 800$).

When we created such an instance over 20 trials, we obtained that the average and the standard deviation of $\Delta_{\min}$ are approximately $5.38$ and $0.84$.
When considering Gaussian noise with variance $\sigma_*^2 = 1$, the value of $\Delta_\tmin$ is much larger than $\sigma_*$, and thus one can expect that an algorithm could identify the best arm roughly after pulling 20 arms that are sufficiently linearly independent.
Furthermore, $\Delta_\tmax := \min_{x\in \cX, x\neq x^*} \la x^* - x,\th^*\ra$ has the average and the standard deviation of approximately $24.79$ and $0.68$.
The large value of $\Delta_\tmax$ implies that the cost of pulling a suboptimal arm can be very large.

We run OFUL, OFUL-C, and LOSAN with 20 trials where each trial samples a fresh $\th^*$ and an arm set.
We report the resulting regret in Figure~\ref{fig:expr-more-losan-failure}(a) where the error band is twice the standard error.
As one can see, LOSAN is outperformed by both OFUL and OFUL-C.
The problem is easy enough that once an algorithm finds the best arm, then it rarely increases the regret, which means that they have arrived at a regime where the regret is polylogarithmic.
While the slope of the regret after time 1000 is mostly the same for all methods, LOSAN incurs more regret early on.

To inspect the reason, we plot the largest upper confidence bound (UCB), i.e., $\max_{x\in \cX, \th\in \cC_{t-1}} \la x,\th\ra$), in Figure~\ref{fig:expr-more-losan-failure}(b).
Both OFUL and OFUL-C drop the value of the largest UCB significantly around $t=20$.
While LOSAN initially has a tighter confidence bound at the beginning, its drop of confidence bound happens at a much slower rate during which LOSAN pulls suboptimal arms more frequently, resulting in high regret.
We speculate that such a difference is within a constant factor at best; we would be very surprised if one can show that LOSAN is orderwise worse than OFUL/OFUL-C (i.e., exist a series of instances where the difference can be arbitrarily large). 
Further examination and potential improvement for LOSAN is left as future work.

As a comparison, the instance used in Figure~\ref{fig:expr-more-losan-success} has $\Delta_\tmin = \Delta_\tmax$, which means that not knowing the best arm has the same cost, and the bottleneck of the regret is mainly to precisely locate $\th^*$.
On the other hand, the bottleneck in the spherical instance used here is to quickly figuring out $\th^*$ very roughly (since that is enough to find the best arm), and any delay in doing so results in large cost in regret (due to large $\Delta_\tmax$).

\begin{figure}[t]
  \begin{center}
  \begin{tabular}{cc}
    \includegraphics[width=0.4\linewidth]{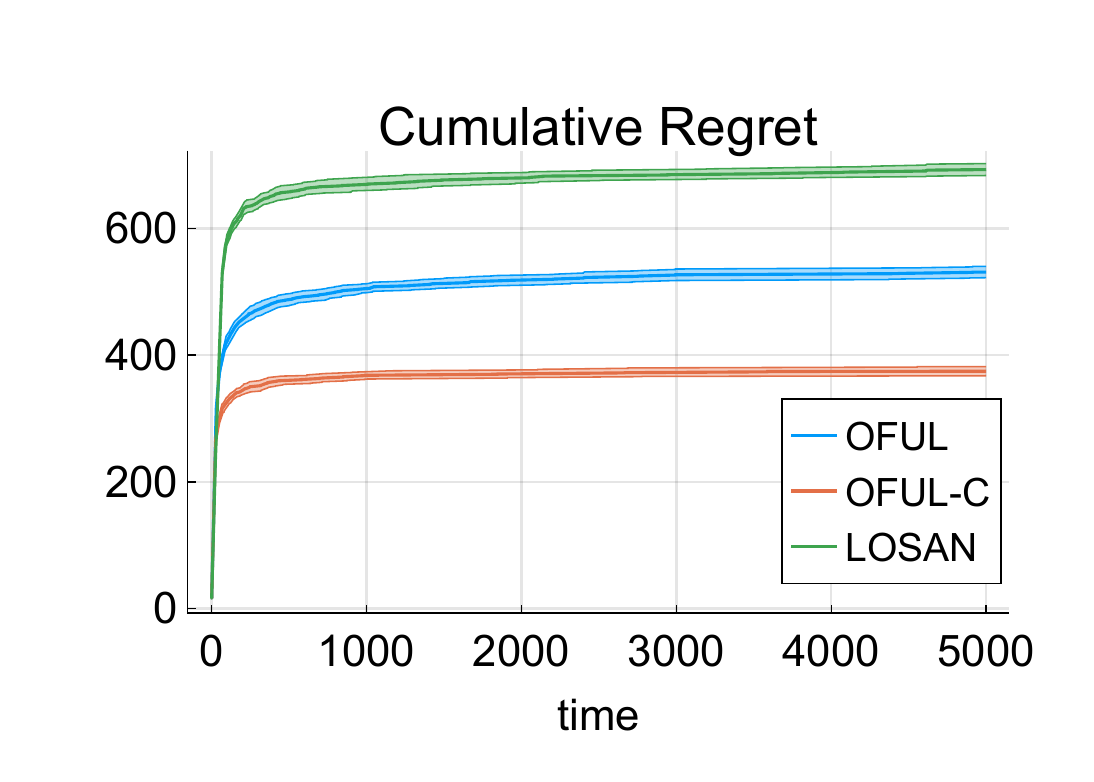}
    & \includegraphics[width=0.4\linewidth]{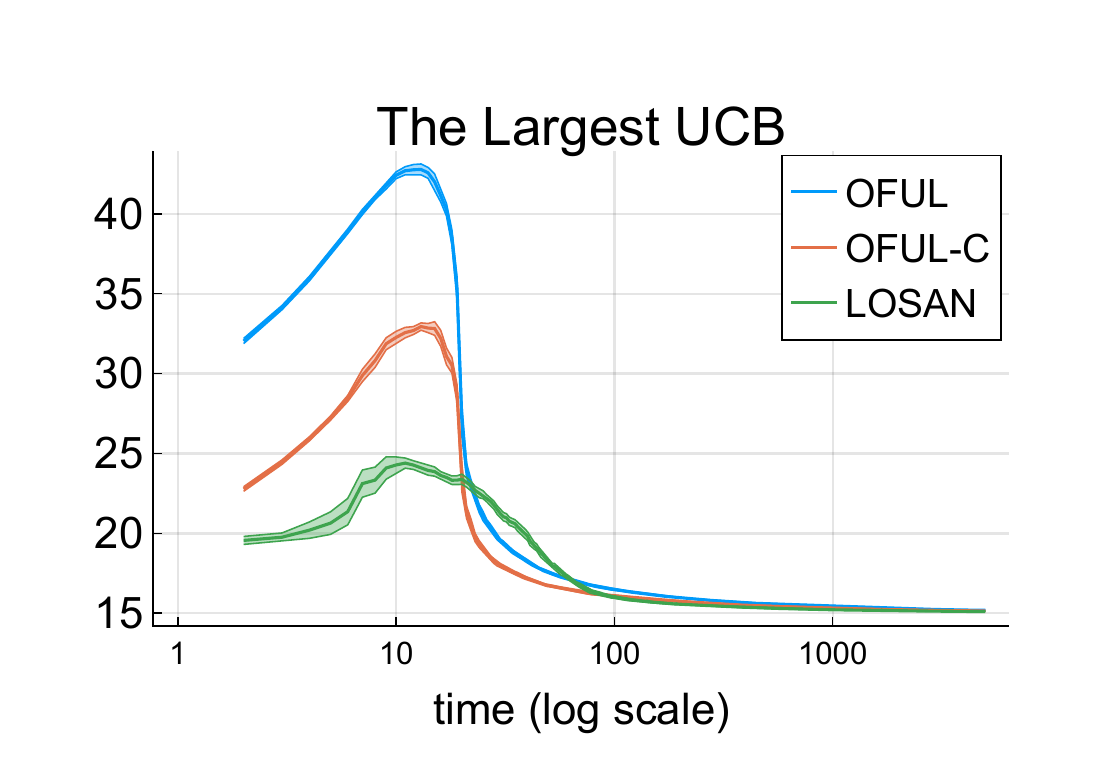}\\
    (a) & (b) 
  \end{tabular}
  \end{center}
  \caption{A case where LOSAN performs worse than baseline method.}
  \label{fig:expr-more-losan-failure}
\end{figure}

\end{document}